\documentclass[11pt]{article}

\input{preamble.tex}
\usepackage[mathscr]{euscript}

\newcommand\tldrDone[1]{}

\usepackage{graphicx}
\usepackage{subcaption}

\newcommand{\generate}{\mathtt{generate}}
\newcommand{\detect}{\mathtt{detect}}
\newcommand{\shuffle}{\mathtt{shift\txt{-}generate}}

\newcommand{\decode}{\Gamma}
\newcommand{\its}{$\mathtt{ITS}$}
\newcommand{\itsedit}{$\mathtt{ITS}$-$\mathtt{edit}$}
\newcommand{\gumb}{$\mathtt{EXP}$}
\newcommand{\gumbedit}{$\mathtt{EXP}$-$\mathtt{edit}$}
\newcommand{\gumbhash}{$\mathtt{EXP}$-$\mathtt{hash}$}
\newcommand{\kgwtwo}{$\mathtt{KGW}$-$\mathtt{2.0}$}
\newcommand{\kgwone}{$\mathtt{KGW}$-$\mathtt{1.0}$}

\newcommand{\len}[1]{\mathtt{len}(#1)}

\newcommand{\R}{\mathbb{R}}
\newcommand{\N}{\mathbb{N}}

\newcommand{\eps}{\varepsilon}

\newcommand{\E}[2][]{\mathbb{E}_{\ifx &#1& \else #1 \fi}\left[#2\right]}

\renewcommand{\P}[1]{\mathbb{P}\(#1\)}

\DeclareMathOperator*{\argmin}{arg\,min}

\newcommand{\Ep}{\mathbb{E}}

\renewcommand{\P}{\mathbb{P}}

\newcommand{\norm}[1]{\left\|{#1}\right\|} 

\newcommand{\abs}[1]{\left|{#1}\right|} 
\newcommand{\est}[1]{\widehat{#1}}

\newcommand{\mc}[1]{\mathcal{#1}}

\newcommand{\indic}[1]{\mathbf{1}\!\left\{#1\right\}} 

\newcommand{\defeq}{:=}

\definecolor{innerboxcolor}{rgb}{.9,.95,1}
\definecolor{outerlinecolor}{rgb}{.6,0,.2}
\definecolor{shcolor}{RGB}{27, 87, 14}
\definecolor{rckcolor}{RGB}{0,0,255}

\newcommand{\txt}[1]{\textup{#1}}
\newcommand{\iid}{\overset{\text{i.i.d.}}{\sim}}

\newcommand{\wt}[1]{\widetilde{#1}}

\begin{document}

\Crefname{algorithm}{Algorithm}{Algorithm}

\begin{center}
  \Large{Robust Distortion-free Watermarks for Language Models} \\
  \vspace{.5cm}

  \large{Rohith Kuditipudi ~~~~ John Thickstun ~~~~ Tatsunori Hashimoto ~~~~ Percy Liang} \\
  \vspace{.25cm}
  \large{Department of Computer Science \\
    Stanford University}
  \\
  \vspace{.2cm}
  \large{July 2023}
\end{center}



\begin{abstract}
    We propose a methodology for planting watermarks in text from an 
    autoregressive language model that are robust to perturbations without 
    changing the distribution over text up to a certain maximum generation budget.
    We generate watermarked text by mapping a sequence of random numbers---which 
    we compute using a randomized watermark key---to a sample from the language model.
    To detect watermarked text, any party who knows the key can align the 
    text to the random number sequence.
    We instantiate our watermark methodology with two sampling schemes: 
    inverse transform sampling and exponential minimum sampling.
    We apply these watermarks to three language models---OPT-1.3B, LLaMA-7B and Alpaca-7B---to 
    experimentally validate their statistical power and robustness to various paraphrasing attacks.
    Notably, for both the OPT-1.3B and LLaMA-7B models, we find we can reliably detect watermarked text ($p \leq 0.01$) from $35$ tokens 
    even after corrupting between $40$-$50$\% of the tokens via random edits (i.e., substitutions,
    insertions or deletions). 
    For the Alpaca-7B model, we conduct a case study on the feasibility of watermarking responses 
    to typical user instructions. Due to the lower entropy of the responses, detection is more difficult:
    around $25\%$ of the responses---whose median length is around $100$ tokens---are detectable with $p \leq 0.01$, 
    and the watermark is also less robust to certain automated paraphrasing attacks we implement.\footnote{
        We release all code publicly at \url{https://github.com/jthickstun/watermark}.
    }
\end{abstract}

\section{Introduction}\label{sec:intro}
The ability of language models to mass produce human-like text
creates an acute, renewed emphasis on the importance of provenance of generated content.
For example, the website StackOverflow has banned users from posting answers using OpenAI's
ChatGPT model to mitigate the spread of misinformation on the platform \cite{vincent23stackoverflow}.
A reliable forensic tool for attributing text to a particular language model would empower 
individuals---such as platform moderators and teachers---to 
enact and enforce policies on language model usage; it would also  
better enable model providers to track the (mis)use of their models, e.g.,
to scrub synthetic text from the training data of future language models. 




To achieve provenance,
a \textit{watermark} is a signal embedded within some 
generated content---in our case, text from a language model---that encodes the source of the content.
We consider a setting where an untrusted third party user queries a language model (LM) by sending prompts 
to a trusted provider (Figure~\ref{fig:main-top}):
the LM provider generates text from their language model with a watermark so that a detector may later 
identify the source of the text if the user publishes it.
The ideal watermark should satisfy at least the following three desiderata:
\begin{itemize}
    \item[1.] \textbf{distortion-free}---the watermark should preserve the original text distribution; 
    \item[2.] \textbf{agnostic}---it should be detectable without the language model and/or prompt;
    \item[3.] \textbf{robust}---it should withstand perturbations of the watermarked text.
\end{itemize}
Existing watermarks either distort the model's sampling distribution, thus altering the API functionality \cite{kirchenbauer2023watermark,aaronson2023openai}, 
or are not robust to editing or cropping the text \cite{christ2023undetectable}.
Meanwhile, classical steganographic techniques for covertly encoding messages within samples of text from a language model
are neither agnostic nor robust \cite{ziegler2019neural}.
We develop the first watermarks for attributing text to a language model that achieve all three desiderata.

\begin{figure}
    \centering
    \includegraphics[width=15cm]{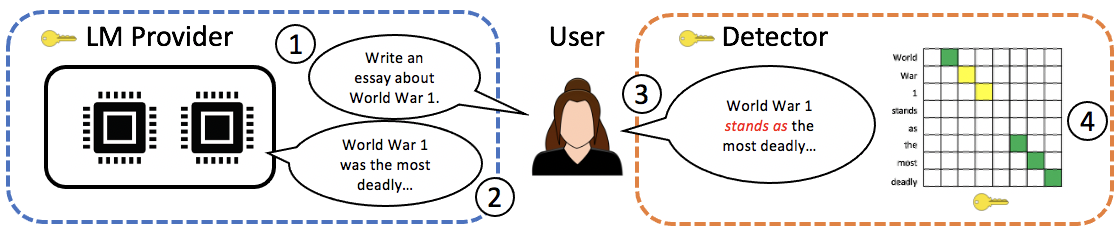}
    \caption{We define the following watermarking protocol between three parties: 
    the LM provider, the user, the detector. The LM provider and the detector 
    coordinate via a shared key, while the user is an untrusted third party. 
    The protocol consists of four steps: 1) the user sends a prompt to the LM provider;
    2) the LM provider generates watermarked text to the user; 3) the user edits 
    the watermarked text (to avoid detection) and publishes the edited text; 4) the 
    detector verifies which of the published text came from the LM provider.
    }
    \label{fig:main-top}
\end{figure}

Our methodology consists of two components, which the LM provider and detector respectively use to execute 
the two steps of the protocol in Figure~\ref{fig:main-top} under their control: 
a $\generate$ method that deterministically maps a sequence $\xi$ of random numbers encoded by a
watermark key---which we call the watermark key sequence---to a sample from the language model, and a $\detect$ method that aligns a putative watermarked text
with the watermark key sequence using the shared key.
Informally, our watermarks are \emph{distortion-free} in the sense that---marginalizing over the watermark key sequence---each 
call to $\generate$ is equal in distribution to a sample from the original language model, i.e., the distribution
\begin{align*}
    P(\textbf{text}) = \int_{\xi} \indic{\textbf{text} = \generate(\xi,\textbf{prompt})} d \nu(\xi)
\end{align*}
is equal to the original language model's sampling distribution. 

The challenge of detecting watermarked text is that the detector cannot simply recompute $\generate$ and compare its 
output against the text since they do not necessarily know the prompt which produced the text: in practice, users often 
crop the prompt when publishing text from a language model.
Our watermarks are \emph{agnostic} in the sense that they are easily detectable 
with a suitable model-agnostic and prompt-agnostic test statistic $\phi$ such that 
$\phi(\generate(\xi,\textbf{prompt}),\xi) \ll \phi(\textbf{text},\xi)$
for any $\textbf{text}$ that is independent of the watermark key sequence.
The idea here is that the detector may use $\phi$ within $\detect$ to compute a $p$-value with respect to the null 
hypothesis that the text is independent of the watermark key sequence, i.e.,
that the text is not watermarked.

To ensure $\detect$ is \emph{robust} to edits of the watermarked text, 
the core idea underpinning the design of each test statistic $\phi$ is to leverage techniques 
for robust sequence alignment to align a putative watermarked text with the watermark key sequence;
we quantify the quality of the alignment using an ``alignment cost" specific to each watermark.
The sequence alignment procedure ensures the watermark is detectable from even a 
small, corrupted block of watermarked text planted within some other larger text.
Of course, a sufficiently motivated and/or sophisticated user can still evade detection 
by simply rewriting the text from scratch themselves (or, using another language model to generate the text);
the point of a robust watermark is simply that the amount of effort and/or resources a user requires to produce text 
that evades watermark detection should be commensurate to what they would have expended had they not had access 
to the watermarked language model in the first place.

Whereas $\generate$ is a deterministic function, if our watermark 
produced the same text every time for each prompt it would not be very useful.
We resolve this limitation by designing a wrapper $\shuffle$ around $\generate$ that calls $\generate$ using a randomly chosen subsequence of $\xi$
instead of generating tokens from the same starting point each time.
For the same reasons that $\detect$ is robust to editing and cropping watermarked text,
calling $\generate$ in this fashion does not affect watermark detectability.
In practice, the statistical power of our watermarks improves exponentially with respect 
to the length of the putative watermarked text and diminishes only linearly with the 
length of the random number sequence;
thus, by increasing the length of the random number sequence, we can reduce the probability 
of reusing the same random subsequence while still ensuring our watermark has good statistical power
(i.e., that it yields low $p$-values for watermarked text).
So long as we do not reuse an element of the key sequence, successive calls to $\shuffle$ will 
be jointly indistinguishable from regular calls to the language model.

To remark briefly on the work most closely related to ours, we contrast the distortion-free property of our watermarks
with the hashing-based watermarks of \citet{kirchenbauer2023watermark} and \citet{aaronson2023openai}
that bias the distribution of watermarked text towards certain $k$-grams by hashing 
a sliding window of the previous $k-1$ tokens to determine the next token pseudorandomly.
We give examples of prompts (e.g., ``Give me a list of 20 movies.") for which the bias due to hashing is clearly noticeable in our experiments.
\citet{christ2023undetectable} propose a variation of hashing in which the window size changes based on the entropy of the generated tokens to avoid hash collisions with high probability.
Their motivation is similar to ours in that they focus on preserving 
the original text distribution; however, like \citet{kirchenbauer2023watermark} and \citet{aaronson2023openai}, using larger window sizes
hurts robustness as an adversary can break the watermark by replacing a single token in each window.
Our watermark is not only distortion-free but also robust to substantial corruption of the text,
which is crucial in practice.
We defer a more thorough discussion of related work to the next section (Section~\ref{sec:related-work}).





We describe the details of our methodology in Section~\ref{sec:results}, 
wherein we give two instantiations of watermarks---using inverse transform sampling 
and exponential minimum sampling---and provide analyses of their statistical power.
We experimentally validate the power
and robustness of our watermarks using the OPT-1.3B, LLaMA-7B and Alpaca-7B 
language models in Section~\ref{sec:experiments}.
Across all models, we find the second instantiation using exponential minimum sampling to be 
the most powerful.
For both the OPT-1.3B and LLaMA-7B models, using this watermark we can reliably detect watermarked text ($p \leq 0.01$) from $35$ tokens 
even after corrupting between $40$-$50$\% of the tokens via random edits (i.e., substitutions,
insertions or deletions);
the watermark also remains detectable from $50$ tokens even after paraphrasing 
the text by translating to French/Russian and back.
For the Alpaca-7B model, we conduct a case study on the feasibility of watermarking responses 
to typical user instructions. Due to the lower entropy of the responses, detection is more difficult:
around $25\%$ of the responses---whose median length is around $100$ tokens---are detectable with $p \leq 0.01$, 
and the watermark is also less robust to paraphrasing.
We release code for implementing the watermark and reproducing the experiments in this paper, as well as additional supplementary material including an in-browser demo of the watermark detector.\footnote{For assets and supplemental material, see: \url{https://github.com/jthickstun/watermark}.}


\subsection{Related work}\label{sec:related-work}
Text watermarking is a special case of linguistic steganography, 
in that the goal is to convey a hidden message---the watermark---within a passage of text.
Existing approaches to linguistic steganography fall under two broad categories: 
\emph{edit-based} methods that modify a pre-existing text,
and \emph{generative} methods that construct a distribution over cover text \citep{ueoka2021frustratingly}. 
Crucially, in contrast to steganography, the literature on digital watermarking has historically foregrounded 
robustness to corruption as a key attribute of a good watermark \citep{katzenbeisser2000digital,atallah2001natural}.
In this light, a text watermark should be able to withstand some perturbations of the text,
thus precluding the direct application of many existing techniques for linguistic steganography \citep{dai2019towards,ziegler2019neural,shen2020near}.




Older work on text watermarking considers editing a pre-existing text to include 
a watermark \citep{rizzo2019fine,abdelnabi2021adversarial,yang2022tracing}; 
for a survey of edit-based watermarks, see \citet{kamaruddin2018review}. In contrast, 
we are interested in generating watermarked text while preserving the distribution over the text from a language model.
Work on generative watermarking is nascent, underwritten by recent advances in open-ended text generation \citep{brown2020language}. 
Pioneering work by \citet{venugopal2011watermarking} proposed a generative watermark for the output of a machine translation system, biasing the system towards translations with particular features that can later be detected using a hypothesis test.

Our work is most closely related to \citet{kirchenbauer2023watermark}, 
who watermark text from a language model by reweighting the token log-probabilities 
from the model at inference time as a function (i.e., hash) of the previous $k-1$ tokens, where $k \in \N$ is a hyperparameter.
In ongoing unpublished work concurrent to ours, \citet{aaronson2023openai} describes a technique for watermarking language models 
using exponential minimum sampling (a close relative of the Gumbel trick \cite{papandreou2011perturbandmaprf}) 
to sample from the model, where the inputs to the sampling mechanism
are also a hash of the previous $k-1$ tokens.
Neither watermark is distortion-free, and in fact we show in our experiments that the 
distortions manifest noticeably in practice (e.g., excessive repetition of certain tokens).
Specifically, both \citet{kirchenbauer2023watermark} and \citet{aaronson2023openai}
bias the distribution toward a subset of $k$-grams.
Increasing $k$ makes the bias less noticeable but hurts the robustness of both watermarks:
an adversary can break the signal from a particular token by
replacing any one of the previous $k-1$ tokens.

Also concurrent to our work, 
\citet{christ2023undetectable} propose watermarking blocks of text from a language model
by hashing each block to seed a sampler for the next block.
\citet{christ2023undetectable} vary their block sizes---which are analogous to the hyperparameter $k$ of \citet{kirchenbauer2023watermark} 
and \citet{aaronson2023openai}---as a function of the empirical entropy of the constituent tokens to avoid using the same seed twice with high probability.
Their work is similar to ours in that they preserve the original language model's sampling distribution;
however, the resulting watermark is not robust since in order to mitigate the distortion induced by hashing
the block sizes must be sufficiently large to avoid hash collisions with high probability over all 
blocks and---similar to \citet{kirchenbauer2023watermark} and \citet{aaronson2023openai}---replacing 
any token in the previous block breaks the watermark in the next block.
Whereas \citet{christ2023undetectable}---who do not run experiments---choose their block sizes to be sufficiently 
large to minimize distortion, \citet{kirchenbauer2023watermark} and \citet{aaronson2023openai} 
recommend choosing $k$ to be a small constant in practice, which ensures a moderate amount of robustness 
by introducing some distortion.
Finally, whereas our definition distortion-freeness implies exact equality in distribution of watermarked text to unwatermarked text for a single query 
to the language model, \citet{christ2023undetectable} propose a definition of ``undetectability'' that implies approximate equality in distribution, i.e., 
approximate distortion-freeness, over multiple queries.
Using $\shuffle$, we also achieve approximate distortion-freeness in the multiple query setting, though the runtime of our watermark detection procedure must grow with the number of queries; 
we discuss these trade-offs in more detail in Section~\ref{sec:discussion}.

An alternative approach for detecting synthetic text is to 
learn a classifier between synthetic and human text \citep{jawahar2020automatic,mitchell2023detectgpt}.
A key advantage of such methods over watermarking is that they do not require coordination with the original producer of the text 
(i.e., the LM provider); however, their effectiveness is distribution dependent and they do not provide a priori (distribution-free)
guarantees on the significance level of detection (i.e., Type I errors).

Finally, we note that our setting is different from the literature on planting watermarks in the training data of machine learning models, 
e.g., to infer the model's training set or otherwise influence the model's output \citep{he2022protecting,he2022cater,zhao2023protecting}.
Such watermarks are not distortion-free by design, since the point is to plant some learnable signal in the training data that influences 
the behavior of models which train on the watermarked data.



\section{Methodology and theoretical analysis}\label{sec:results}
Let $\mc{V}$ be a discrete set, i.e., the \textit{vocabulary}, and let 
$p \in \mc{V}^* \to \Delta(\mc{V})$ be an autoregressive \textit{language model}
which maps a string of arbitrary length to a distribution over the vocabulary, with
$p(\cdot \mid x)$ 
denoting the distribution of the next token given the prefix $x \in \mc{V}^*$. 
Let $\Xi$ denote the space in which the elements of the watermark key sequence lie.
Recall the main protocol (Figure~\ref{fig:main-top}) which defines our problem setting:
\begin{itemize}
  \item[0.] The LM provider shares a random watermark key sequence $\xi \in \Xi^*$ with the detector;
    \item[1.] The user sends a prompt $x \in \mc{V}^*$ to the LM provider;
    \item[2.] The LM provider generates text $Y \in \mc{V}^*$ by $Y = \generate(x,\xi)$;
    \item[3.] The user publishes text $\wt{Y} \in \mc{V}^*$, which may be either (i) (an edited version of) the generated text $Y$
      or (ii) text independent of $Y$ (e.g., text that they wrote themselves);
    \item[4.] The detector determines if $\wt{Y}$ is watermarked---i.e., if $\wt{Y}$ depends on the watermark key sequence---by computing 
                a $p$-value $\est{p} = \detect(\wt{Y},\xi)$ with respect to the null hypothesis 
                that $\wt{Y}$ is independent of $\xi$ (i.e., not watermarked).
\end{itemize}

\subsection{Protocol details}\label{sec:protocol}
In the protocol, the LM provider calls the $\generate$ method (Algorithm~\ref{algorithm:generate}) to autoregressively generate 
text from a language model using a \emph{decoder} function $\decode : \Xi \times \Delta(\mc{V}) \to \mc{V}$ which
maps an element $\xi_i$ of the watermark key and a distribution over the next token to a next token prediction.
By design, over the randomness of $\xi_i$ the prediction should constitute a sample from the distribution,
i.e., $\P(\decode(\xi_i,\mu) = y_i) = \mu(y_i)$.
\begin{definition}\label{defn:distortion-free}
    A decoder $\decode : \Xi \times \Delta(\mc{V}) \to \mc{V}$ is \emph{distortion-free} 
    with respect to (the distribution of) a random variable $\xi \in \Xi$ if for any $\mu \in \Delta(\mc{V})$ and $y \in \mc{V}$ it satisfies
    $\P(\decode(\xi,\mu) = y) = \mu(y)$.
\end{definition}
We relate Definition~\ref{defn:distortion-free} to our informal definition of distortion-free text 
in the introduction through the following simple lemma.
Essentially, so long as the watermark key sequence is long enough that we do not reuse any part of it to generate text, 
the only material difference between an LM provider using $\generate$ versus sampling directly from the language model 
is that the sequence $\xi$ is an input to the method rather than resampled i.i.d. within the method for each call.
We treat the language model $p$, the decoder $\decode$, and generation length $m$ as internal parameters of the $\generate$ method.
\begin{lemma}\label{lemma:dist-free-equiv}
    Let $m,n \in \N$ with $n \geq m$.
    Let $\decode$ be distortion free with respect to a distribution $\nu \in \Delta(\Xi)$
    and let $\{\xi_i\}_{i=1}^n \iid \nu$. Let $Y = \generate(\xi;m,p,\decode)$.
    Then $Y_i \sim p(\cdot \mid Y_{:i-1})$ for $i \in [m]$.
\end{lemma}
\begin{proof}
    As $n \geq m$, we have $\{\xi_i\}_{i=1}^m \iid \nu$.
    The claim then follows immediately from applying Definition~\ref{defn:distortion-free}
    to Line~\ref{line:generate} of $\generate$ for $i \in [m]$.
\end{proof}

To simplify the remainder of the presentation, we do not pass a prompt as input to $\generate$.
As the language model $p$ is arbitrary and $\detect$ is model-agnostic, 
this simplification is without loss of generality since $p$ itself may model the distribution of 
text from some base model given an arbitrary prompt.
Also, unless stated otherwise, without loss of generality we let $\mc{V} = [N]$ throughout the paper,
where $N \in \N$ is the vocabulary size.

\begin{algorithm}[h]
    \DontPrintSemicolon
    \SetKwInOut{Input}{Input}
    \SetKwInOut{Output}{Output}
    \SetKwInOut{Params}{Params}
    \SetKwInOut{Noise}{Noise}
    \caption{\label{algorithm:generate}
        Watermarked text generation ($\generate$)}
    \Input{watermark key sequence $\xi \in \Xi^n$}
    \Params{generation length $m$, language model $p$, decoder $\decode$
    }
    \Output{string $y \in \mc{V}^m$}
    \For{$i \in 1,\dots,m$}{ 
        $y_i \leftarrow \decode(\xi_{i\%n},p(\cdot \mid y_{:i-1}))$ \label{line:generate} 
    }
    \Return $y$\;
\end{algorithm}

The detector calls the $\detect$ method (Algorithm~\ref{algorithm:detect}) to compute---via a permutation test with $T$ resamples---a $p$-value 
with respect to a test statistic $\phi : \mc{V}^* \times \Xi^* \to \R$ for the null 
hypothesis that $\wt{Y}$ is not watermarked, i.e., that $\wt{Y}$ is independent of $\xi$. 
The output $\est{p}$ of $\detect$ is a proper non-asymptotic $p$-value: if $\wt{Y}$ is not watermarked, 
then each $(\wt{Y},\xi^{(t)})$ constitutes an independent, identically distributed copy of $(\wt{Y},\xi)$ and therefore by symmetry $\est{p}$ is 
uniformly distributed over $\{1/(T+1),2/(T+1),\dots,1\}$ for any (non-atomic) test statistic.\footnote{
    By non-atomic, we mean for any $c \in \R$ that $\P(\phi(Y,\xi) = c) = 0$ so that almost surely we will not 
    have to break ties (meaning, if $\phi(y,\xi) = \phi_t$) when computing $\est{p}$.
    In case of ties (i.e., if the test statistic is atomic), we can either modify $\detect$ to break 
    ties uniformly at random, or simply report valid but conservative $p$-values by leaving $\detect$ as is.
}
If $\phi$ returns a small $p$-value (e.g., 0.0001) then the text is likely watermarked;
if the $p$-value is large (e.g., 0.25), then the text might not be.

\begin{algorithm}[h]
    \DontPrintSemicolon
    \SetKwInOut{Input}{Input}
    \SetKwInOut{Output}{Output}
    \SetKwInOut{Params}{Params}
    \SetKwInOut{Noise}{Noise}
    \caption{\label{algorithm:detect}
        Watermarked text detection ($\detect$)
        } 
    \Input{string $y \in \mc{V}^*$, watermark key sequence $\xi \in \Xi^n$}
    \Params{test statistic $\phi$; watermark key sequence distribution $\nu \in \Delta(\Xi^n)$; resample size $T$}
    \Output{p-value $\est{p} \in [0,1]$}
    \For{$t \in 1,\dots,T$}{
        $\xi^{(t)} \sim \nu$\;
        $\phi_t \leftarrow \phi(y,\xi^{(t)})$\;
    }
    $\est{p} \leftarrow \frac{1}{T+1} \left(1 + \sum_{t=1}^T \indic{\phi_t \leq \phi(y,\xi)}\right)$\;
    \Return $\est{p}$\;
\end{algorithm}

The goal then is to design the test statistic $\phi$ (Algorithm~\ref{algorithm:test}) such that $\est{p}$ will typically be small 
if $\wt{Y}$ is watermarked.
In particular, the goal is to identify an alignment cost $d : (\mc{V} \times \Xi)^* \to \R$, which measures the quality of a match between a subsequence of the input text and a subsequence of the watermark key, and use this to define $\phi$ as the minimum cost alignment between length $k$ subsequences of the text and key.

This alignment-based detection strategy makes the watermark robust, since even if the user crops or otherwise corrupts $Y$, a single block of preserved watermarked text within
some larger body of unwatermarked text will suffice to trigger a low $p$-value from $\detect$.
The actual form of the alignment cost will be specific to each watermark---in particular, it will depend on the nature of 
the decoder $\decode$ in $\generate$.
Our most robust watermarks incorporate a soft notion of edit distance (i.e., Levenshtein distance) 
into the computation of the alignment cost via dynamic programming, 
with runtime scaling quadratically in the block size.
Thus, letting $m$ be the length of 
the input text $y$, $n$ be the length of the watermark key sequence $\xi$, and $k$ be the block size,
the cost of computing the test statistic is $O(m n k^2)$.


\begin{algorithm}[h]
    \DontPrintSemicolon
    \SetKwInOut{Input}{Input}
    \SetKwInOut{Output}{Output}
    \SetKwInOut{Params}{Params}
    \SetKwInOut{Noise}{Noise}
    \caption{\label{algorithm:test}
        Test statistic ($\phi$)
        } 
    \Input{string $y \in \mc{V}^*$, watermark key sequence $\xi \in \Xi^n$}
    \Params{alignment cost $d$, block size $k$}
    \Output{test statistic value $\phi(y,\xi) \in \R$}
    \For{$i \in 1,\dots,\len{y}-k + 1$}{
        \For{$j \in 1,\dots,n$}{
            $y^i \leftarrow \{y_{i+\ell}\}_{\ell = 0}^{k-1}$,
            $\xi^j \leftarrow \{\xi_{(j+\ell) \% n}\}_{\ell=0}^{k-1}$\;
            $\est{d}_{i,j} \leftarrow d(y^i,\xi^j)$\;
        }
    }
    \Return $\min_{i,j} \est{d}_{i,j}$
\end{algorithm}

To illustrate how the decoder and the alignment cost fit together, we give a simple example
for the toy setting of a binary vocabulary. \\

\begin{example}\label{example:binary}
    Consider a binary vocabulary $\mc{V} = \{0,1\}$.
      To generate $Y \in \{0,1\}^*$ from the model, 
      the LM provider shares
      $\{\xi_i\}_{i=1}^n \iid \txt{Unif}([0,1])$ with the detector
      and let $Y_i = 0$ if $\xi_i \leq p(0 \mid Y_{:i-1})$ and $Y_i = 1$ otherwise.
      In particular, defining the decoder $\decode$ 
      by
      \begin{align*}
          \decode(\xi_i,\mu) \defeq
              \begin{cases} 
                  0 & \xi_i \leq \mu(0) \\
                  1 & \xi_i > \mu(0),
              \end{cases}
      \end{align*}
      let $Y = \generate(\xi;m,p,\decode)$ for some $m \leq n$.
      Then $Y$ is a valid sample from the language model as 
      $\P(\xi_i \leq p(0 \mid Y_{:i-1})) = p(0 \mid Y_{:i-1})$,
      and crucially $Y$ and $\xi$ are correlated (i.e., if $\xi_i$ is sufficiently close to zero then $Y_i = 0$, 
      and likewise if $\xi_i$ is sufficiently close to one then $Y_i = 1$). 
      Thus, we can define the alignment cost $d(y,\xi) = \norm{y - \xi}_1$.

      Assuming for the sake of this example that $n = m$ and the user does not corrupt 
      the watermarked text from the LM provider, i.e., $\wt{Y} = Y$, the detector can run $\detect$
      to verify that $\wt{Y}$ is watermarked using the test statistic $\phi$ with alignment cost $d$ and
      block size $k = m$. The value of the test statistic will then be at most 
      the $\ell_1$ norm of $\wt{Y} - \xi$.
\end{example}

\subsection{Handling multiple queries}\label{sec:multi-query}
In the above example, the LM provider generates the same text each time from the watermark key sequence,
which is not ideal in practice.
One solution for avoiding reusing elements of the watermark key sequence across queries is to make $\generate$ stateful,
thus enabling the LM provider to generate a total of $\lfloor n/m \rfloor$ independent watermarked text samples of $m$ tokens each
from the language model.
Instead, to avoid persisting state, we provide a randomized 
wrapper $\shuffle$ (Algorithm~\ref{algorithm:shuffle}) around $\generate$ and
modify the watermarking protocol from the start of the section to allow the LM provider to call the $\shuffle$ 
instead of $\generate$ in the second step of the protocol.
The wrapper $\shuffle$ randomly shifts the watermark key sequence before passing the shifted sequence to $\generate$.
Shifting the watermark key sequence does not affect the value of the test statistic in $\detect$,
since to compute the test statistic the detector anyways searches over all subsequences of the watermark key sequence 
to find the best match for each block of text.
There are $n$ possible shifts, each of which may produce a distinct text;
while in principle these $n$ texts will correlate with each other due to sharing elements of 
the watermark key sequence, in practice we find the effects of these correlations are not noticeable.
The so-called birthday paradox \cite{elfving1966selected} implies the LM provider can typically expect to call $\shuffle$ $\Omega(\sqrt{n})$ times,
each time producing a different text, before reusing the same offset twice.
In fact, the provider can expect call $\shuffle$ $\Omega(\sqrt{n/m})$ times before reusing a subsequence, in which case 
the constituent $\Omega(\sqrt{nm})$ tokens in these texts
will be indistinguishable from regular samples from the language model.

In general, we can bound the distortion (i.e., statistical distance from regular samples) of $\shuffle$ in the multi-query setting by 
the probability of reusing an element of the watermark key sequence.
Specifically, for $T$ queries and a maximum generation length of $m$ tokens per query,
we will achieve negligible, i.e., $o(1)$ distortion, so long as $n = \omega(mT^2)$.
Thus, similar to \citet{christ2023undetectable} we can achieve approximate distortion-freeness in the multi-query setting.
However, unlike \citet{christ2023undetectable}, to achieve approximate distortion-freeness in this setting the computational cost of our watermark detection procedure must 
grow with the target number of queries.
In principle, this trade-off between the degree of distortion and the runtime of watermark detection means at least in an asymptotic sense that the latter
effectively upper bounds the number of queries to the LM provider an attacker would require to learn information about the watermark key sequence.
In practice, we expect the cost per token of queries to the LM provider will be significant enough to make such attacks expensive to implement.



\begin{algorithm}[h]
    \DontPrintSemicolon
    \SetKwInOut{Input}{Input}
    \SetKwInOut{Output}{Output}
    \SetKwInOut{Params}{Params}
    \SetKwInOut{Noise}{Noise}
    \caption{\label{algorithm:shuffle}
        Randomized watermarked text generation ($\shuffle$)}
    \Input{watermark key sequence $\xi \in \Xi^n$}
    \Params{generation length $m$, language model $p$, decoder $\decode$}
    \Output{string $y \in \mc{V}^m$}
    $\tau \sim \txt{Unif}([n])$,
    $\xi' \leftarrow \{\xi_{(i+\tau)\% n}\}_{i=1}^{m}$\;
    \Return $\generate(\xi';m,p,\decode)$\;
\end{algorithm}

\subsection{Terminology: watermark strategies and watermark potential}
Henceforth, we use the term \emph{watermarking strategy} to refer to a concrete instantiation of the 
$\shuffle$, $\generate$ and $\detect$ methods by specifying 
the internal parameters of both algorithms (i.e., the decoder $\decode$, the test statistic $\phi$ and the watermark key sequence distribution $\nu$).
We give concrete watermarking strategies in the following sections (Sections~\ref{sec:watermark} and \ref{sec:gumbel}).
For each watermarking strategy, we show two main results: we prove the decoder is distortion-free and also 
obtain high probability upper bounds on the $p$-values of watermarked text---as a function of the length 
of the text and the watermark key sequence.
We emphasize that only the former result (i.e., that the decoder is distortion-free) is critical 
to the validity of our main claims; we intend the latter collection of results 
to provide intuition for when we would expect the detector to have sufficient power and to anticipate the 
forthcoming experimental results in Section~\ref{sec:experiments}.
The strength of the $p$-value upper bounds will depend on the observed token probabilities of (watermarked) text,
through a quantity which we evocatively term the \textit{watermark potential}. 

\begin{definition}\label{defn:water-ent}
    (watermark potential) Define $\alpha : \mc{V}^* \to \R$ by 
    \begin{align*}
        \alpha(y) \defeq 1 - \frac{1}{\len{y}} \sum_{i = 1}^{\len{y}} p(y_i \mid y_{:i-1}).
    \end{align*}
\end{definition}

The watermark potential of text from a deterministic language model will always be zero,
whereas for a high-entropy model it will approach one.
The degree to which it is possible for the detector to reliably distinguish watermarked text from unwatermarked 
text necessarily depends on the watermark potential of the LM provider's language model.
For example, if the language model is deterministic, then any distortion-free watermark will necessarily have zero statistical power.
We formalize this intuition by establishing the following general lower bound on the detection accuracy of any watermarking strategy
as a function of the watermark potential of the original language model.
In particular, we lower bound the error of any classifier $h : \mc{V}^* \times \Xi^* \to \{-1,+1\}$ 
that tries to distinguish watermarked (positive label) versus nonwatermarked text (negative label)
given some watermark key $\xi$ (we make no assumption on the distribution of $\xi$ except that it is independent 
of unwatermarked text by definition).
We defer the proof of Lemma~\ref{lemma:lower-bound-all} to Appendix~\ref{app:lb}.

\begin{lemma}\label{lemma:lower-bound-all}
    Let $Y_i' \sim p(\cdot \mid Y_{:i-1}')$ for $i \in [m]$.
    Let $Y \stackrel{d}{=} Y'$ and let $\xi \in \Xi^*$ be a random variable that is 
    independent of $Y'$. Let $h : \mc{V}^* \times \Xi^* \to \{-1,+1\}$ be a classifier.
    Let $c > 0$ and define the set $\mc{V}_c \subset \mc{V}^m$ by 
    \begin{align*}
        \mc{V}_c \defeq \{y : p(y_i \mid y_{:i-1}) \geq \exp(-c/2) \ \txt{for all} \ i \in [m]\}.
    \end{align*}
    Then
    \begin{align*}
        \P(h(Y,\xi) = -1) + \P(h(Y',\xi) = 1) \geq \Ep \left[\exp\left(-c m \alpha(Y)\right)\indic{Y \in \mc{V}_c}\right].
    \end{align*}
\end{lemma}

Lemma~\ref{lemma:lower-bound-all} implies it is impossible to test between 
any watermarked and non-watermarked text (i.e., between $Y$ versus $Y'$) that are 
equal in distribution (i.e., distortion-free) if the text typically has low watermark potential,  
irrespective of the design of the watermark key;
in particular, the sum of the 
Type I and II (resp., false positive/negative) error rates of $h$ will be close to one if the watermark potential is close to zero.
The theorem is not tight: depending on the language model, its result may be vacuous for small 
values of $c$ 
(e.g., the constants which appear in our upper bounds) since only texts 
whose token likelihoods all exceed $\exp(-c/2)$ contribute to the lower bound.
Also our upper bounds scale inverse exponentially with the square of 
the watermark potential, which will always be smaller 
than the watermark potential itself since the watermark potential is bounded between 
zero and one.


The point of the forthcoming $p$-value upper bounds for the watermarking strategies in Sections~\ref{sec:watermark} and \ref{sec:gumbel} 
is to establish the existence of test statistics for each watermark such that
the statistical power of the watermark improves exponentially with the length of the text and decays at most 
linearly with the length of the watermark key sequence.
The test statistics we use to prove these upper bounds differ slightly from those we employ in our experiments: 
in the former case, we prioritize the simplicity of stating the bounds in terms of watermark potential, 
whereas in the latter case, we prioritize empirical performance. 


\subsection{Watermarking via inverse transform sampling}\label{sec:watermark}
Inverse transform sampling is a general technique for sampling from a univariate distribution 
by taking the pushforward of a uniform
random variable through its inverse cumulative distribution function (CDF).
Crucially, the technique is valid irrespective of the ordering of the CDF, a property which we presently leverage 
to construct a watermarking strategy in which $\generate$ is distortion-free and also $\detect$ is agnostic.
In particular, we implement $\generate$ with a decoder that maps a sequence of uniform random variables and permutations to 
tokens using inverse transform sampling.
To detect watermarked text, the detector correlates the sequence of permuted indices of the tokens in the text
with the sequence of uniform random variables to detect watermarked text.
Meanwhile, for any nonwatermarked text, the sequence of permuted token indices will be i.i.d. uniform irrespective of the 
text itself and thus not correlate with the sequence of uniform random variables.

Formally, with $\Pi$ as the space of permutations over the vocabulary $[N]$,
for $\xi = (u,\pi) \in [0,1] \times \Pi =: \Xi$ and any distribution $\mu \in \Delta([N])$, define the decoder by
\begin{align}
  \decode(\xi,\mu ) \defeq \pi^{-1}\left(\min \left\{\pi(i) : 
    \mu(\{j:\pi(j)\leq\pi(i)\}) \geq u\right\}\right), \label{eqn:rho-defn}
\end{align}
i.e., $\decode(\xi,\mu)$ is the token with the smallest index in the permutation $\pi$ such that 
CDF of $\mu$ with respect to $\pi$ is at least $u$.
Generalizing the intuition from Example~\ref{example:binary}, we show 
this decoder is distortion-free in the following theorem.
\begin{theorem}\label{thm:stealth}
    Define $\decode$ by equation~\eqref{eqn:rho-defn}.
    Let $\pi \in \Pi$ be arbitrary and let
    $U \sim \txt{Unif}([0,1])$, with 
    $\xi \defeq (U,\pi)$.
    Then $\decode$ is distortion-free with respect to $\xi$.
\end{theorem}
\begin{proof}
    Recalling Definition~\ref{defn:distortion-free}, the result follows from showing for any $\mu \in \Delta([N])$
    and $y \in [N]$ that $\P(\decode(\mu,\xi) = y) = \mu(y)$.
    To this end, by equation~\eqref{eqn:rho-defn},
    we have $\decode(\mu,\xi) = y$ 
    if and only if $U$ lies in the interval
    \begin{align*}
        \left[\mu(\{y':\pi(y')<\pi(y)\}),\mu(\{y':\pi(y')\leq\pi(y)\})\right).
    \end{align*}
    As the width of this interval is exactly $\mu(y)$, the result follows immediately.
\end{proof}

Having shown that the decoder is distortion-free, we now proceed to analyze the detectability of the watermark.
For convenience, define the normalization $\eta : [N] \to [0,1]$ by 
$\eta(i) \defeq (i-1)/(N-1)$.
Analogous to the toy example, the sequences $\{\eta(\pi_i(Y_i))\}_{i=1}^m$ and $U$ are correlated.
Thus, for the sake of analysis, we define alignment cost $d : (\mc{V} \times \Xi)^* \to \R$ by
\begin{align}
    d(y,(u,\pi)) \defeq -\sum_{i=1}^{\len{y}} (u_i - 1/2) \cdot (\eta(\pi_i(y_i)) - 1/2) \label{eqn:d-defn},
\end{align}
i.e., the negative covariance (each $U_i$ and $\eta(\pi_i(Y_i))$ both have expectation $1/2$).


We exactly characterize in Lemma~\ref{lemma:sound-transform} the difference in the expected value of our alignment cost 
on some text assuming the text is watermarked (i.e., generated using the same key as the detector)
versus not watermarked in terms of the watermark potential of the text (Definition~\ref{defn:water-ent}).
To state the result, we define the constant 
$C_0 \defeq \txt{Var}(\eta(\txt{Unif}([N])))$, where we abuse notation slightly 
to temporarily treat $\eta$ as a pushforward map over distributions.\footnote{
    Note that $C_0 = \txt{Var}(\txt{Unif}([0,1])) + o_N(1) = 1/12 + o_N(1)$.
}
We defer the proof of Lemma~\ref{lemma:sound-transform} to Appendix~\ref{app:transform}.

\begin{lemma}\label{lemma:sound-transform}
    Let $m,n \in \N$ with $n \geq m$, where $m$ is the generation length and $n$ is the watermark key length. 
    Define the decoder $\decode$ by equation~\eqref{eqn:rho-defn} 
    and the alignment cost $d$ by equation~\eqref{eqn:d-defn}.
    Let $\xi,\xi' \iid \txt{Unif}(\Xi^n)$ with $Y = \generate(\xi;m,p,\decode)$.
    Then almost surely for all $i \in [m]$ and $j \in [n]$ we have
    \begin{align*}
        \Ep[d(Y_i,\xi_j') - d(Y_i,\xi_i) \mid Y] = C_0 \cdot (1 - p(Y_i \mid Y_{:i-1})) = C_0 \alpha(Y_{i-1:i}).
    \end{align*}
\end{lemma}


Summing the result of Lemma~\ref{lemma:sound-transform} over $i \in [m]$ implies
for any $j \in [n]$ that
\begin{align*}
    \Ep[d(Y,\xi_{(j+1:j+m)\% n}') - d(Y,\xi_{1:m}) \mid Y] = C_0 m \alpha(Y).
\end{align*}
Thus,
we can upper bound the $p$-value output by $\detect$ in Lemma~\ref{lemma:perm-test} 
using a standard concentration argument and taking a union bound over $j \in [n]$.
We defer the proof of Lemma~\ref{lemma:perm-test} to Appendix~\ref{app:transform}.
In fact, we actually prove a more general result for $k \leq m$ wherein we allow $\wt{Y}$ to be
a subsequence of $Y$ which the user may choose adaptively.
We defer this more general result to Appendix~\ref{app:transform} as it is more cumbersome to state.

\begin{lemma}\label{lemma:perm-test}
    Let $m,n \in \N$ with $n \geq m$, where $m$ is the generation length and $n$ is the watermark key length. 
    Define the decoder $\decode$ by equation~\eqref{eqn:rho-defn}, alignment cost $d$ by equation~\eqref{eqn:d-defn},
    and $\phi$ by Algorithm~\ref{algorithm:test} with block size $k = m$.
    Let $\xi,\xi' \iid \txt{Unif}(\Xi^n)$ with $Y = \generate(\xi;n,p,\decode)$ and
    $\wt{Y} = Y$.
    Then almost surely
    \begin{align*}
        \P(\phi(\wt{Y},\xi') \leq \phi(\wt{Y},\xi) \mid \wt{Y}) \leq 2n\exp(-k C_0^2 \alpha(\wt{Y})^2/2).
    \end{align*}
\end{lemma}
Lemma~\ref{lemma:perm-test} implies that with high probability the value of the test statistic on watermarked text with the correct key 
will be lower than with a resampled key.
In particular, ignoring discretization errors due to the finite number of resamples $T$ in $\detect$, 
the lemma implies watermarked samples with watermark potential bounded 
away from zero (i.e., if the language model is not effectively deterministic) 
will have exponentially small expected $p$-values with respect to the length $m$ of the text.
The bound grows only linearly with the length $n$ of the random number sequence,
implying for moderately large $m$ (e.g., $m = 50$) an LM provider can generate plenty of distortion-free watermarked text
(i.e., $n = 2^{\Omega(m)}$ total tokens) while still enabling detection of the watermark from snippets of $m$ tokens (e.g.,
$50$ tokens typically amount to a couple sentences of text).
Of course, recall the computational complexity of detection scales linearly with $n$, 
which in practice may be a more relevant limitation than the statistical power of the watermark.\footnote{
    Note that both $\detect$ and the test statistic (Algorithm~\ref{algorithm:test}) are easily parallelizeable.
}

\subsubsection{Robustness to substitutions, insertions and deletions}
We show in Lemma~\ref{lemma:perm-test-subs} an analogous result to Lemma~\ref{lemma:perm-test} holds 
even if an adversary corrupts the original watermarked text by substituting tokens.
To state the lemma, we introduce a quantity $\wt{\alpha}$ which depends on both the corrupted and original watermarked text and 
accounts for the decrease in the expected value of 
the test statistic (which recall for the original text is equal up to a numerical constant to the watermark potential of the text)
due to token substitutions.
We defer the proof of Lemma~\ref{lemma:perm-test-subs} to Appendix~\ref{app:transform}.


\begin{lemma}\label{lemma:perm-test-subs}
    Let $m,n \in \N$ with $n \geq m$, where $m$ is the generation length and $n$ is the watermark key length.
    Define the decoder $\decode$ by equation~\eqref{eqn:rho-defn}, alignment cost $d$ by equation~\eqref{eqn:d-defn},
    and $\phi$ by Algorithm~\ref{algorithm:test} with $k = m$.
    Let $\xi,\xi' \iid \txt{Unif}(\Xi^n)$ with $Y = \generate(\xi;m,p,\decode)$
    and let $\wt{Y} \in \mc{V}^m$ be conditionally independent 
    of $\xi$ and $\xi'$ given $Y$. Define 
    \begin{align*}
        \wt{\alpha}(y,\wt{y}) \defeq \frac{1}{\len{y}} \sum_{i=1}^{\len{y}} \indic{y_i = \wt{y}_i} \left(1 - p(y_i \mid y_{:i-1}) \right) - \indic{y_i \neq \wt{y}_i} \frac{1}{N-1}.
    \end{align*}
    Then almost surely
    \begin{align*}
        \P(\phi(\wt{Y},\xi') \leq \phi(\wt{Y},\xi) \mid Y,\wt{Y}) \leq 2n\exp(-k C_0^2 \wt{\alpha}(Y,\wt{Y})^2/2).
    \end{align*}
\end{lemma}

Lemma~\ref{lemma:perm-test-subs} implies that even if an adversary replaces the vast 
majority of tokens in a watermarked text, detection with low $p$-values will still be possible 
so long as the remaining tokens have watermark potential bounded away from zero.
In particular, the permuted indices of the original tokens will still positively correlate with 
the corresponding uniform random variables from the watermark key sequence, while those of the substituted tokens 
will exhibit a small negative correlation scaling as $O(1/N)$.

To handle insertions and deletions, we can robustify our test statistic by incorporating
a soft notion of edit distance into our original alignment cost.
The parameter $\gamma$ in Definition~\ref{defn:levenshtein} assigns a cost to 
each insertion and deletion operation when aligning the tokens $y$ with the sequence $\xi$,
while the base alignment cost $d_0$ defines the quality of the alignment via a cost
function over substitutions.
In practice, we drop the minimizations over $y' \in \mc{V}$ and $\xi' \in \Xi$ in the second and third cases respectively 
of the definition;
we include them here to make our subsequent theoretical analysis cleaner.
\begin{definition}\label{defn:levenshtein}
    (Levenshtein cost) Let $\gamma \in \R$ and $d_0 : \mc{V} \times \Xi \to \R$.
    For $y \in \mc{V}^*$ and $\xi \in \Xi^*$, define the \textit{Levenshtein cost} $d_\gamma : \mc{V}^* \times \Xi^* \to \R$ by 
    \begin{align*}
        d_\gamma(y,\xi) \defeq \min
        \begin{cases} 
            d_\gamma(y_{2:},\xi_{2:}) + d_0(y_1,\xi_1)  \\
            d_\gamma(y,\xi_{2:}) + \min_{y' \in \mc{V}}d_0(y',\xi_1) + \gamma \\
            d_\gamma(y_{2:},\xi) + \min_{\xi' \in \Xi}d_0(y_1,\xi') + \gamma,
        \end{cases}
    \end{align*}
    with $d_\gamma(y,(u,\pi)) \defeq \gamma \cdot \len{y}$ if $\xi$ is empty and vice versa (as base cases).\footnote{
        For $y \in \mc{V}^*$ (resp., $\xi \in \Xi^*$), we let $y_{\len{y}+1:}$ (resp., $\xi_{\len{\xi}+1}$)
        denote the empty string/sequence.
    }
\end{definition}

Redefining the test statistic $\phi$ using $d_\gamma$ as the alignment cost---using $d_0$ from equation~\eqref{eqn:d-defn}---ensures 
$\detect$ is robust not only to substituting tokens, 
but also inserting and deleting tokens from watermarked text,
as we show in Lemma~\ref{lemma:perm-test-indel}.
We defer the proof of Lemma~\ref{lemma:perm-test-indel} to Appendix~\ref{app:transform}.
To state the lemma, we first recursively define a notion of edit distance between two strings.
The definition is equivalent to the minimum number of insertion and/or deletion operations 
needed to transform one string into the other (see Lemma~\ref{obs:edit-dist-basic}).

\begin{definition}\label{defn:edit-dist}
    (edit distance)
    For $y,\wt{y} \in \mc{V}^*$, define the \textit{edit distance} by 
    \begin{align*}
        d_\txt{edit}(y,\wt{y}) \defeq
        \begin{cases} 
            d_\txt{edit}(y_{2:},\wt{y}_{2:}) & y_1 = \wt{y}_1 \\
            1 + \min \{d_\txt{edit}(y_{2:},\wt{y}),d_\txt{edit}(y,\wt{y}_{2:})\} & y_1 \neq \wt{y}_1,
        \end{cases}
    \end{align*}
    with $d_\txt{edit}(y,\wt{y}) = \len{y}$ if $\wt{y}$ is empty and vice versa.
\end{definition}


\begin{lemma}\label{lemma:perm-test-indel}
    Let $n,m \in \N$ with $n \geq m$, where $m$ is the generation length and $n$ is the watermark key length.
    Define the decoder $\decode$ by equation~\eqref{eqn:rho-defn},
    alignment cost $d = d_\gamma$ with $d_0$ from equation~\eqref{eqn:d-defn} and $\gamma > 1/2$,
    and $\phi$ by Algorithm~\ref{algorithm:test} using block size $k \leq m$ that divides evenly into $m$.
    Let $\xi,\xi' \iid \txt{Unif}(\Xi^n)$ with $Y = \generate(\xi;m,p,\decode)$.
    Let $\wt{Y} \in \mc{V}^m$ be conditionally independent 
    of $\xi$ and $\xi'$ given $Y$, with $d_\txt{edit}(Y,\wt{Y}) \leq \eps m$.
    Then almost surely
    \begin{align*}
        \P(\phi(\wt{Y},\xi') \leq \phi(\wt{Y},\xi) \mid \wt{Y},Y) \leq m n(2k)^{k/(4\gamma-1)}\exp(-k C_0^2 (\alpha(Y) - \gamma\eps)_+^2/2).
    \end{align*}
\end{lemma}

We prove the result by showing there must exist a length $k$ substring of the corrupted text $\wt{Y}$ 
within edit distance $k \eps$ of a substring of $Y$ that the detector will be able to distinguish as watermarked.
For fixed $k$, the set of strings within edit distance $\eps k$ of an original block
watermarked text blows up combinatorially with $\eps$.
To ensure we can detect the watermark, the result implies we must set $\gamma = \Omega(1/\eps)$,
which means our bound on the expected $p$-value is vacuous as soon as $\eps = \Omega(1/\log k)$.
Admittedly, our analysis is not tight; for example, as a preview of the experimental results to come, 
in practice we find smaller values of $\gamma$ (i.e., $\gamma < 1$) to perform 
significantly better.
However, one takeaway from the result is that using a block size $k < m$,
where here $m$ is the length of the input text, for detection can be an effective strategy
when the user has substantially corrupted the text.
The assumption that $k$ divides evenly into $m$ is an artifact of our analysis and not important in practice.

\subsubsection{What we run in practice}


In practice, to reduce overhead in both $\generate$ and $\detect$, 
we use a single random permutation\footnote{
    In principle, with a single random permutation the permuted token indices of both watermarked and nonwatermarked text 
    are no longer conditionally independent of each other, and so the results of Lemmas~\ref{lemma:perm-test},
    \ref{lemma:perm-test-subs} and \ref{lemma:perm-test-indel} no longer apply.
    However, in practice we observe no degradation in statistical power. Also, irrespective of the lemmas,
    the $p$-values from $\detect$ are still valid by construction.
}
instead of a full sequence, i.e., we let
$\pi_i = \pi$ for all $i \in [n]$ for $\pi \sim \txt{Unif}(\pi)$. 
Recall Theorem~\ref{thm:stealth} makes no assumption about the distribution of the permutations;
thus, the watermark is still distortion-free.
Also, for the test statistic, we find using
\begin{align}
    d(y,(u,\pi)) \defeq \sum_{i=1}^{\len{y}} |u_i -  \eta(\pi_i(y_i))| \label{eqn:d-practice}
\end{align}
as the alignment cost performs better empirically than the alignment cost in equation~\eqref{eqn:d-defn}.
To reiterate, the output of $\detect$ is a valid $p$-value irrespective of the test statistic we use.

Henceforth, we refer to this version of the watermarking strategy as {\its},
and we refer to the corresponding Levenshtein version as {\itsedit}, wherein we define 
the base alignment cost $d_0$ by equation~\eqref{eqn:d-practice}
and use the following simplified notion of Levenshtein cost:

\begin{definition}\label{defn:levenshtein-simple}
    (simple Levenshtein cost) Let $\gamma \in \R$ and $d_0 : \mc{V} \times \Xi \to \R$.
    For $y \in \mc{V}^*$ and $\xi \in \Xi^*$, define the alignment cost function $d_\gamma : \mc{V}^* \times \Xi^* \to \R$ by 
    \begin{align*}
        d_\gamma(y,\xi) \defeq \min
        \begin{cases} 
            d_\gamma(y_{2:},\xi_{2:}) + d_0(y_1,\xi_1)  \\
            d_\gamma(y,\xi_{2:}) + \gamma \\
            d_\gamma(y_{2:},\xi) + \gamma,
        \end{cases}
    \end{align*}
    with $d_\gamma(y,(u,\pi)) \defeq \gamma \cdot \len{y}$ if $\xi$ is empty and vice versa (as base cases).\footnote{
        For $y \in \mc{V}^*$ (resp., $\xi \in \Xi^*$), we let $y_{\len{y}+1:}$ (resp., $\xi_{\len{\xi}+1}$)
        denote the empty string/sequence.
    }
\end{definition}
In summary, for {\its} we use the decoder from equation~\eqref{eqn:rho-defn}, the test statistic 
from Algorithm~\ref{algorithm:test} with the alignment cost from equation~\eqref{eqn:d-practice}, and the watermark 
key distribution as the uniform distribution over $[0,1]^n \times \Pi$, where recall $n$ is the length of 
the watermark key sequence.
Meanwhile, {\itsedit} differs from {\its} only in that we define the test statistic using the Levenshtein cost from 
Definition~\ref{defn:levenshtein-simple} with the base cost again from equation~\eqref{eqn:d-practice}.

\subsection{Watermarking via exponential minimum sampling}\label{sec:gumbel}
\citet{aaronson2023openai} proposes mapping variables in $[0,1]^N$
to tokens in the vocabulary $[N]$ using exponential minimum sampling
to generate watermarked text.
Whereas \citet{aaronson2023openai} proposes the use of distortion-inducing hashes much like \citet{kirchenbauer2023watermark},
we use exponential minimum sampling to implement the decoder in $\generate$, which (after defining a suitable corresponding 
test statistic) enables an alternative distortion-free and robust watermarking strategy to inverse transform sampling.
In particular, for $\xi \in [0,1]^N =: \Xi$ and $\mu \in \Delta([N])$, define the decoder by
\begin{align}\label{eqn:rho-defn-gumbel}
    \decode(\xi,\mu) \defeq \argmin_{i \in [N]} \ -\log (\xi_i)/\mu(i).
\end{align}
We show this decoder is distortion-free in Theorem~\ref{thm:stealth-gumbel},
whose proof we defer to Appendix~\ref{app:gumbel}.
\begin{theorem}\label{thm:stealth-gumbel}
    Define the decoder $\decode$ by equation~\eqref{eqn:rho-defn-gumbel} and
    let $\xi \sim \txt{Unif}([0,1]^N)$.
    Then $\decode$ is distortion-free with respect to $\xi$.
\end{theorem}

For the sake of analysis, we define the alignment cost 
as a slight variation of the proposal of \citet{aaronson2023openai} (see Section~\ref{sec:gumbel-practice}) by
\begin{align}
    d(y,\xi) \defeq -\sum_{i=1}^{\len{y}} \log \xi_{i,y_i}, \label{eqn:d-defn-gumbel}
\end{align}
again defining the test statistic $\phi$ by Algorithm~\ref{algorithm:test}.
Similar to Lemma~\ref{lemma:sound-transform} for ITS, we exactly characterize 
the difference in the expected values of the alignment cost on watermarked versus non-watermarked
text in terms of the watermark potential of the text.
We defer the proof of Lemma~\ref{lemma:sound-gumbel} to Appendix~\ref{app:gumbel}.

\begin{lemma}\label{lemma:sound-gumbel}
    Let $n \in \N$. Define $\decode$ by equation~\eqref{eqn:rho-defn-gumbel} 
    and $d$ by equation~\eqref{eqn:d-defn-gumbel}.
    Let $\xi,\xi' \iid \txt{Unif}(\Xi^n)$ with $Y = \generate(\xi;n,p,\decode)$.
    Then almost surely for all $i \in [n]$ we have
    \begin{align*}
        \Ep[d(Y_i,\xi_i') - d(Y_i,\xi_i) \mid Y] = 1 - p(Y_i \mid Y_{:i-1}) = \alpha(Y_{i-1:i}). 
    \end{align*}
  \end{lemma}

Summing the result of Lemma~\ref{lemma:sound-gumbel} over $i \in [m]$ implies
for any $j \in [n]$ that
\begin{align*}
    \Ep[d(Y,\xi_{(j+1:j+m)\% n}') - d(Y,\xi_{1:m}) \mid Y] = m \alpha(Y).
\end{align*}
Thus, defining the test statistic $\phi$ 
by Algorithm~\ref{algorithm:test} with respect to the alignment cost $d$ from Eqn~\eqref{eqn:d-defn-gumbel},
we can again upper bound the $p$-value output by $\detect$ in Lemma~\ref{lemma:perm-test-gumbel} 
using a standard concentration argument and taking a union bound over $j \in [n]$.
We defer the proof of Lemma~\ref{lemma:perm-test-gumbel} to Appendix~\ref{app:gumbel}.
Once again, we actually prove a more general result that allows $\wt{Y}$ to be any length $k$ subsequence of $Y$.

\begin{lemma}\label{lemma:perm-test-gumbel}
    Let $m,n \in \N$ with $n \geq m$. Define $\decode$ by equation~\eqref{eqn:rho-defn-gumbel}, $d$ by equation~\eqref{eqn:d-defn-gumbel},
    and $\phi$ by Algorithm~\ref{algorithm:test} with $k = m$.
    Let $\xi,\xi' \iid \txt{Unif}(\Xi^n)$ with $Y = \generate(\xi;n,p,\decode)$ and $\wt{Y} = Y$.
    Then almost surely
    \begin{align*}
        \P(\phi(\wt{Y},\xi') \leq \phi(\wt{Y},\xi) \mid \wt{Y}) \leq 2n\exp\left(-\min\{m\alpha(\wt{Y})^2/8,m\alpha(\wt{Y})/4\}\right).
    \end{align*}
\end{lemma}


\subsubsection{Robustness to corruptions}
Showing high probability $p$-value upper bounds for corruptions of watermarked text 
that hold almost surely given the corrupted text---i.e., analogues of Lemmas~\ref{lemma:perm-test-subs} and \ref{lemma:perm-test-indel}---is more difficult, 
primarily due to the fact that the summands in the alignment metric from equation~\eqref{eqn:d-defn-gumbel} are no longer bounded and thus bounding the influence 
of each substitution and/or insertion operation on the test statistic requires more careful analysis.
Of course, we could in principle tweak the alignment metric 
by truncating the summands in order to prove the analogous results;
however, as the main intuitions would carry over from Lemmas~\ref{lemma:perm-test-subs} and \ref{lemma:perm-test-indel} 
and the results are not critical to the main thrust of the paper, we do not carry this plan out.

\subsubsection{What we run in practice}\label{sec:gumbel-practice}

As in the case of ITS, in practice we find using a slight variation of the 
alignment cost in equation~\eqref{eqn:d-defn-gumbel} performs better.
Namely, following the prescription of \citet{aaronson2023openai},
we modify the previous alignment cost to instead be
\begin{align}
    d(y,\xi) \defeq \sum_{i=1}^k \log (1 - \xi_{i,y_i}). \label{eqn:d-practice-gumbel}
\end{align}
Henceforth, we refer to this version of the watermarking strategy as {\gumb}, 
and we refer to the corresponding Levenshtein version wherein we define 
the base alignment cost $d_0$ by equation~\eqref{eqn:d-practice-gumbel} as {\gumbedit}.

In summary, for {\gumb} we use the decoder from equation~\eqref{eqn:rho-defn-gumbel}, the test statistic 
from Algorithm~\ref{algorithm:test} with the alignment cost from equation~\eqref{eqn:d-practice-gumbel}, and the watermark 
key distribution as the uniform distribution over $\Xi^n$, where recall $n$ is the length of 
the watermark key sequence and $\Xi = [0,1]^N$.
Meanwhile, {\gumbedit} differs from {\gumb} only in that we define the test statistic using the Levenshtein cost from 
Definition~\ref{defn:levenshtein-simple} with the base cost again from equation~\eqref{eqn:d-practice-gumbel}.

\section{Experimental results}\label{sec:experiments}
We empirically validate the statistical power of our watermarking strategies 
(i.e., {\its}, {\itsedit}, {\gumb}, and {\gumbedit})
via experiments with the OPT-1.3B \citep{zhang2022opt} and LLaMA-7B \citep{touvron2023llama} models.\footnote{
    We will also at times collectively refer to {\its} and {\itsedit} as the ITS watermarks and/or strategies and 
    {\gumb} and {\gumbedit} as the EXP watermarks and/or strategies.
}
We run experiments using $\generate$ rather than $\shuffle$, mainly for the sake of reproducibility;
recall however that this choice has no impact on the $p$-values we report.
We test for all watermarks using a block size $k$ (in Algorithm~\ref{algorithm:test}) 
equal to the length $m$ of the text.
Following the methodology of \citet{kirchenbauer2023watermark}, 
we generate watermarked text continuations of prompts 
sampled from the news-like subset of the C4 dataset \citep{raffel2020exploring}.
We vary the generation length $m$ (Experiment 1) and the random number sequence length $n$ (Experiment 2),
and we report median $p$-values of watermarked text over $500$ samples.\footnote{
    The median $p$-value corresponds to the significance level (i.e., Type I error rate) 
    at which the power of our watermark detector is at least $0.5$.
}

We also evaluate robustness to four kinds of paraphrasing attacks:
randomly substituting a fraction of the generated tokens 
with tokens chosen uniformly at random from the vocabulary (Experiment 3);
randomly inserting a fraction of tokens among the generated tokens (Experiment 4);
randomly deleting a fraction of the generated tokens (Experiment 5);
using another language model to translate the text from English to French 
and back (Experiment 6). The first three attacks allow us to systematically vary the level of corruption,
while the last attack is an example of an attack we might encounter in the wild.
We defer the details of the translation procedures to Appendix~\ref{app:experiments-roundtrip}.

Finally, using the Alpaca-7B model and evaluation dataset \cite{taori23alpaca}, 
we conduct a case-study on the feasibility of watermarking 
the responses of a performant instruction-tuned language model to user queries.
We also show for certain kinds of instructions that hashing-based watermarks
produce noticeably worse responses than our distortion-free watermarks, thus underlining 
the importance of the distortion-free property in practice.

In all our experiments---except for Experiment 2, where the control variable $n$ is a hyperparameter that is unique 
to our watermarks---we also replicate the watermark of \citet{kirchenbauer2023watermark} as a baseline, 
setting the greenlist fraction $\gamma = 0.25$ and varying the logit bias $\delta \in \{1.0,2.0\}$. 
We respectively refer to these versions of their watermark as {\kgwone} and {\kgwtwo} after the first three authors' last names.
We emphasize their watermark is not directly comparable to our watermarks as it is not distortion-free
(e.g., \citet{kirchenbauer2023watermark} report that even the weakest version we employ with $\delta = 1.0$ and $\gamma = 0.25$
typically increases perplexity by 5--10\%).

In their work, \citet{kirchenbauer2023watermark} report approximate $p$-values, which they obtain from 
computing the $z$-score of a certain test statistic.
To ensure a fair comparison, we use $\detect$ (with $T = 5000$) to report $p$-values for all watermarks;\footnote{
    This setting of $T$ means we never report $p$-values less than $1/5000$ (i.e., $0.0002$) in any of our experiments. 
} 
in the case of {\kgwone} and {\kgwtwo}, we run $\detect$ using the original inexact $p$-values they report as the test statistic.
We report error bars for the median $p$-value based on a bootstrapped estimate of the standard deviation
using $1000$ resamples.

Instead of recomputing the test statistic $T$ times for each prompt---as we originally prescribe 
in $\detect$---to save computation we simply sample $T$ prompts and compute the test statistic once 
for each ground-truth length $m$ completion;
we then use the empirical distribution of these test statistics as the reference distribution 
within $\detect$, which gives a proper $p$-value with respect to the null hypothesis that 
the text is an original completion from the dataset.
For reference, we include the full pseudocode for this modified version of $\detect$ in Appendix~\ref{app:experiments-pvals},
and we also plot the full distributions of $p$-values for nonwatermarked generations (i.e.,
regular samples from the language models) to verify they are indeed roughly uniform over the interval $[0,1]$.

We defer further details regarding our experimental protocol to Appendix~\ref{app:experiments}.

\subsection{Varying text and watermark key length}
We vary the length $m$ of watermarked text in Figure~\ref{fig:experiment-1}, fixing the watermark key length 
$n = 256$ for each of our watermarks and setting $\gamma = 0.4$
for {\itsedit} and $\gamma = 0.0$ for {\gumbedit} 
(see Appendix~\ref{app:experiments-hyperparam} for the details of tuning $\gamma$).
Our ITS watermarks slightly outperform {\kgwone} while our EXP watermarks slightly outperform {\kgwtwo},
despite the fact that {\kgwone} and {\kgwtwo} both distort the text distribution. 
The EXP watermarks are notably more powerful than the ITS watermarks, requiring roughly two to three times
fewer tokens to achieve a comparably low median $p$-value.
One conceivable advantage 
of the ITS watermarks over the EXP watermarks is that they have comparatively less overhead:
the watermark key for {\gumb} and {\gumbedit} is a sequence of $n$ vectors in $[0,1]^N$, where recall $N$ is the size 
of the vocabulary, while for {\its} and {\itsedit} it is simply a sequence of $n$ numbers in $[0,1]$.
All watermarking strategies perform worse on LLaMA-7B than OPT-1.3B, due to the fact that LLaMA-7B typically produces lower entropy text than OPT-1.3B.
Due to the discrete nature of the test statistic of \citet{kirchenbauer2023watermark}, i.e., the number of tokens 
in the text belonging to a ``greenlist" versus a ``redlist'', the median $p$-values 
for the {\kgwone} and {\kgwtwo} watermarks are occasionally unstable, particularly for small values of $m$.

\begin{figure}
    \centering
    \begin{subfigure}[b]{0.49\textwidth}
      \includegraphics[width=\textwidth]{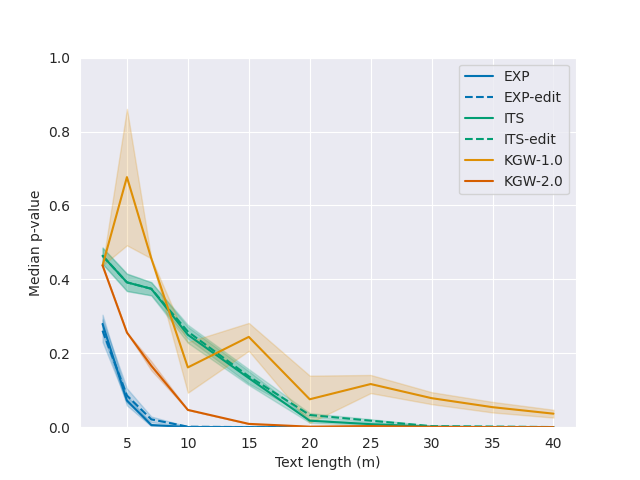}
      \label{fig:experiment-1-opt}
      \caption{OPT-1.3B}
    \end{subfigure}
    \hfill
    \begin{subfigure}[b]{0.49\textwidth}
      \includegraphics[width=\textwidth]{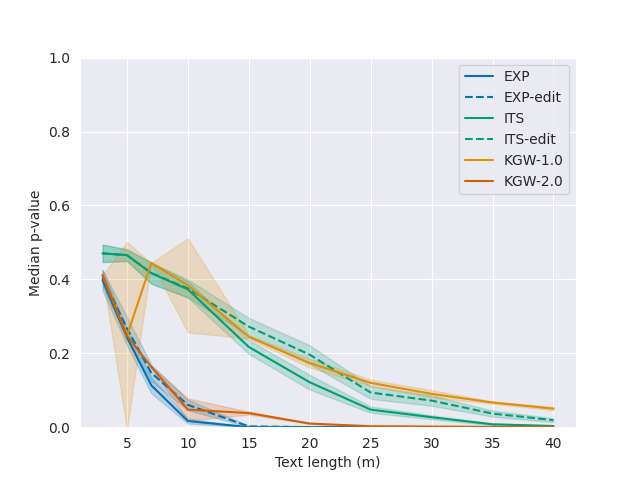}
      \label{fig:experiment-1-llama}
      \caption{LLaMA-7B}
    \end{subfigure}
  
    \caption{Median $p$-value of watermarked text relative to varying the text length $m$, 
    for OPT-1.3B and LLaMA-7B models. Our watermark strategies are competitive with 
    those of \citet{kirchenbauer2023watermark}, despite the fact that they
    distort the text distribution to generate watermarked text whereas we do not.}
    \label{fig:experiment-1}
  \end{figure}
  
We vary the length $n$ of the watermark key sequence $\xi$ in Figure~\ref{fig:experiment-2-combined}
for different lengths $m$ of watermarked text from the ITS and EXP watermarks respectively.
Recall $n$ corresponds to the total number of tokens we can generate while maintaining our 
distortion-free guarantee.
As our theory predicts, the $p$-values of watermarked text grow linearly with $n$.
The rate of growth is fairly mild and decreases rapidly with $m$;
even for $n = 4096$, which is larger than the maximum generation length of both the OPT-1.3B and LLaMA-7B models,
slightly increasing the number of tokens (by 4--8 tokens in the case of EXP, and 10--20 tokens in the case of ITS)
suffices to distinguish watermarked text with roughly the same statistical power as $n = 64$.

\begin{figure}[htbp]
    \centering
    \begin{subfigure}[b]{0.49\textwidth}
        \includegraphics[width=\textwidth]{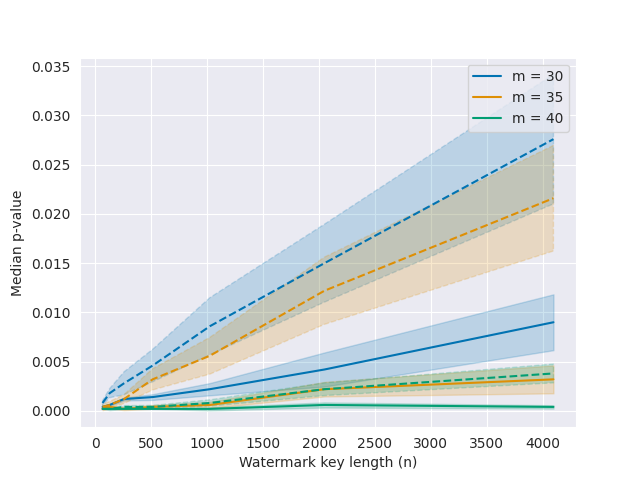}
        \caption{OPT-1.3B: {\its} (solid) and {\itsedit} (dashed)}
        \label{fig:experiment-2-opt-tr}
    \end{subfigure}
    \hfill
    \begin{subfigure}[b]{0.49\textwidth}
        \includegraphics[width=\textwidth]{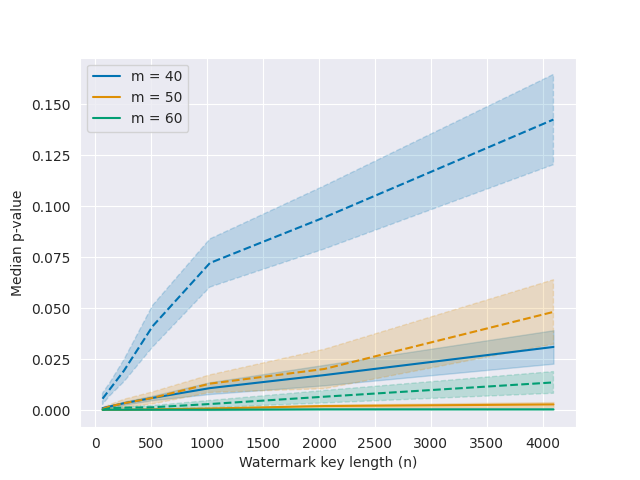}
        \caption{LLaMA-7B: {\its} (solid) and {\itsedit} (dashed)}
        \label{fig:experiment-2-llama-tr}
    \end{subfigure}
    \begin{subfigure}[b]{0.49\textwidth}
      \includegraphics[width=\textwidth]{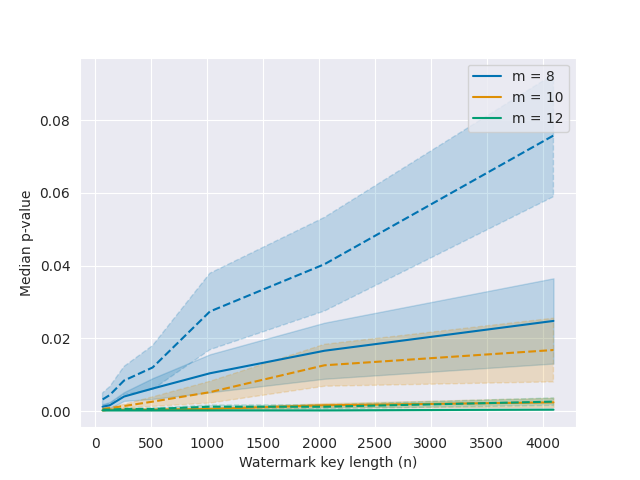}
      \caption{OPT-1.3B: {\gumb} (solid) and {\gumbedit} (dashed)}
      \label{fig:experiment-2-opt-gu}
  \end{subfigure}
    \hfill
    \begin{subfigure}[b]{0.49\textwidth}
        \includegraphics[width=\textwidth]{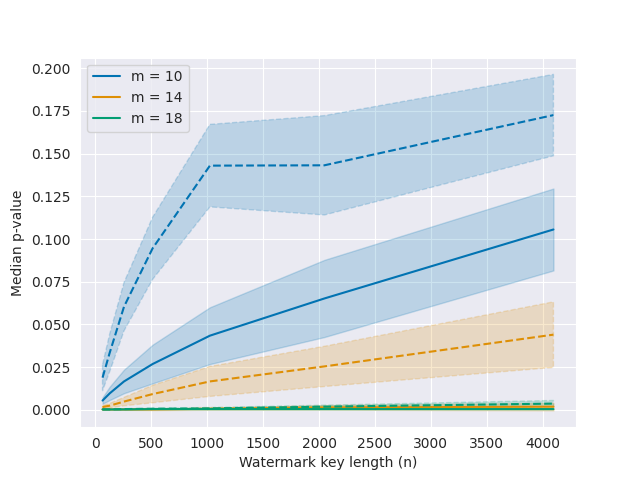}
        \caption{LLaMA-7B: {\gumb} (solid) and {\gumbedit} (dashed)}
        \label{fig:experiment-2-llama-gu}
    \end{subfigure}

    \caption{Median $p$-value of watermarked text for varying the watermark key length $n$.
    Across all watermarks for both the OPT-1.3B and LLaMA-7B models, the median $p$-values grow linearly with $n$ 
    but decay rapidly with increasing text length $m$.}
    \label{fig:experiment-2-combined}
\end{figure}

\subsection{Robustness to corruption and paraphrasing}
We now proceed to evaluate the robustness of our watermark strategies to various forms of corruption and paraphrasing.
We focus on comparing our strongest watermarks ({\gumb} and {\gumbedit}) 
against {\kgwtwo}, deferring results for all other watermarks to Appendix~\ref{app:experiments-defer}.
As larger $n$ increases the computational overhead of computing our test statistics 
and the effect of larger $n$ on statistical power is mild (as shown in Figure~\ref{fig:experiment-2-combined}),
we run all experiments with $n = 256$, which in any case is sufficiently large to ensure the watermarked text across all experiments is distortion-free.
Decreasing the insertion/deletion penalty $\gamma$ improves robustness (at least up to a point) but hurts the statistical power of the {\itsedit}  and {\gumbedit} watermarks for larger $n$, 
since reducing the penalizer for edits effectively increases the number of candidate alignments under consideration.
We run {\itsedit} and {\gumbedit} with the same choices of $\gamma$ as in the previous section.
We defer the details of tuning $\gamma$ to Appendix~\ref{app:experiments-hyperparam}.

We vary the fraction of substituted tokens in Figure~\ref{fig:experiment-3},
and we vary the fraction of inserted and deleted tokens in Figures~\ref{fig:experiment-4} and \ref{fig:experiment-5} respectively.
For the insertion experiment, we pass only the first $m$ tokens to the detector; similarly, for the deletion experiment,
we initially generate more than $m$ watermarked tokens so that even after deleting a fraction thereof, there are at least $m$ tokens remaining.
The {\gumb} and {\gumbedit} watermarks are comparably robust to substitution errors, but the latter is far more robust 
to insertion and deletion errors.

We compare our watermarks against the most robust version of {\kgwtwo}, in the sense that we hash only the previous token 
to determine the next token distribution and thus bias the distribution towards some subset of bigrams.
If instead we hash the previous $k$ tokens for $k > 1$, then substituting any one of the previous $k$ tokens will break 
the watermark signal in a particular token, and thus the statistical power of their watermark will be worse 
than what we report in our experiments.

\begin{figure}
    \centering
    \begin{subfigure}[b]{0.49\textwidth}
      \includegraphics[width=\textwidth]{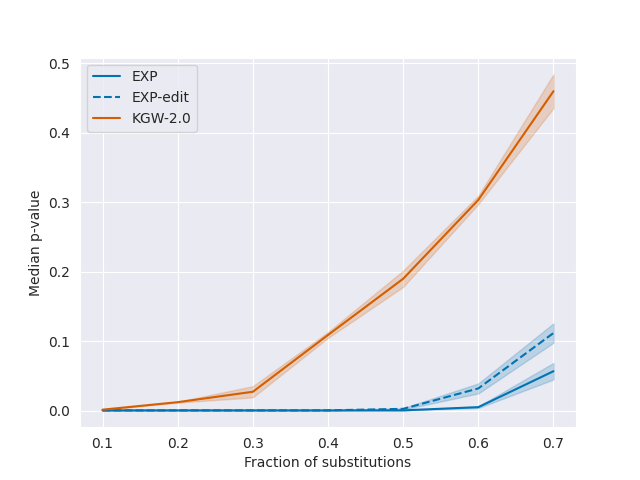}
      \label{fig:experiment-3-opt}
      \caption{OPT-1.3B}
    \end{subfigure}
    \hfill
    \begin{subfigure}[b]{0.49\textwidth}
      \includegraphics[width=\textwidth]{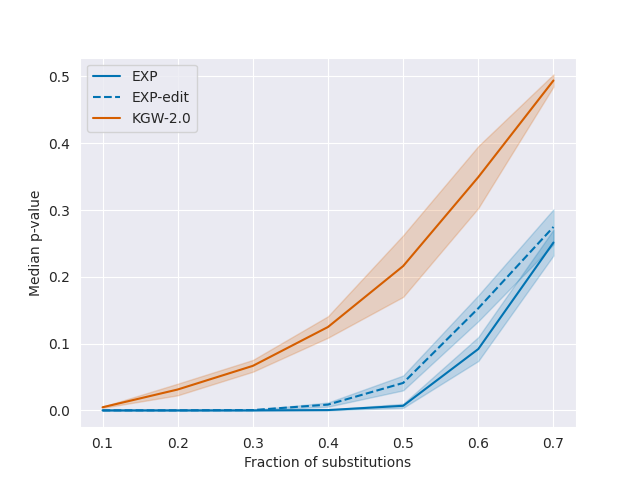}
      \label{fig:experiment-3-llama}
      \caption{LLaMA-7B}
    \end{subfigure}
  
    \caption{Median $p$-value of watermarked text relative to the fraction of substitution errors, 
      for OPT-1.3B and LLaMA-7B models with $m = 35$. Both versions of the EXP watermark significantly 
      outperform {\kgwtwo}, again despite {\kgwtwo} distorting the text distribution.
    }
    \label{fig:experiment-3}
\end{figure}

\begin{figure}
    \centering
    \begin{subfigure}[b]{0.49\textwidth}
      \includegraphics[width=\textwidth]{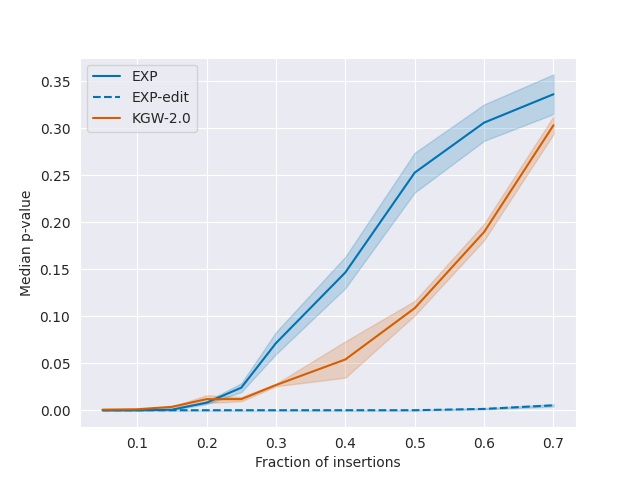}
      \label{fig:experiment-4-opt}
      \caption{OPT-1.3B}
    \end{subfigure}
    \hfill
    \begin{subfigure}[b]{0.49\textwidth}
      \includegraphics[width=\textwidth]{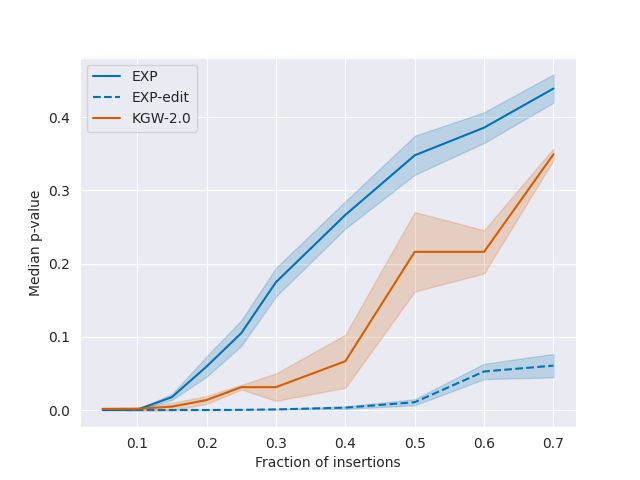}
      \label{fig:experiment-4-llama}
      \caption{LLaMA-7B}
    \end{subfigure}
  
    \caption{Median $p$-value of watermarked text relative to the fraction of insertion errors, 
    for OPT-1.3B and LLaMA-7B models with $m = 35$. 
    {\gumbedit} is by far the most robust.
    }
    \label{fig:experiment-4}
\end{figure}

\begin{figure}
    \centering
    \begin{subfigure}[b]{0.49\textwidth}
      \includegraphics[width=\textwidth]{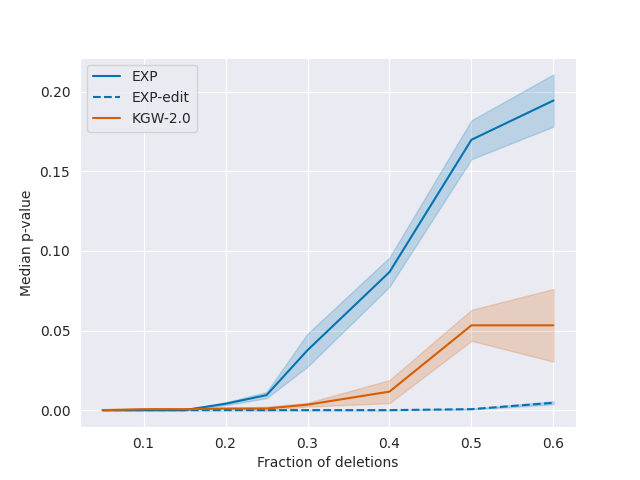}
      \label{fig:experiment-5-opt}
      \caption{OPT-1.3B}
    \end{subfigure}
    \hfill
    \begin{subfigure}[b]{0.49\textwidth}
      \includegraphics[width=\textwidth]{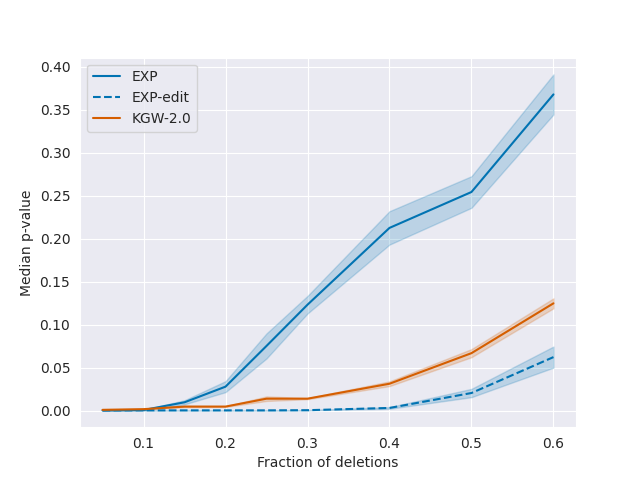}
      \label{fig:experiment-5-llama}
      \caption{LLaMA-7B}
    \end{subfigure}
  
    \caption{Median $p$-value of watermarked text relative to the fraction of deletion errors, 
    for OPT-1.3B and LLaMA-7B models with $m = 35$. {\gumbedit} is again the most robust.
    }
    \label{fig:experiment-5}
\end{figure}


Finally, in Figures~\ref{fig:experiment-6-french} and \ref{fig:experiment-6-russian} we implement a ``roundtrip translation'' attack, wherein we attempt to paraphrase watermarked texts of varying lengths
by translating the (English) texts into another language (i.e., French and Russian respectively) and back again using 
a machine translation model (details in Appendix~\ref{app:experiments-roundtrip}).
We include a representative example of the original and (re-)translated texts in Figure~\ref{fig:french-example}.
Using Russian is a noticeably more effective attack than French: none of the watermarks aside from {\gumbedit}
are able to reliably detect watermarked text with $p < 0.05$ irrespective of $m$.
In fact, for Russian the power of both {\gumb} and {\kgwtwo} stagnates (or even diminishes) with increasing text length,
perhaps due to the inherent ambiguity in translating longer texts.

In many cases, both using French and Russian, the roundtrip translation still preserves large chunks of
the original text, which suffices for watermark detection even using {\gumb}, which is substantially less robust 
to insertion and deletion errors than {\gumbedit}.
Aside from inspecting a few examples, we did not verify that the roundtrip translations preserve the basic semantics of the original text;
thus, it is possible our results provide an overly pessimistic view of the robustness of our watermarks to these attacks, 
since in practice users would presumably not publish such examples.
It is also possible that using different machine translation models---or more generally, different forms 
of automated paraphrasing---might be far more effective in evading watermark 
detection than those we employed.
We publish the full set of watermarked generations for each watermarking strategy, 
along with their (roundtrip) translations, as part of our code release.


\begin{figure}
    \centering
    \includegraphics[width=\textwidth]{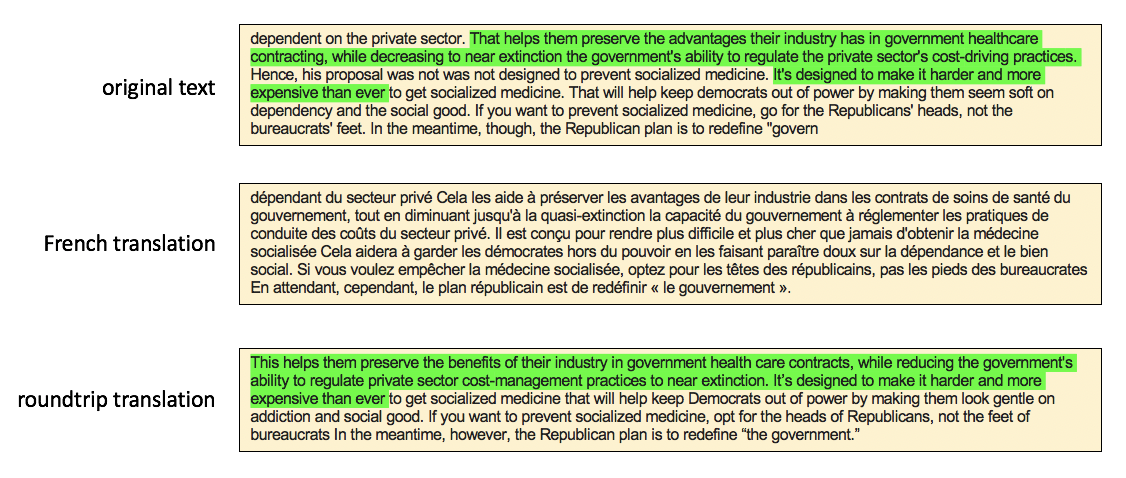}
    \caption{
      An illustrative example of a roundtrip translation attack via French. Given the first $50$ tokens of 
    the roundtrip translation (highlighted in green, in addition to the closest matching snippets to these tokens from the original text), 
    $\detect$ returns $\est{p} \leq 0.0002$.
    }
    \label{fig:french-example}
\end{figure}

\begin{figure}
    \centering
    \begin{subfigure}[b]{0.49\textwidth}
      \includegraphics[width=\textwidth]{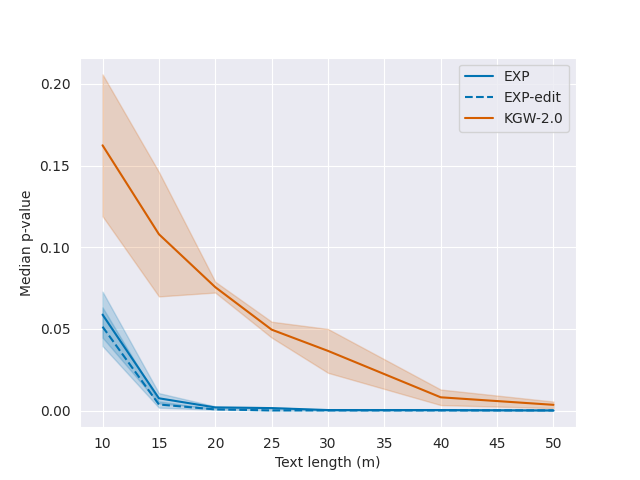}
      \label{fig:experiment-6-opt-french}
      \caption{OPT-1.3B}
    \end{subfigure}
    \hfill
    \begin{subfigure}[b]{0.49\textwidth}
      \includegraphics[width=\textwidth]{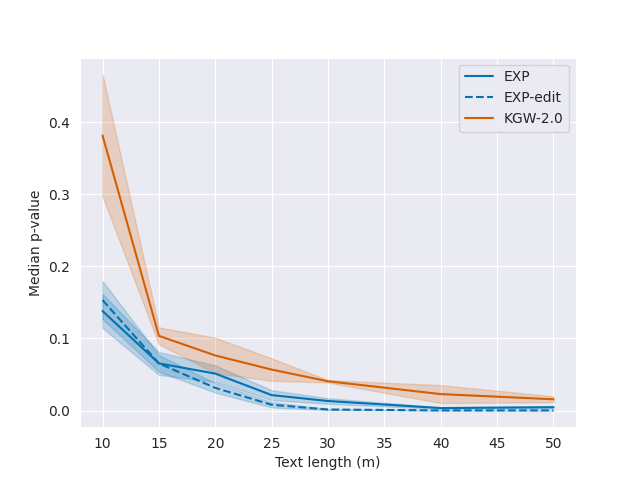}
      \label{fig:experiment-6-llama-french}
      \caption{LLaMA-7B}
    \end{subfigure}
  
    \caption{Median $p$-value of watermarked text relative to the text length, after roundtrip translation via French, 
    for OPT-1.3B and LLaMA-7B models with $m = 35$.
    {\gumb} performs comparably to {\gumbedit},
    indicating that the roundtrip translation attack tends to preserve at least some snippets of the original text.
    }
    \label{fig:experiment-6-french}
\end{figure}

\begin{figure}
    \centering
    \begin{subfigure}[b]{0.49\textwidth}
      \includegraphics[width=\textwidth]{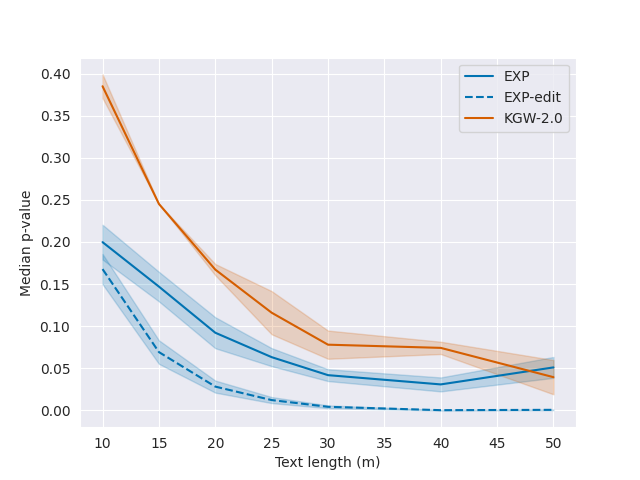}
      \label{fig:experiment-6-opt-russian}
      \caption{OPT-1.3B}
    \end{subfigure}
    \hfill
    \begin{subfigure}[b]{0.49\textwidth}
      \includegraphics[width=\textwidth]{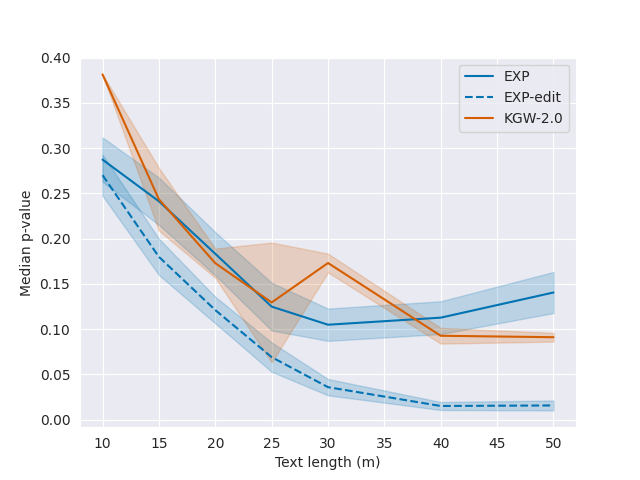}
      \label{fig:experiment-6-llama-russian}
      \caption{LLaMA-7B}
    \end{subfigure}
  
    \caption{Median $p$-value of watermarked text relative to the text length, after roundtrip translation via Russian, 
    for OPT-1.3B and LLaMA-7B models with $m = 35$.
    In contrast to French, {\gumbedit} noticeably outperforms {\gumb}.
    Overall, the attack is noticeably more effective than using French.
    }
    \label{fig:experiment-6-russian}
\end{figure}

\subsection{Case study: instruction following}\label{sec:exp-instruct}
In the wild, most users interact with language models by prompting the model with instructions 
(e.g., ``give me code for...''), and the most widely-used language models (e.g., ChatGPT) 
are specifically fine-tuned to follow such instructions.
Thus, using the instruction fine-tuned Alpaca-7B model,
we presently conduct a case study on the effectiveness of watermarking a performant instruction following model.
In particular, we sample $200$ instructions from the Alpaca-7B evaluation dataset and generate watermarked responses 
of at most $200$ tokens for each.
We then compute conditionally valid $p$-values for each response using the original version of $\detect$
with $T = 500$.
We also replicate the roundtrip translation attack from Experiment 6.
We publish the full set of watermarked generations for each method, 
along with their (roundtrip) translations, and the instruction prompts as part of our code release.

We plot the distribution of $p$-values for the {\gumbedit} and {\kgwtwo} watermarks in Figure~\ref{fig:experiment-7},
as well as the $p$-values versus the watermark potential of the watermarked text in Figure~\ref{fig:experiment-7-potential}.
In general, the Alpaca-7B responses have considerably lower per-token watermark potential than both the OPT-1.3B and LLaMA-7B
models, and thus the statistical power of our watermark is worse despite the responses typically being longer than in 
the previous experiments (i.e., Experiments 1 and 6).
In particular, based on the same random sample of $200$ prompts (from the Alpaca evaluation set in the case of Alpaca-7B, and from the 
news-like subset of the C4 dataset in the cases of LLaMA-7B and OPT-1.3B),
the average per-token watermark potential of text from Alpaca-7B is $0.28$, compared to $0.59$
for LLaMA-7B and $0.67$ for OPT-1.3B.
Unlike the previous experiments, {\kgwtwo} noticeably outperforms the {\gumbedit} watermark. 
Figure~\ref{fig:experiment-7-potential} indicates this difference in performance 
is largely due to the fact {\kgwtwo} distorts the distribution of the text and produces
responses with noticeably larger watermark potential than regular responses from the model.
For responses whose unnormalized watermark potential (i.e., watermark potential multiplied by the number of tokens in the response, to 
account for the varying lengths of the responses) exceeds roughly $60$, both watermarks tend to yield $p$-values close to zero.
Paraphrasing the responses via roundtrip translation attacks into both French and Russian degrades 
the statistical power of both watermarks, as we show in Figures~\ref{fig:experiment-7-french} 
and \ref{fig:experiment-7-russian}.

\begin{figure}
    \centering
    \begin{subfigure}[b]{0.49\textwidth}
      \includegraphics[width=\textwidth]{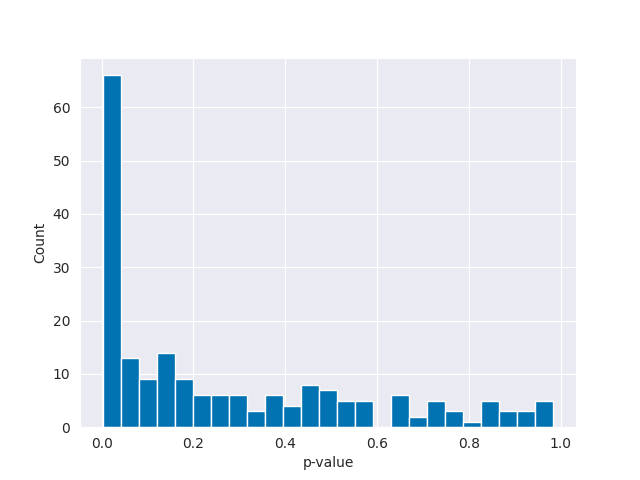}
      \label{fig:experiment-7-gu}
      \caption{{\gumbedit}}
    \end{subfigure}
    \hfill
    \begin{subfigure}[b]{0.49\textwidth}
      \includegraphics[width=\textwidth]{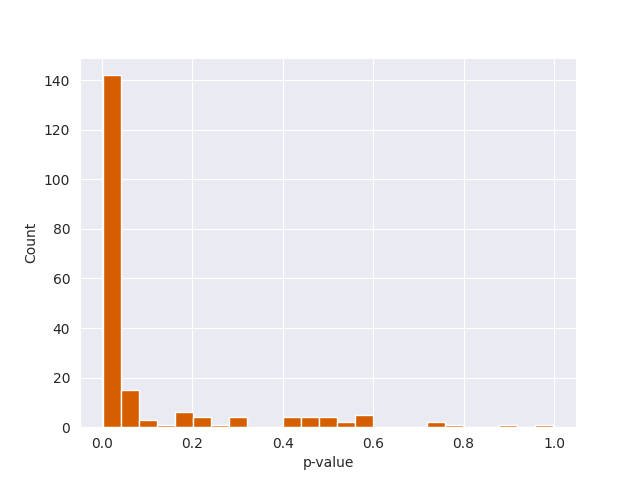}
      \label{fig:experiment-7-ki20}
      \caption{{\kgwtwo}}
    \end{subfigure}
  
    \caption{Histogram of $p$-values of watermarked text from Alpaca-7B. {\kgwtwo} is noticeably better 
    than {\gumbedit}, though again the results are not strictly comparable as {\kgwtwo} is not distortion-free.
    }
    \label{fig:experiment-7}
\end{figure}

\begin{figure}
    \centering
    \begin{subfigure}[b]{0.49\textwidth}
      \includegraphics[width=\textwidth]{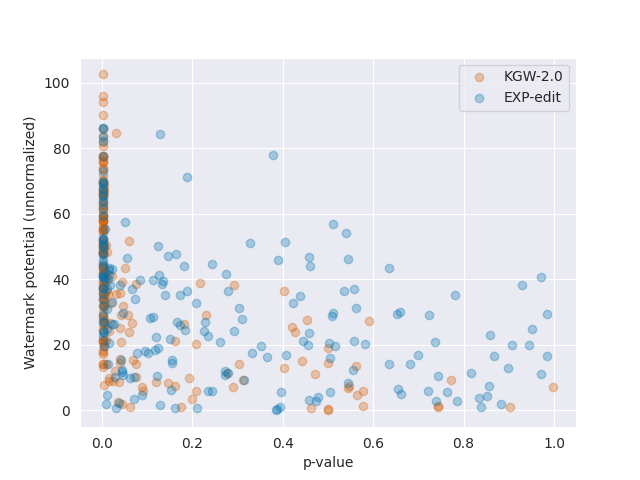}
      \label{fig:experiment-7-scatter}
      \caption{Scatterplot of $p$-values.}
    \end{subfigure}
    \hfill
    \begin{subfigure}[b]{0.49\textwidth}
      \includegraphics[width=\textwidth]{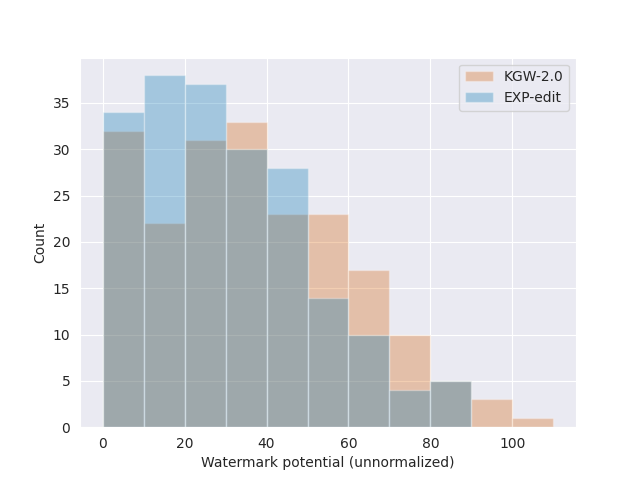}
      \label{fig:experiment-7-histogram}
      \caption{Histogram of watermark potential.}
    \end{subfigure}
  
    \caption{Watermark potential versus statistical power of {\gumbedit} versus {\kgwtwo}. {\kgwtwo} noticeably distorts the 
    text distribution, tending to produce higher watermark potential text overall than the original language model (and consequently, {\gumbedit}).
    }
    \label{fig:experiment-7-potential}
\end{figure}

\begin{figure}
    \centering
    \begin{subfigure}[b]{0.49\textwidth}
      \includegraphics[width=\textwidth]{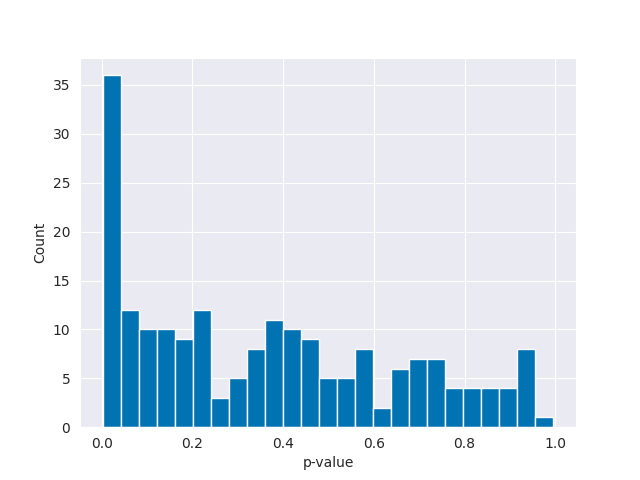}
      \label{fig:experiment-7-gu-french}
      \caption{{\gumbedit}}
    \end{subfigure}
    \hfill
    \begin{subfigure}[b]{0.49\textwidth}
      \includegraphics[width=\textwidth]{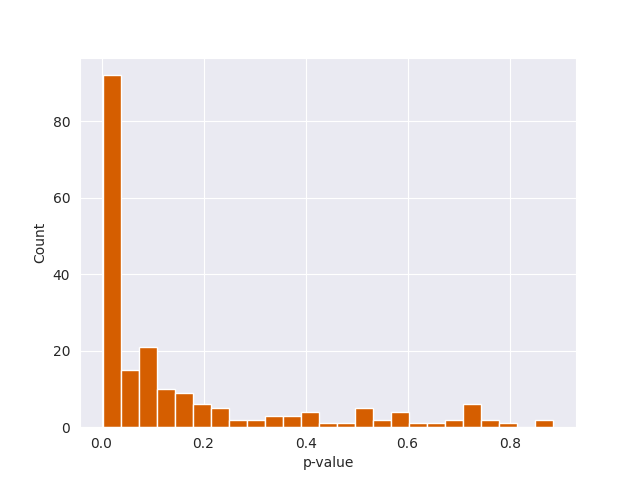}
      \label{fig:experiment-7-ki20-french}
      \caption{{\kgwtwo}}
    \end{subfigure}
  
    \caption{Histogram of $p$-values of watermarked text after roundtrip translation via French. {\kgwtwo} outperforms {\gumbedit}, albeit by noticeably distorting the text distribution.
    }
    \label{fig:experiment-7-french}
\end{figure}

\begin{figure}
    \centering
    \begin{subfigure}[b]{0.49\textwidth}
      \includegraphics[width=\textwidth]{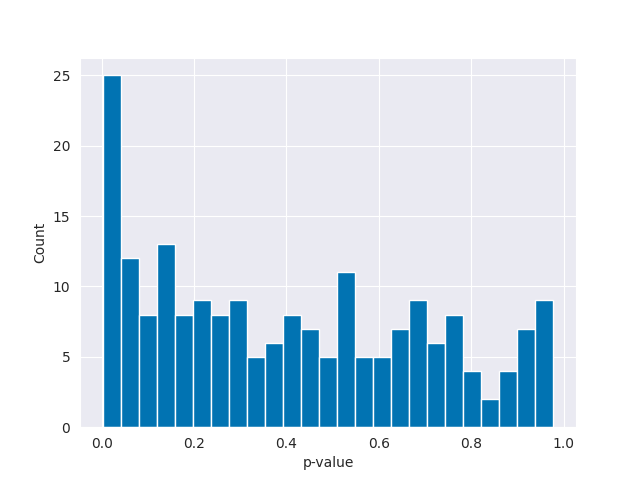}
      \label{fig:experiment-7-gu-russian}
      \caption{{\gumbedit}}
    \end{subfigure}
    \hfill
    \begin{subfigure}[b]{0.49\textwidth}
      \includegraphics[width=\textwidth]{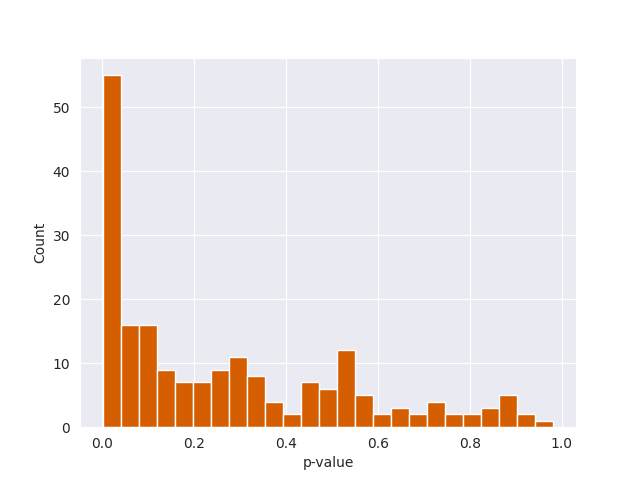}
      \label{fig:experiment-7-ki20-russian}
      \caption{{\kgwtwo}}
    \end{subfigure}
  
    \caption{Histogram of $p$-values of watermarked text after roundtrip translation via Russian. {\kgwtwo} again does significantly better 
    than {\gumbedit.}
    }
    \label{fig:experiment-7-russian}
\end{figure}

Finally, recall the main distinguishing feature of our watermark compared to \citet{kirchenbauer2023watermark} and \citet{aaronson2023openai} is that we do not hash previous tokens to determine 
the distribution of the next token.
To demonstrate the pitfalls of hashing, we implement a version of the watermark \citet{aaronson2023openai} proposes by modifying the $\generate$ 
method of {\gumb} to obtain the vector $\xi_i \in [0,1]^N$ from seeding a random number generator using the previous $k$ tokens
instead of using the watermark key;
we call this version {\gumbhash}.
We then prompt Alpaca-7B with requests for various kinds of lists.
Because Alpaca-7B tends to separate items in lists
by the same recurring token, e.g., a comma or a newline character, 
and because this recurring token determines the next token, for $k = 1$ the lists degenerate into repetition (Figure~\ref{fig:experiment-7-hashing}).\footnote{
    The authors would like to pat themselves on the back by drawing the reader's attention to the fact that the title of this paper is not among those suggested by Alpaca-7B.
}

From inspection, hashing with $k > 1$ substantially improves the quality of samples;
however, even using $k = 4$ can sometimes produce noticeably repetitive text.
We reiterate that while increasing $k$ may improve sample quality by making the distortions of watermarked text less noticeable,
doing so harms the robustness of the watermark (e.g., replacing just $20\%$ of the tokens would suffice to evade detection for $k = 4$).
Moreover, using a more robust hash function does not avoid this trade-off between robustness and distortion-freeness, 
as there is a direct trade-off between the likelihood of a hash collision and the robustness of the hash.
In addition to Figure~\ref{fig:experiment-7-hashing}, 
we include more examples (for both $k = 1$ and $k = 4$) and different prompts in Appendix~\ref{app:experiments-instruct} and our code release.

\begin{figure}
    \centering
    \begin{subfigure}[b]{0.49\textwidth}
      \includegraphics[width=\textwidth]{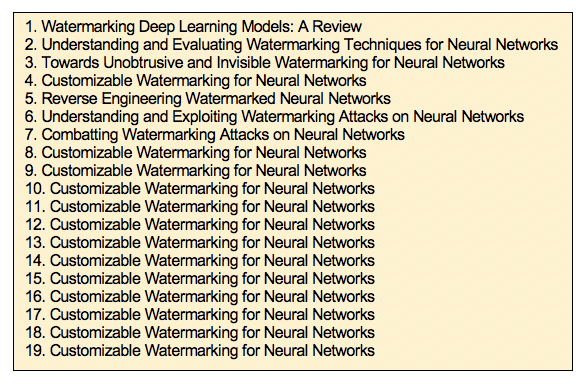}
      \label{fig:experiment-7-hashing-hash}
      \caption{{\gumbhash}}
    \end{subfigure}
    \hfill
    \begin{subfigure}[b]{0.49\textwidth}
      \includegraphics[width=\textwidth]{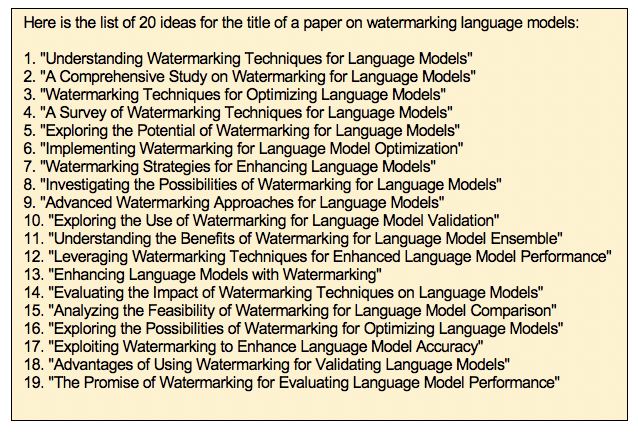}
      \label{fig:experiment-7-hashing-ctrl}
      \caption{{\gumb}}
    \end{subfigure}
  
    \caption{
        Example responses from Alpaca-7B to the prompt: ``Give me 20 ideas for the title of a paper on watermarking language models.''
        We generate (a) by hashing the previous token to determine the inputs to the {\gumb} decoder, while (b) is a regular sample 
        from our {\gumb} strategy.
        Hashing causes the model to degenerate into repetition.
    }
    \label{fig:experiment-7-hashing}
\end{figure}

\section{Discussion}\label{sec:discussion}
In this paper, we give the first distortion-free watermarking strategies for language models that 
are robust to editing and/or cropping.
The key idea underpinning our approach is to leverage methods for robust sequence alignment 
to align a putative watermarked text to a watermark key sequence which the LM
provider uses to generate watermarked text.
The statistical power of our watermarks improves exponentially with respect to the length of the 
text and diminishes only linearly with respect to the length of the watermark key sequence.

The core assumption underlying watermarking
is that the LM provider and the watermark detector coordinate 
by sharing information in advance, e.g., a watermark key.
Indeed, the main inherent limitation of watermarking is that the detector must trust the LM provider to 
faithfully apply the watermark when generating text.
A second limitation, 
which is not inherent to watermarking language models but does presently apply to all known watermarks,
is that the LM provider cannot release the model weights, since then users could simply 
query the model directly instead of through the LM provider.
Planting robust watermarks directly into the weights of a language model without 
degrading the quality of the model is an important direction for future work.

\subsection{Trade-offs among watermarks}
Hashing-based watermarks \citep{aaronson2023openai,kirchenbauer2023watermark,christ2023undetectable}
incur a direct trade-off between the degree of distortion versus robustness: larger hash windows reduce distortion but 
hurt robustness.
We avoid sacrificing distortion-freeness for robustness by choosing to formulate watermark detection as a sequence alignment problem; 
however, this design choice introduces a new trade-off: the computational complexity of our watermark detection algorithms grows linearly with
the length of the watermark key sequence.
In contrast, the complexities of the watermark detection algorithms of both \citet{christ2023undetectable}
and also \citet{aaronson2023openai} and \citet{kirchenbauer2023watermark} depend (in essence) only on 
the length of the input text.
Just as the watermark key length imposes a cap on both
the total number of distortion-free watermarked tokens the LM provider may generate for a single query as well as 
the expected total they can generate across multiple queries before reusing a part of the key sequence, the
the window size imposes a cap on the number of (distortion-free) tokens one can expect to generate using a hashing-based 
watermark without incurring a hash collision.
Whether this apparent tension between detection complexity, robustness and (approximate) distortion-freeness 
is due to fundamental trade-offs is an interesting open question.

To illustrate how these trade-offs manifest in practice,
suppose an LM provider responds with $m = 100$ tokens to 
a sequence of $T = 10$ user queries (which may be adaptively chosen).
Recall from Section~\ref{sec:multi-query} that
we must set the key length $n = \omega(mT^2)$ to achieve approximate distortion-freeness using our watermarks,
where here $mT^2 = 10000$;
for $n = 10000$, the runtimes\footnote{
    We report average runtimes and the associated standard deviations across $5$ calls on an Apple M2 Macbook Pro Laptop. We include benchmarking scripts
    with our code release.
} of our implementation of the test statistics for 
{\its}, {\itsedit}, {\gumb}, {\gumbedit} are $0.004 \pm 0.0002$, $0.60 \pm 0.01$, $2.20 \pm 0.01$ and $3.21 \pm 0.01$ seconds respectively.\footnote{
    In principle, running $\detect$ requires recomputing the test statistic for each resampled watermark key in order 
    to obtain an exact $p$-value. However, as we discuss in Section~\ref{app:experiments} and Appendix~\ref{app:experiments-pvals}, we can avoid this recomputation 
    and still obtain approximate $p$-values with respect to some reference distribution of unwatermarked text,
    in which case we need only compute the single test statistic once (using the original watermark key) during watermark detection (Algorithm~\ref{algorithm:modified-detect}).
}

In order to achieve approximate distortion-freeness using a hashing-based watermark in the same setting, the LM provider must set the window size $k$ to be sufficiently large
as a function of $m$ and $T$.
For example, \citet{christ2023undetectable} argue that the probability of incurring the same sequence of tokens twice decays exponentially with the observed 
entropy of the token sequence (which will depend on the previous tokens); thus, they dynamically adjust the window size $k$ during generation to ensure the observed entropy of the constituent 
tokens in the window is sufficiently large such that the probability of a hash collision with any of the other windows is negligible across all $mT$ tokens.
Specifically, letting $h$ denote the expected observed per-token entropy (i.e., log-probability),
the typical window size for a hashing-based watermark would need to be at least $k \approx \frac{\log m T}{h}$ in order to ensure approximate distortion-freeness in our setting.
Concretely, in the setting of Section~\ref{sec:exp-instruct}, we have $h \approx 0.70$ for Alpaca-7B,
in which case such a hashing-based watermark would not be robust to replacing  
more than roughly $10\%$ of watermarked tokens.

\subsection{Recommendations in practice and combining watermarks}
We conclude with some salient recommendations for practitioners aiming to watermark their deployed language models.
First, though in principle the length of the watermark key sequence $n$---which recall imposes a cap on the total 
number of distortion-free watermarked tokens the LM provider can generate---can grow (nearly) exponentially 
in the block size $k$ of the test statistic while still enabling watermark detection from as few as $k$ tokens,
in practice we find that using a fairly small watermark key sequence (e.g., $n = 256$) 
does not noticeably affect the quality of watermarked text (i.e., even when generating more than $n$ tokens total)
while allowing for fast detection and improved robustness.
In settings where robustness is important (e.g., discouraging students from using a language model for homework assistance) 
we recommend practitioners use our {\gumbedit} watermark, as it is by far the most robust watermark 
of those we tested.
Meanwhile, in settings where throughput of detection is important (e.g., scrubbing synthetic text from a large training corpus),
we recommend practitioners use our {\its} watermark: its detection is essentially an instance of maximum inner-product search, a problem 
for which there exist various fast implementations and indexing structures (e.g., via vector databases).

Finally, we remark that for certain hashing-based watermarks we can combine our watermark with the hashing-based watermark 
to generate watermarked text that is detectable
using either of the two corresponding watermark detection algorithms.
For example, we can use the hashing-based watermark of Kirchenbauer et al. (2023), 
which biases the distribution of the next token by upweighting certain logits over others,
to determine the distribution of the next token and then use our watermarks to sample from this next token distribution.
One can then later determine which detection procedure to run for a collection of putative watermarked text depending on whether throughput or robustness is a higher priority.
Similarly, we can combine our watermarks with the watermark of \citet{christ2023undetectable} by alternating tokens between 
the watermarks (and only hashing the alternate tokens).
The resulting watermarked text will be approximately distortion-free since both watermarks
are approximately distortion-free.
One can choose between the two detection procedures to optimize precision versus recall in adversarial settings; 
in particular, the watermark of \citet{christ2023undetectable} is hard to spoof (due to its multi-query undetectability guarantee) while 
our watermarks are harder to remove (due to our robustness guarantees).
Exploring such combinations of watermarks with complementary strengths is an exciting direction for future work.



\section*{Acknowledgement}
We thank Saminul Haque, Gary Cheng and Padma Kuditipudi for pointing out errors in preliminary drafts of this work
and for their helpful feedback in general.
This work is supported by an Open Philanthropy Project Award (OpenPhil) and an NSF Frontier Award (NSF Grant no. 1805310).

\bibliography{references}

\begin{thebibliography}{30}
\providecommand{\natexlab}[1]{#1}
\providecommand{\url}[1]{\texttt{#1}}
\expandafter\ifx\csname urlstyle\endcsname\relax
  \providecommand{\doi}[1]{doi: #1}\else
  \providecommand{\doi}{doi: \begingroup \urlstyle{rm}\Url}\fi

\bibitem[Aaronson(2023)]{aaronson2023openai}
S.~Aaronson.
\newblock `{R}eform' {AI} {A}lignment with {S}cott {A}aronson.
\newblock AXRP - the AI X-risk Research Podcast, 2023.
\newblock URL \url{https://axrp.net/episode/2023/04/11/episode-20-reform-ai-alignment-scott-aaronson.html}.

\bibitem[Abdelnabi and Fritz(2021)]{abdelnabi2021adversarial}
S.~Abdelnabi and M.~Fritz.
\newblock Adversarial watermarking transformer: Towards tracing text provenance with data hiding.
\newblock In \emph{IEEE Symposium on Security and Privacy}, 2021.

\bibitem[Atallah et~al.(2001)Atallah, Raskin, Crogan, Hempelmann, Kerschbaum, Mohamed, and Naik]{atallah2001natural}
M.~J. Atallah, V.~Raskin, M.~Crogan, C.~Hempelmann, F.~Kerschbaum, D.~Mohamed, and S.~Naik.
\newblock Natural language watermarking: Design, analysis, and a proof-of-concept implementation.
\newblock In \emph{Information Hiding: 4th International Workshop, IH 2001 Pittsburgh, PA, USA, April 25--27, 2001 Proceedings 4}, pages 185--200. Springer, 2001.

\bibitem[Brown et~al.(2020)Brown, Mann, Ryder, Subbiah, Kaplan, Dhariwal, Neelakantan, Shyam, Sastry, Askell, et~al.]{brown2020language}
T.~Brown, B.~Mann, N.~Ryder, M.~Subbiah, J.~D. Kaplan, P.~Dhariwal, A.~Neelakantan, P.~Shyam, G.~Sastry, A.~Askell, et~al.
\newblock Language models are few-shot learners.
\newblock \emph{Advances in Neural Information Processing Systems 33}, 33:\penalty0 1877--1901, 2020.

\bibitem[Christ et~al.(2023)Christ, Gunn, and Zamir]{christ2023undetectable}
M.~Christ, S.~Gunn, and O.~Zamir.
\newblock Undetectable watermarks for language models.
\newblock \emph{arXiv preprint arXiv:2306.09194}, 2023.

\bibitem[Dai and Cai(2019)]{dai2019towards}
F.~Dai and Z.~Cai.
\newblock Towards near-imperceptible steganographic text.
\newblock In \emph{Proceedings of the 57th Annual Meeting of the Association for Computational Linguistics}, pages 4303--4308, 2019.

\bibitem[Elfving et~al.(1966)Elfving, Birkhoff, and von Mises]{elfving1966selected}
G.~Elfving, G.~Birkhoff, and R.~von Mises.
\newblock Vol. 2. probability and statistics, {G}eneral.
\newblock In \emph{Selected Papers of {R}ichard von {M}ises}. American Mathematical Society, 1966.

\bibitem[He et~al.(2022{\natexlab{a}})He, Xu, Lyu, Wu, and Wang]{he2022protecting}
X.~He, Q.~Xu, L.~Lyu, F.~Wu, and C.~Wang.
\newblock Protecting intellectual property of language generation {API}s with lexical watermark.
\newblock In \emph{Proceedings of the Thirty-Sixth AAAI Conference on Artificial Intelligence}, 2022{\natexlab{a}}.

\bibitem[He et~al.(2022{\natexlab{b}})He, Xu, Zeng, Lyu, Wu, Li, and Jia]{he2022cater}
X.~He, Q.~Xu, Y.~Zeng, L.~Lyu, F.~Wu, J.~Li, and R.~Jia.
\newblock Cater: Intellectual property protection on text generation apis via conditional watermarks.
\newblock In \emph{Advances in Neural Information Processing Systems 35}, 2022{\natexlab{b}}.

\bibitem[Jawahar et~al.(2020)Jawahar, Abdul-Mageed, and Laks~Lakshmanan]{jawahar2020automatic}
G.~Jawahar, M.~Abdul-Mageed, and V.~Laks~Lakshmanan.
\newblock Automatic detection of machine generated text: A critical survey.
\newblock In \emph{International Conference on Computational Linguistics}, 2020.

\bibitem[Kamaruddin et~al.(2018)Kamaruddin, Kamsin, Por, and Rahman]{kamaruddin2018review}
N.~S. Kamaruddin, A.~Kamsin, L.~Y. Por, and H.~Rahman.
\newblock A review of text watermarking: theory, methods, and applications.
\newblock \emph{IEEE Access}, 2018.

\bibitem[Katzenbeisser and Petitcolas(2000)]{katzenbeisser2000digital}
S.~Katzenbeisser and F.~Petitcolas.
\newblock Digital watermarking.
\newblock \emph{Artech House, London}, 2:\penalty0 2, 2000.

\bibitem[Kirchenbauer et~al.(2023)Kirchenbauer, Geiping, Wen, Katz, Miers, and Goldstein]{kirchenbauer2023watermark}
J.~Kirchenbauer, J.~Geiping, Y.~Wen, J.~Katz, I.~Miers, and T.~Goldstein.
\newblock A watermark for large language models.
\newblock \emph{arXiv preprint arXiv:2301.10226}, 2023.

\bibitem[Mitchell et~al.(2023)Mitchell, Lee, Khazatsky, Manning, and Finn]{mitchell2023detectgpt}
E.~Mitchell, Y.~Lee, A.~Khazatsky, C.~D. Manning, and C.~Finn.
\newblock Detectgpt: Zero-shot machine-generated text detection using probability curvature.
\newblock \emph{arXiv preprint arXiv:2301.11305}, 2023.

\bibitem[Papandreou and Yuille(2011)]{papandreou2011perturbandmaprf}
G.~Papandreou and A.~L. Yuille.
\newblock Perturb-and-map random fields: Using discrete optimization to learn and sample from energy models.
\newblock \emph{2011 International Conference on Computer Vision}, pages 193--200, 2011.

\bibitem[Raffel et~al.(2020)Raffel, Shazeer, Roberts, Lee, Narang, Matena, Zhou, Li, and Liu]{raffel2020exploring}
C.~Raffel, N.~Shazeer, A.~Roberts, K.~Lee, S.~Narang, M.~Matena, Y.~Zhou, W.~Li, and P.~J. Liu.
\newblock Exploring the limits of transfer learning with a unified text-to-text transformer.
\newblock \emph{The Journal of Machine Learning Research}, 21\penalty0 (1):\penalty0 5485--5551, 2020.

\bibitem[Rizzo et~al.(2019)Rizzo, Bertini, and Montesi]{rizzo2019fine}
S.~G. Rizzo, F.~Bertini, and D.~Montesi.
\newblock Fine-grain watermarking for intellectual property protection.
\newblock \emph{EURASIP Journal on Information Security}, 2019.

\bibitem[Shen et~al.(2020)Shen, Ji, and Han]{shen2020near}
J.~Shen, H.~Ji, and J.~Han.
\newblock Near-imperceptible neural linguistic steganography via self-adjusting arithmetic coding.
\newblock In \emph{Proceedings of Empirical Methods for Natural Language Processing}, pages 303--313, 2020.

\bibitem[Taori et~al.(2023)Taori, Gulrajani, Zhang, Dubois, Li, Guestrin, Liang, and Hashimoto]{taori23alpaca}
R.~Taori, I.~Gulrajani, T.~Zhang, Y.~Dubois, X.~Li, C.~Guestrin, P.~Liang, and T.~B. Hashimoto.
\newblock Stanford alpaca: An instruction-following llama model.
\newblock \url{https://github.com/tatsu-lab/stanford_alpaca}, 2023.

\bibitem[Tiedemann and Thottingal(2020)]{tiedemann2020opus}
J.~Tiedemann and S.~Thottingal.
\newblock {OPUS-MT} — {B}uilding open translation services for the {W}orld.
\newblock In \emph{Proceedings of the 22nd Annual Conference of the European Association for Machine Translation}, 2020.

\bibitem[Tiedemann et~al.(2022)Tiedemann, Aulamo, Bakshandaeva, Boggia, Gr{\"o}nroos, Nieminen, Raganato, Scherrer, Vazquez, and Virpioja]{tiedemann2022democratizing}
J.~Tiedemann, M.~Aulamo, D.~Bakshandaeva, M.~Boggia, S.-A. Gr{\"o}nroos, T.~Nieminen, A.~Raganato, Y.~Scherrer, R.~Vazquez, and S.~Virpioja.
\newblock Democratizing machine translation with opus-mt.
\newblock \emph{arXiv preprint arXiv:2212.01936}, 2022.

\bibitem[Touvron et~al.(2023)Touvron, Lavril, Izacard, Martinet, Lachaux, Lacroix, Rozi{\`e}re, Goyal, Hambro, Azhar, et~al.]{touvron2023llama}
H.~Touvron, T.~Lavril, G.~Izacard, X.~Martinet, M.-A. Lachaux, T.~Lacroix, B.~Rozi{\`e}re, N.~Goyal, E.~Hambro, F.~Azhar, et~al.
\newblock Llama: Open and efficient foundation language models.
\newblock \emph{arXiv preprint arXiv:2302.13971}, 2023.

\bibitem[Ueoka et~al.(2021)Ueoka, Murawaki, and Kurohashi]{ueoka2021frustratingly}
H.~Ueoka, Y.~Murawaki, and S.~Kurohashi.
\newblock Frustratingly easy edit-based linguistic steganography with a masked language model.
\newblock In \emph{Proceedings of the 2021 Conference of the North American Chapter of the Association for Computational Linguistics: Human Language Technologies}, 2021.

\bibitem[Venugopal et~al.(2011)Venugopal, Uszkoreit, Talbot, Och, and Ganitkevitch]{venugopal2011watermarking}
A.~Venugopal, J.~Uszkoreit, D.~Talbot, F.~J. Och, and J.~Ganitkevitch.
\newblock Watermarking the outputs of structured prediction with an application in statistical machine translation.
\newblock In \emph{Proceedings of Empirical Methods for Natural Language Processing}, 2011.

\bibitem[Vincent(2022)]{vincent23stackoverflow}
J.~Vincent.
\newblock {AI}-generated answers temporarily banned on coding {Q}\&{A} site {S}tack {O}verflow.
\newblock \emph{The Verge}, 2022.
\newblock URL \url{https://www.theverge.com/2022/12/5/23493932/chatgpt-ai-generated-answers-temporarily-banned-stack-overflow-llms-dangers}.

\bibitem[Wainwright(2019)]{wainwright19}
M.~J. Wainwright.
\newblock \emph{High-Dimensional Statistics: A Non-Asymptotic Viewpoint}.
\newblock Cambridge University Press, 2019.

\bibitem[Yang et~al.(2022)Yang, Zhang, Chen, Zhang, Ma, Wang, and Yu]{yang2022tracing}
X.~Yang, J.~Zhang, K.~Chen, W.~Zhang, Z.~Ma, F.~Wang, and N.~Yu.
\newblock Tracing text provenance via context-aware lexical substitution.
\newblock In \emph{Proceedings of the Thirty-Sixth AAAI Conference on Artificial Intelligence}, 2022.

\bibitem[Zhang et~al.(2022)Zhang, Roller, Goyal, Artetxe, Chen, Chen, Dewan, Diab, Li, Lin, et~al.]{zhang2022opt}
S.~Zhang, S.~Roller, N.~Goyal, M.~Artetxe, M.~Chen, S.~Chen, C.~Dewan, M.~Diab, X.~Li, X.~V. Lin, et~al.
\newblock Opt: Open pre-trained transformer language models.
\newblock \emph{arXiv preprint arXiv:2205.01068}, 2022.

\bibitem[Zhao et~al.(2023)Zhao, Wang, and Li]{zhao2023protecting}
X.~Zhao, Y.-X. Wang, and L.~Li.
\newblock Protecting language generation models via invisible watermarking.
\newblock \emph{arXiv preprint arXiv:2302.03162}, 2023.

\bibitem[Ziegler et~al.(2019)Ziegler, Deng, and Rush]{ziegler2019neural}
Z.~Ziegler, Y.~Deng, and A.~M. Rush.
\newblock Neural linguistic steganography.
\newblock In \emph{Proceedings of Empirical Methods for Natural Language Processing}, pages 1210--1215, 2019.

\end{thebibliography}
\bibliographystyle{abbrvnat}

\appendix

\section{Proof of Lemma~\ref{lemma:lower-bound-all}}\label{app:lb}
\begin{proof}
    To show the claim, we first lower bound the probability that $Y = Y'$.
    In particular,
    \begin{align*}
        \P(Y = Y') & = \sum_y \P(Y = y) \P(Y' = y) \\
        & = \sum_y \P(Y = y) \prod_{i \in [m]} p(y_i \mid y_{:i-1}) \\
        & = \sum_y \P(Y = y) \prod_{i \in [m]} (1 - (1 - p(y_i \mid y_{:i-1}))) \\
        & \stackrel{(\star)}{\geq} \sum_{y\in \mc{V}_c} \P(Y = y) \exp\left(-c\sum_{i \in [m]} 1 - p(y_i \mid y_{:i-1})\right) \\
        & \geq \Ep \left[\exp\left(-c m\alpha(Y)\right) \indic{Y \in \mc{V}^c}\right],
    \end{align*}
    where ($\star$) follows from
    $\exp(-cx) \leq 1 - x$ for $0 \leq x \leq 1 - \exp(-c/2)$.
    It then follows immediately that we can bound the total variation distance 
    between the joint distributions of $(Y,\xi)$ and $(Y',\xi)$ by
    \begin{align*}
        D_{TV}((Y,\xi)||(Y',\xi)) & \leq \P((Y,\xi) \neq (Y',\xi)) \\
        & \leq 1 - \Ep \left[\exp\left(-c m\alpha(Y)\right)\indic{Y \in \mc{V}^c}\right].
    \end{align*}

    Observe for any event $A$ that
    \begin{align*}
        D_{TV}((Y,\xi)||(Y',\xi)) \geq \P((Y,\xi) \in A) - \P((Y',\xi) \in A),
    \end{align*}
    and thus, combining the previous two displays, we have
    \begin{align*}
        \P((Y,\xi) \in A) + \P((Y',\xi) \notin A) & \geq \P((Y,\xi) \in A) + \P((Y,\xi) \notin A) - D_{TV}((Y,\xi)||(Y',\xi)) \\
        & \geq \Ep \left[\exp\left(-c m\alpha(Y)\right)\indic{Y \in \mc{V}^c}\right].
    \end{align*}
    The desired result thus follows from letting $A$ be the event that $h$ predicts $-1$.
\end{proof}
\section{Analysis of inverse transform sampling}\label{app:transform}
We first introduce the following supporting lemma.
Recall $C_0 = \txt{Var}(\eta(\txt{Unif}([N])))$ for $\eta(i) = (i-1)/(N-1)$.

\begin{lemma}\label{lemma:core-result-cov}
    Let $\mu \in \Delta([N])$. Let $(U,\pi) \sim \txt{Unif}([0,1]) \times \txt{Unif}(\Pi)$ and
    $Y = \decode((U,\pi),\mu)$. 
    Then $\frac{1}{C_0}\txt{Cov}(U,\eta(\pi(Y)) \mid Y) =  1-\mu(Y)$ almost surely.
\end{lemma}
\begin{proof}
    We first characterize the conditional distribution of $\pi$ given $Y$
    and the conditional distribution of $U$ given both $\pi$ and $Y$, 
    where recall $\pi$ and $Y$ are discrete.
    Applying Bayes' formula and Theorem~\ref{thm:stealth}, we have
    \begin{align}
        \P(\pi \mid Y) = \frac{\P(Y \mid \pi) \P(\pi)}{\P(Y)}
        \stackrel{(\star)}{=} \frac{\mu(Y) \P(\pi)}{\P(Y)}
        = \P(\pi). \label{eqn:helper-display-1}
    \end{align}
    Also, defining the interval
    \begin{align*}
        I(Y,\pi) \defeq \left[\mu(\{y : \pi(y) < \pi(Y)\}),\mu(\{y : \pi(y) \leq \pi(Y)\})\right],
    \end{align*}
    for any interval $I \subset [0,1]$ we have
    \begin{align}
        \P(U \in I \mid Y,\pi)
        \stackrel{(a)}{=} \frac{\P(Y \mid U \in I,\pi)\P(U \in I)\P(\pi)}{\mu(Y)\P(\pi)}
        \stackrel{(b)}{=} \frac{\abs{I \cap I(Y,\pi)}}{\mu(Y)}
        \stackrel{(c)}{=} \frac{\abs{I \cap I(Y,\pi)}}{\abs{I(Y,\pi)}}, \label{eqn:helper-display-2}
    \end{align}
    where ($a$) follows from Bayes' formula and the independence of $U$ and $\pi$;
    ($b$) follows from the definition~\eqref{eqn:rho-defn} of the decoder $\decode$; 
    and ($c$) follows from $I(Y,\pi) \subset [0,1]$ having width equal to
    $\mu(Y)$.
    The displays~\eqref{eqn:helper-display-1} and~\eqref{eqn:helper-display-2} respectively imply $\pi \mid Y \sim \txt{Unif}(\Pi)$
    and $U \mid \pi,Y \sim \txt{Unif}(I(Y,\pi))$, from which it follows that
    \begin{align*}
        \Ep\left[U \mid Y,\pi(Y)\right] & = \Ep\left[\mu(\{y:\pi(y)<\pi(Y)\})  + \frac{\abs{I(Y,\pi)}}{2} \;\middle|\; Y,\pi(Y)\right] \\
        & = \frac{\left(\pi(Y)-1\right)\left(1-\mu(Y)\right)}{n-1} + \frac{\mu(Y)}{2} \\
        & = 1/2 + \left(\eta(\pi(Y))-1/2\right)\left(1-\mu(Y)\right).
    \end{align*}
    By symmetry, we have $\Ep[U] = \Ep[\eta(\pi(Y))] = 1/2$, the former 
    because $\P(Y \mid U) = \P(Y \mid 1-U)$ for any $U$ and 
    the latter because recall $\pi \mid Y$ is uniform over $\Pi$.
    Thus, marginalizing the preceding display over $\pi(Y)$ gives
    \begin{align*}
        \txt{Cov}(U,\eta(\pi(Y)) \mid Y) & = \Ep \left[\left(U-1/2\right)\left(\eta(\pi(Y)) - 1/2\right) \mid Y\right] \\
        & = (1 - \mu(Y)) \txt{Var}(\eta(\pi(Y)) \mid Y),
    \end{align*}
    from which the desired result follows immediately from recalling
    $\pi(Y) \mid Y \sim \txt{Unif}([N])$ and the definition of the constant $C_0$.
\end{proof}
\subsection{Proof of Lemma~\ref{lemma:sound-transform}}
\begin{proof}
    Recall by definition
    \begin{align*}
        d(Y_i,\xi_i) = -(U_i - 1/2) \cdot (\eta(\pi_i(Y_i)) - 1/2),
    \end{align*}
    where (as in the proof of Lemma~\ref{lemma:core-result-cov}) we have $\Ep[U_i \mid Y] = \Ep[\eta(\pi_i(Y_i)) \mid Y] = 1/2$.
    Lemma~\ref{lemma:core-result-cov} thus implies $\Ep[d(Y_i,\xi_i) \mid Y] = -C_0 \cdot (1 - p(Y_i \mid Y_{:i-1}))$,
    while trivially $\Ep[d(Y_i,\xi_j') \mid Y] = 0$ as $Y$ and $\xi'$ are independent.
    The result follows immediately.
\end{proof}
\subsection{Proof of Lemma~\ref{lemma:perm-test}}
We prove the following more general result, from which Lemma~\ref{lemma:perm-test} follows as a corollary.
\begin{lemma}
    Let $m,n \in \N$ with $n \geq m$, where $m$ is the generation length and $n$ is the watermark key length. 
    Define the decoder $\decode$ by equation~\eqref{eqn:rho-defn}, alignment score $d$ by equation~\eqref{eqn:d-defn},
    and $\phi$ by Algorithm~\ref{algorithm:test} with block size $k \leq m$.
    Let $\xi,\xi' \iid \txt{Unif}(\Xi^n)$ with $Y = \generate(\xi;n,p,\decode)$.
    Let $\wt{Y}$ be a substring of $Y$ of length at least $k$ that is conditionally independent of 
    $\xi$ and $\xi'$ given $Y$, i.e., $\wt{Y} = Y_{\tau+1:\tau+\ell}$ for $\ell \geq k$.
    Then for $\est{\alpha} \defeq 1 - \frac{1}{k}\sum_{i=\tau+1}^{\tau+k} p(Y_i \mid Y_{:i-1})$, almost surely
    \begin{align*}
        \P(\phi(\wt{Y},\xi') \leq \phi(\wt{Y},\xi) \mid \wt{Y},Y) \leq 2n\exp\left(-k C_0^2\est{\alpha}^2/2\right).
    \end{align*}
\end{lemma}
\begin{proof}
    Recall by definition
    \begin{align}
        d(y,(u,\pi)) = -\sum_{i=1}^{\len{y}} (u_i - 1/2) \cdot (\eta(\pi_i(y_i)) - 1/2), \label{eqn:d-defn-review}
    \end{align}
    Lemma~\ref{lemma:sound-transform} and the conditional independence of $\tau$ and $\xi$ given $Y$ 
    imply for any $j \in [n]$ that 
    \begin{align*}
        \Ep[d(\wt{Y}_{1:k},\xi_{(j+1:j+k)\% n}') \mid Y,\wt{Y}] - \Ep[d(\wt{Y}_{1:k},\xi_{\tau+1:\tau+k})\mid Y,\wt{Y}] = k C_0 \est{\alpha}.
    \end{align*}
    
    Each summand in equation~\eqref{eqn:d-defn-review} lies between $-1/4$ and $1/4$,
    and also $(U_i,\pi_i)$ is conditionally independent of $U_{-i}$ and $\pi_{-i}$ given $Y$.
    Thus, Hoeffding's inequality \cite[Proposition 2.5]{wainwright19} implies for $j \in [n]$ that 
    \begin{align*}
        & \P\left(d(\wt{Y},\xi_{(j+1:j+k)\% n}') \leq d(\wt{Y},\xi_{\tau+1:\tau+k}) \mid Y,\wt{Y}\right) \\
        & \leq \P\left(d(\wt{Y},\xi_{1:m}) - \Ep[d(\wt{Y},\xi_{1:m})] \geq kC_0\est{\alpha}/2 \mid Y,\wt{Y}\right) \\
        &  \qquad + \P\left(\Ep[d(\wt{Y},\xi_{j+1:j+m}')] - d(\wt{Y},\xi_{j+1:j+m}') \geq k C_0 \est{\alpha}/2\mid Y,\wt{Y}\right) \\
        & \leq 2\exp\left(-m C_0^2\est{\alpha}^2/2\right).
    \end{align*}
    Recalling the definition of the test statistic $\phi$ via Algorithm~\ref{algorithm:test},
    the main claim then follows from taking a union bound over all $j \in [n]$.
\end{proof}

\subsection{Proof of Lemma~\ref{lemma:perm-test-subs}}
\begin{proof}
    We begin with the following observation for a single token.
    \begin{observation}\label{obs:robust-perm-test}
        Let $P \in \Delta([N])$. Let $(U,\pi) \sim \txt{Unif}([0,1]) \times \txt{Unif}(\Pi)$ and
        $Y = \decode((U,\pi),P)$. Let $\wt{Y} \in [N]$ be conditionally independent of $(U,\pi)$ given $Y$.
        If $\wt{Y} \neq Y$, then almost surely
        \begin{align*}
            \txt{Cov}(U,\eta(\pi(\wt{Y})) \mid Y,\wt{Y}) = -\frac{1}{N-1}\txt{Cov}(U,\eta(\pi(Y)) \mid Y,\wt{Y}).
        \end{align*}
    \end{observation}
    \begin{proof}[Proof of Observation~\ref{obs:robust-perm-test}]
        Observe the conditional distribution of $\pi(\wt{Y})$ given $Y$ is uniform over 
        $[N] \setminus \{\pi(Y)\}$.
        Let $X$ be a random variable that is equal to $\eta(\pi(Y))$
        with probability $1/N$ and otherwise equal to $\eta(\pi(\wt{Y}))$.
        Observe $X$ is independent of $Y$ and thus also $U$ by assumption---in particular, $(N-1)X + 1 \mid Y \sim \txt{Unif}([N])$
        irrespective of the value of $Y$.
        The claim thus follows from rearranging terms in the equality
        \begin{align*}
            0 = \txt{Cov}(U,X \mid Y,\wt{Y}) = \frac{1}{N} \txt{Cov}(U,\eta(\pi(Y)) \mid Y,\wt{Y}) + \frac{N-1}{N}\txt{Cov}(U,\eta(\pi(\wt{Y})) \mid Y,\wt{Y}).
        \end{align*}
    \end{proof}

    Lemma~\ref{lemma:sound-transform} and Observation~\ref{obs:robust-perm-test} 
    together imply for any $j \in [n]$ that 
    \begin{align*}
        \Ep[d(\wt{Y},\xi_{j+1:j+m}') \mid \wt{Y},Y] - \Ep[d(\wt{Y},\xi_{1:m}) \mid \wt{Y},Y] = m C_0\wt{\alpha}(Y,\wt{Y}),
    \end{align*}
    i.e., by adding the two results together using Observation~\ref{obs:robust-perm-test} to account for 
    the influence of each substituted token on the expectation.
    Using the same concentration argument as in the proof of Theorem~\ref{lemma:perm-test},
    we then have 
    \begin{align*}
        & \P\left(d(\wt{Y},\xi_{j+1:j+m}') \leq d(\wt{Y},\xi_{1:m})\mid \wt{Y},Y\right) \\
        & \leq \P\left(d(\wt{Y},\xi_{1:m}) - \Ep[d(\wt{Y},\xi_{1:m})] \geq m\wt{\alpha}(Y,\wt{Y})/2 \mid \wt{Y},Y\right) \\
        &  \qquad + \P\left(\Ep[d(\wt{Y},\xi_{j+1:j+m}')] - d(\wt{Y},\xi_{j+1:j+m}') \geq m \wt{\alpha}(Y,\wt{Y})/2\mid \wt{Y},Y\right) \\
        & \leq 2\exp\left(-m C_0^2\wt{\alpha}(Y,\wt{Y})^2/2\right).
    \end{align*}
    Recalling the definition of the test statistic $\phi$ via Algorithm~\ref{algorithm:test},
    the main claim then follows from taking a union bound over all $j \in [n]$ and recalling $k = m$ by assumption.
\end{proof}
\subsection{Proof of Lemma~\ref{lemma:perm-test-indel}}
\begin{proof}
    We begin with the following useful facts about edit distance. Throughout, let 
    $\mc{S}(y)$ denote the set of substrings of a string $y \in \mc{V}^*$, including the empty string.
    \begin{observation}\label{obs:edit-dist-basic}
        Let $y,\wt{y} \in \mc{V}^*$. Then $d_\txt{edit}(y,\wt{y})$ is the length of the smallest 
        sequence of insertion and/or deletion operations to obtain $\wt{y}$ from $y$. 
    \end{observation}
    \begin{proof}[Proof of Observation~\ref{obs:edit-dist-basic}]
        We proceed via induction on the sum $\len{y} + \len{\wt{y}}$. The base case 
        where $y$ and $\wt{y}$ are both empty is trivial.
        Now suppose the claim holds all strings whose lengths sum to at most 
        $\len{y} + \len{\wt{y}} - 1$.
        Recalling the definition of $d_\txt{edit}$ (Definition~\ref{defn:edit-dist}),
        there are three cases.

        First, suppose $d_\txt{edit}(y,\wt{y}) = d_\txt{edit}(y_{2:},\wt{y}_{2:})$. Then by induction 
        there exists a sequence of $d_\txt{edit}(y,\wt{y})$ insertion and/or deletion operations to obtain 
        $\wt{y}_{2:}$ from $y_{2:}$. Because $y_{1} = \wt{y}_{1}$, the same sequence suffices to obtain $\wt{y}$
        from $y$ and thus the claim follows.

        Second, suppose $d_\txt{edit}(y,\wt{y}) = 1 + d_\txt{edit}(y_{2:},\wt{y})$. Again by induction,
        there exists a sequence of $d_\txt{edit}(y,\wt{y}) - 1$ insertion and/or deletion operations to obtain 
        $\wt{y}$ from $y_{2:}$.
        It follows immediately (i.e., by first deleting $y_1$) there exists a sequence of $d_\txt{edit}(y,\wt{y})$ such operations 
        to obtain $\wt{y}$ from $y$, and so the claim holds.

        The third case follows by symmetry with the second case.
    \end{proof}
    \begin{observation}\label{obs:edit-dist-0}
        Let $y,\wt{y} \in \mc{V}^*$. Then for any $\tau < \len{y}$, we have
        \begin{align*}
            d_\txt{edit}(y,\wt{y}) \geq 
            \min_{y' \in \mc{S}(\wt{y})} d_\txt{edit}(y_{:\tau},y') + \min_{y' \in \mc{S}(\wt{y})} d_\txt{edit}(y_{\tau+1:},y').
        \end{align*}
    \end{observation}
    \begin{proof}[Proof of Observation~\ref{obs:edit-dist-0}]
        Observation~\ref{obs:edit-dist-basic} implies there exists a sequence of $d_\txt{edit}(y,\wt{y})$ insertion 
        and/or deletion operations to obtain $\wt{y}$ from $y$.
        We may partition this sequence of operations into sequences based respectively on whether they occur on $y_{:\tau}$ or $y_{\tau+1:}$.
        Let $\wt{y}_\txt{pre}$ be the result of performing the first sequence of operations on $y_{:\tau}$ and let 
        $\wt{y}_\txt{suf}$ be the result of performing the second sequence of operations on $y_{\tau+1:}$.
        Then $\wt{y}$ is the concatenation of $\wt{y}_\txt{pre}$ and $\wt{y}_\txt{suf}$, and so the claim follows from the fact that
        \begin{align*}
            d_\txt{edit}(y,\wt{y}) & = d_\txt{edit}(y_{:\tau},\wt{y}_\txt{pre}) + d_\txt{edit}(y_{\tau+1:},\wt{y}_\txt{suf}) \\
            & \geq \min_{y' \in \mc{S}(\wt{y})} d_\txt{edit}(y_{:\tau},y') + \min_{y' \in \mc{S}(\wt{y})} d_\txt{edit}(y_{\tau+1:},y').
        \end{align*}
    \end{proof}
    \begin{observation}\label{obs:edit-dist-1}
        Let $y,\wt{y} \in \mc{V}^*$ and $\xi \in \Xi^*$. Then 
        $d_\gamma(y,\xi) \leq \gamma d_\txt{edit}(y,\wt{y}) + d_\gamma(\wt{y},\xi)$.
    \end{observation}
    \begin{proof}[Proof of Observation~\ref{obs:edit-dist-1}]
        The case $d_\txt{edit}(y,\wt{y}) = 0$ is trivial as we then have $y = \wt{y}$.
        Now suppose $d_\txt{edit}(y,\wt{y}) = 1$, and let $i$ be the first index such that 
        $y_i \neq \wt{y}_i$.
        Then, unrolling the recursive definition of $d_\gamma(\wt{y}_{i:},\xi_{j:})$, there must exist 
        $c \in \R$ and an index $j$ such that both
        $d_\gamma(\wt{y},\xi) = c + d_\gamma(\wt{y}_{i:},\xi_{j:})$ and 
        $d_\gamma(y,\xi) \leq c + d_\gamma(y_{i:},\xi_{j:})$.
        Moreover, from the definition of edit distance, either ${y}_{i+1:} = \wt{y}_{i:}$
        or vice versa.

        We claim $d_\gamma(y_{i:},\xi_{j:}) \leq d_\gamma(\wt{y}_{i:},\xi_{j:}) + \gamma$.
        If $y_{i+1:} = \wt{y}_{i:}$, then the claim obtains as 
        \begin{align*}
            d_\gamma(y_{i:},\xi_{j:}) 
            & \leq d_\gamma(y_{i+1:},\xi_{j:}) + \min_{\xi' \in \Xi} d_0(y_i,\xi') + \gamma \\
            & \stackrel{(\star)}{\leq} d_\gamma(y_{i+1:},\xi_{j:}) + \gamma \\
            & = d_\gamma(\wt{y}_{i:},\xi_{j:}) + \gamma,
        \end{align*}
        with $(\star)$ following from the fact that $d_0(y_i,\xi') = 0$ for $\xi' = (1/2,\pi)$ irrespective of 
        $y_i$ and $\pi$.

        Otherwise, if $y_{i:} = \wt{y}_{i+1:}$, then from unrolling the recursive definition of $d_\gamma(\wt{y}_{i:},\xi_{j:})$
        there must exist some index $j' \geq j$ such that either
        \begin{align*}
            d_\gamma(\wt{y}_{i:},\xi_{j:}) = d_\gamma(\wt{y}_{i+1:},\xi_{j':}) + \gamma + \min_{\xi' \in \Xi} d_0(\wt{y}_i,\xi') + \sum_{j \leq \ell < j'} \gamma + \min_{y' \in \mc{V}} d_0(y',\xi_\ell)
        \end{align*}
        or 
        \begin{align*}
            d_\gamma(\wt{y}_{i:},\xi_{j:}) = d_\gamma(\wt{y}_{i+1:},\xi_{j'+1:}) + d_0(\wt{y}_i,\xi_{j'}) + \sum_{j \leq \ell < j'} \gamma + \min_{y' \in \mc{V}} d_0(y',\xi_\ell).
        \end{align*}
        In the first case, we have $\gamma + \min_{\xi' \in \Xi} d_0(\wt{y}_i,\xi') > 0$ since $\gamma > 1/2$ by assumption, and so the claim follows as 
        \begin{align*}
            d_\gamma(y_{i:},\xi_{j:}) & \leq d_\gamma(y_{i:},\xi_{j':}) + \sum_{j \leq \ell < j'} \gamma + \min_{y' \in \mc{V}} d_0(y',\xi_\ell) \\
            & = d_\gamma(\wt{y}_{i+1},\xi_{j':}) + \sum_{j \leq \ell < j'} \gamma + \min_{y' \in \mc{V}} d_0(y',\xi_\ell) \\
            & < d_\gamma(\wt{y}_{i:},\xi_{j:}).
        \end{align*}
        In the second case, we have $d_0(\wt{y}_j)$ the claim follows as 
        \begin{align*}
            d_\gamma(y_{i:},\xi_{j:}) & \leq d_\gamma(y_{i:},\xi_{j'+1:}) + \sum_{j \leq \ell < j'+1} \gamma + \min_{y' \in \mc{V}} d_0(y',\xi_\ell) \\
            & = d_\gamma(\wt{y}_{i+1},\xi_{j'+1:}) + \sum_{j \leq \ell < j'+1} \gamma + \min_{y' \in \mc{V}} d_0(y',\xi_\ell) \\
            & \leq d_\gamma(\wt{y}_{i:},\xi_{j:}) + \gamma.
        \end{align*}
        
        Thus, assuming $d_\txt{edit}(y,\wt{y}) \leq 1$, we have shown $d_\gamma(y_{i:},\xi_{j:}) \leq d_\gamma(\wt{y}_{i:},\xi_{j:}) + \gamma$,
        from which it follows that $d_\gamma(y,\xi) \leq d_\gamma(\wt{y},\xi) + \gamma$.
        The general result follows immediately by applying Observation~\ref{obs:edit-dist-basic} and 
        summing the bound for a single edit over the (smallest) sequence of 
        edits to obtain $\wt{y}$ from $y$. 
    \end{proof}
Proceeding with the main proof, define for convenience the quantity 
\begin{align*}
    \est{\alpha}_\tau \defeq \frac{1}{k}\sum_{i=1}^k p(Y_{\tau+i} \mid Y_{:\tau+i-1}).
\end{align*}
Observe 
\begin{align}
    \alpha(Y) = \frac{k}{m}\sum_{\tau = 0}^{m/k-1} \est{\alpha}_{k \tau}, \label{eqn:pigeon-1}
\end{align}
while Observation~\ref{obs:edit-dist-0} together with our assumption that $d_\txt{edit}(Y,\wt{Y}) \leq \eps m$ implies 
\begin{align}
    \frac{k}{m}\sum_{\tau = 0}^{m/k-1} \min_{Y' \in \mc{S}(\wt{Y})} d_\txt{edit}(Y_{k\tau+1:k\tau+k},Y') \leq k \eps. \label{eqn:pigeon-2}
\end{align}
The displays~\eqref{eqn:pigeon-1} and \eqref{eqn:pigeon-2} together imply there exists an index 
$\tau$ and $Y' \in \mc{S}(\wt{Y})$ such that $\est{\alpha}_\tau - \frac{1}{k}\min_{Y' \in \mc{S}(\wt{Y})} d_\txt{edit}(Y_{\tau+1:\tau+k},Y') \geq \alpha(Y) - \eps$.
Reusing the same concentration argument as in the proof of Theorem~\ref{lemma:perm-test}, for $t \geq 0$ we have
\begin{align*}
    \P\left(d_0(Y_{\tau+1:\tau+k},\xi_{\tau+1:\tau+k}) \geq -k \left(C_0\est{\alpha}_\tau + t\right)\mid Y\right)
    & \leq \exp\left(-2 k t^2\right),
\end{align*}
and thus from Observation~\ref{obs:edit-dist-1} it follows that
\begin{align*}
    \P\left(d_\gamma(Y',\xi_{\tau+1:\tau+k}) \geq -k \left(C_0\alpha(Y) - \gamma \eps + t \right)\mid \wt{Y},Y\right) \leq \exp\left(-2 k t^2\right).
\end{align*}
Letting $t = (C_0 \alpha - \gamma \eps)/2$ and recalling the definition of the test statistic, we have
\begin{align}
    \P\left(\phi(\wt{Y},\xi) \geq -k \left(C_0\alpha(Y) - \gamma \eps\right)/2 \mid \wt{Y},Y\right) \leq \exp\left(-k (C_0\alpha(Y) - \gamma \eps)_{+}^2/2\right). \label{eqn:indel-main-1}
\end{align}

All that remains to bound the probability of $\phi(\wt{Y},\xi')$ exceeding the 
threshold from the above display.
To this end, define the set-valued map $\mc{N}_\beta(y) \defeq \{y' \ : \ d_\txt{edit}(y,y') \leq \beta/(4\gamma - 1)\}$.
Then we make the following observation.
\begin{observation}\label{obs:edit-dist-2}
    For any $y \in \mc{V}^*$ and $\xi \in \Xi^*$, there exists $y' \in \mc{N}_{\len{\xi}}(y)$ 
    such that 
    \begin{align*}
        d_\gamma(y,\xi) = \gamma \cdot d_\txt{edit}(y,y') + d_0(y',\xi).
    \end{align*}
\end{observation}
\begin{proof}
    We proceed via induction. The base case where $y$ and $\xi$ both have 
    length $1$ follows trivially by taking $y' = y$; in particular,
    $\gamma > 1/2$ implies $d(y,\xi) \leq \gamma + \min_{y'} d(y',\xi)$ and
    likewise $d(y,\xi) \leq \gamma + \min_{\xi'} d(y,\xi')$.
    Now suppose the result holds so long as $\len{y} + \len{\xi} \leq n-1$.
    We claim that the result must then also hold if the lengths sum to at most $n$.

    We prove this inductive claim by considering three exhaustive cases.
    First, suppose that $d_\gamma(y,\xi) = d_\gamma(y_{2:},\xi_{2:}) + d(y_1,\xi_1)$.
    By our induction hypothesis, there exists
    $\hat{y} \in \mc{N}_{\len{\xi}-1}(y_{2:})$ such that 
    $d_\gamma(y_{2:},\xi_{2:}) = \gamma \cdot d_\txt{edit}(y_{2:},\hat{y}) + d(\hat{y},\xi_{2:})$.
    The desired result then obtains with $y'$ as the concatenation of
    $y_1$ and $\hat{y}$.
    Second, suppose $d_\gamma(y,\xi) = d_\gamma(y,\xi_{2:}) + \min_{\xi' \in \Xi} d(y_1,\xi') + \gamma$.
    By our induction hypothesis, there exists $\hat{y} \in \mc{N}_{\len{\xi}=1}(y)$
    such that $d_\gamma(y_{2:},\xi) = \gamma \cdot d_\txt{edit}(y_{2:},\hat{y}) + d(\hat{y},\xi_{2:})$.
    The result obtains with $y' = \hat{y}$.
    Finally, suppose $d_\gamma(y,\xi) = d_\gamma(y_{2:},\xi) + d(y'',\xi_1) + \gamma$ for some $y'' \in \mc{V}$.
    By our induction hypothesis, there exists $\hat{y} \in \mc{N}_{\len{\xi}-1}(y)$
    such that $d_\gamma(y_{2:},\xi) = \gamma \cdot d_\txt{edit}(y_{2:},\hat{y}) + d(\hat{y},\xi)$.
    The result then obtains by concatenating $y''$ with $\est{y}$.
\end{proof}
Let $\mc{I}_j \defeq \{(j+i)\%n\}_{i=1}^k$.
For any $0 \leq i \leq \len{\wt{Y}} - k$ and $j \in [n]$,
Observations~\ref{obs:edit-dist-1} and~\ref{obs:edit-dist-2} together imply that 
\begin{align}
    d_\gamma(\wt{Y}_{i+1:i+k},\xi_{\mc{I}_j}') 
    & = \min_{y \in \mc{N}_k(\wt{Y}_{i+1:i+k})} \gamma \cdot d_\txt{edit}(\wt{Y}_{i+1:i+k},y) + d_0(y,\xi_{\mc{I}_j}') \label{eqn:edit-min}\\
    & \stackrel{(\star)}{=} \min_{y \in \mc{N}_{k/4(\gamma-1)}(\wt{Y}_{i+1:i+k})} \gamma \cdot d_\txt{edit}(y,\wt{Y}_{i+1:i+k}) + d_0(y,\xi_{\mc{I}_j}'), \label{eqn:edit-reduce}
\end{align}
where ($\star$) follows from the fact that $d_\txt{edit}(\wt{Y}_{i+1:i+k},y) > k/4(\gamma-1)$ implies 
\begin{align*}
    \gamma \cdot d_\txt{edit}(\wt{Y}_{i+1:i+k},y) + d_0(y,\xi_{\mc{I}_j}') \geq k/4 > d_0(\wt{Y}_{i+1:i+k},\xi_{\mc{I}_j}'),
\end{align*}
and therefore the minimizer in equation~\eqref{eqn:edit-min} must be an element of $\mc{N}_{k/4(\gamma-1)}(\wt{Y}_{i+1:i+k})$.

By construction, $\mc{N}_\beta(y)$ consists of the set of strings obtainable from $y$ by a sequence of 
at most $\beta$ insertion and/or deletion operations.
Now define another set-valued map $\mc{N}_{\beta,-}(y)$ as the restriction of $\mc{N}_\beta(y)$
such that we may only insert a particular token into $y$ (which token is immaterial).
As the specific identity of each token we insert into $y$ can only influence the value of $d_\gamma$ by $\pm 1/2$,
for any $\beta$ it follows that 
\begin{align*}
    \min_{y \in \mc{N}_{\beta}(\wt{Y}_{i+1:i+k})} \gamma \cdot d_\txt{edit}(y,\wt{Y}_{i+1:i+k}) + d_0(y,\xi_{\mc{I}_j}')
    \geq \min_{y \in \mc{N}_{\beta,-}(\wt{Y})} d_0(y,\xi_{\mc{I}_j}'),
\end{align*}
and so, letting $\beta = k/4(\gamma-1)$, from equation~\eqref{eqn:edit-reduce} we have 
\begin{align*}
    d_\gamma(\wt{Y}_{i+1:i+k},\xi_{\mc{I}_j}') \geq \min_{y \in \mc{N}_{\beta,-}(\wt{Y}_{i+1:i+k})} d_0(y,\xi_{\mc{I}_j}')
\end{align*}
Let $\wt{Y}(i,\ell)$ denote the $\ell$-th element of $\mc{N}_{\beta,-}(\wt{Y}_{i+1:i+k})$ for some $\wt{Y}$-measurable indexing.
From the independence of $\wt{Y}$ and $\xi'$, we have $\Ep[d_0(\wt{Y}(i,\ell),\xi_{\mc{I}_j}) \mid \wt{Y}] = 0$
for any $\ell$ and $j$.
The cardinality of $\mc{N}_{\beta,-}(\wt{Y}_{i+1:i+k})$ is equal to the number of possible combinations of locations for $\beta$ 
insertion and/or deletion operations on $\wt{Y}$, of which there are at most $(k+\beta)^\beta \leq (2k)^\beta$.
Thus, applying the same concentration argument as in the proof of Theorem~\ref{lemma:perm-test}
and taking a union bound over all $i \leq m - k$, $j \leq n$ and $\ell \leq (2k)^\beta$, we have
\begin{align}
    \P(\phi(\wt{Y},\xi') \leq -\alpha(Y)/2 + \gamma \eps \mid \wt{Y},Y) \leq m n(2k)^{k/(4\gamma-1)}\exp(-k C_0^2 (\alpha(Y) - \gamma\eps)_+^2/2). \label{eqn:indel-main-2}
\end{align}
Combining the displays~\eqref{eqn:indel-main-1} and~\eqref{eqn:indel-main-2} via another union bound 
gives the desired result.
\end{proof}

\section{Analysis of exponential minimum sampling}\label{app:gumbel}
To prove the main theorems, we introduce the following supporting lemma.
The result is well known and we restate it here only for completeness.
\begin{lemma}\label{lemma:core-result-gumbel}
    Let $\mu \in \Delta([N])$ and $\xi \sim \txt{Unif}([0,1]^N)$. Then for any 
    $y \in [N]$ we have
    \begin{align*}
        \P(\decode(\xi,\mu) = y,-\log (\xi_y)/\mu(y) \geq t) = \mu(y) \exp(-t).
    \end{align*} 
\end{lemma}
\begin{proof}
    Suppose $\mu(y) > 0$ as otherwise the claim is trivial.
    Recalling $\xi_i \iid \txt{Unif}([0,1])$, for any $\lambda > 0$ we have
    $-\lambda \log \xi_i \iid \txt{Exp}(\lambda)$,
    i.e.,
    \begin{align*}
        \P(-\lambda \log \xi_i \geq t) = \P(\xi_i \leq \exp(-\lambda t)) = \exp(-\lambda t). 
    \end{align*}
    Thus, the claim follows as
    \begin{align*}
        & \P(\decode(\xi,\mu) = y,-\log (\xi_y)/\mu(y) \geq t) \\
        & = \P(y = \argmin_i -\log(\xi_i)/\mu(i),-\log (\xi_y)/\mu(y) \geq t) \\
        & \stackrel{(\star)}{=} \int_{u \geq t} \mu(y)\exp(-\mu(y) u) \cdot \Pi_{i \in \txt{supp}(\mu),i \neq y} \P(-\log (\xi_i)/\mu(i) > u) \\
        & = \int_{u \geq t} \mu(y)\exp(-\mu(y) u) \cdot \Pi_{i \in \txt{supp}(\mu),i \neq y} \exp(-\mu(i) u) \\
        & = \mu(y) \int_{u \geq t} \Pi_{i \in \txt{supp}(\mu)} \exp(-\mu(i) u) \\
        & = \mu(y) \int_{u \geq t} \exp(-u) \\
        & = \mu(y)\exp(-t),
    \end{align*}
    where in ($\star$) we use the fact that the density of $-\log (\xi_y)/\mu(y)$ at $u$ is $\mu(y) \exp(-\mu(y)u)$.
\end{proof}

\subsection{Proof of Theorem~\ref{thm:stealth-gumbel}}
\begin{proof}
    The result follows immediately from integrating the result of Lemma~\ref{lemma:core-result-gumbel}
    over $t \geq 0$.
\end{proof}

\subsection{Proof of Lemma~\ref{lemma:sound-gumbel}}
\begin{proof}
    Lemma~\ref{lemma:core-result-gumbel} implies 
    $-\log(\xi_i) / p(Y_i \mid Y_{:i-1}) \mid Y \sim \txt{Exp}(1)$,
    and thus $\Ep[-\log(\xi_i) \mid Y] = p(Y_i \mid Y_{:i-1})$.
    Meanwhile, as $\xi_i' \sim \txt{Unif}([0,1])$ independently of $Y$, 
    we have 
    \begin{align*}
        \P(- \log \xi_i' \geq t \mid Y) = \P(\xi_i' \leq \exp(-t)) = \exp(-t),
    \end{align*}
    implying $-\log(\xi_i') \mid Y \sim \txt{Exp}(1)$ and so $\Ep [-\log(\xi_i') \mid Y] = 1$.
    The result follows immediately, recalling $\alpha(Y_{i-1:i}) = 1 - p(Y_i \mid Y_{i-1})$
    by definition.
\end{proof}

\subsection{Proof of Lemma~\ref{lemma:perm-test-gumbel}}
We prove the following general result, from which Lemma~\ref{lemma:perm-test-gumbel} follows as a corollary.
\begin{lemma}
    Let $m,n \in \N$ with $n \geq m$, where $m$ is the generation length and $n$ is the watermark key length. 
    Define the decoder $\decode$ by equation~\eqref{eqn:rho-defn-gumbel}, alignment score $d$ by equation~\eqref{eqn:d-defn-gumbel},
    and $\phi$ by Algorithm~\ref{algorithm:test} with block size $k \leq m$.
    Let $\xi,\xi' \iid \txt{Unif}(\Xi^n)$ with $Y = \generate(\xi;n,p,\decode)$.
    Let $\wt{Y}$ be a substring of $Y$ of length at least $k$ that is conditionally independent of 
    $\xi$ and $\xi'$ given $Y$, i.e., $\wt{Y} = Y_{\tau+1:\tau+\ell}$ for $\ell \geq k$.
    Then for $\est{\alpha} \defeq 1 - \frac{1}{k}\sum_{i=\tau+1}^{\tau+k} p(Y_i \mid Y_{:i-1})$, almost surely
    \begin{align*}
        \P(\phi(\wt{Y},\xi') \leq \phi(\wt{Y},\xi) \mid \wt{Y},Y) \leq 2n\exp\left(-\min\{k\est{\alpha}^2/8,k\est{\alpha}/4\}\right).
    \end{align*}
\end{lemma}
\begin{proof}
    Recall by definition
    \begin{align*}
        d(y,\xi) = -\sum_{i=1}^{\len{y}} \log \xi_{i,y_i}.
    \end{align*}
    Lemma~\ref{lemma:sound-gumbel} and the conditional independence of $\tau$ and $\xi$ given $Y$ 
    imply for any $j \in [n]$ that 
    \begin{align*}
        \Ep[d(\wt{Y},\xi_{(j+1:j+k)\% n}') \mid \wt{Y},Y] - \Ep[d(\wt{Y},\xi_{\tau+1:\tau+k}) \mid \wt{Y},Y] = k \est{\alpha}.
    \end{align*}
    
    From Lemma~\ref{lemma:core-result-gumbel}, we have $-\log \xi_{\tau+i,\wt{Y}_i} \mid \wt{Y},Y \sim \txt{Exp}(\gamma_i)$
    for some $\gamma_i \leq 1$ for all $i \in [m]$.
    Also, from the independence of $\wt{Y}$ and $\xi'$,
    we have $-\log \xi_{j,\wt{Y}_i}' \mid \wt{Y},Y \sim \txt{Exp}(1)$
    for all $i \in [m]$ and $j \in [n]$.
    The following observation thus implies
    $-\log \xi_{i,\wt{Y}_i} \mid \wt{Y},Y$ and $-\log \xi_{j,\wt{Y}_i}' \mid \wt{Y},Y$ 
    are both $(2,2)$-subexponential random variables.
    \begin{observation}\label{obs:gumbel-subexp}
        Let $X \sim \txt{Exp}(1)$. Then $X$ is a $(2,2)$ subexponential random variable.
    \end{observation}
    \begin{proof}[Proof of Observation~\ref{obs:gumbel-subexp}]
        For $t < 1/2$, we have
        \begin{align*}
            \Ep[e^{t (X-\Ep[X])}] & = \int_0^\infty e^{t(x-1)} e^{-x} \, dx \\
            & \stackrel{(a)}{=} \frac{e^{-t}}{1-t} \\
            & \stackrel{(b)}{\leq} (1 - t + t^2)(1 + t + 2t^2) \\
            & \stackrel{(c)}{\leq} (1 + 2t^2) \\
            & \leq e^{2t^2},
        \end{align*}
        where (a) follows from the fact that $t < 1$ 
        (otherwise, the integral would not be finite);
        (b) follows from Taylor expanding $e^{-t}$ and $1/(1-t)$
        and applying the fact that $t < 1/2$ to bound the higher-order terms;
        and (c) again follows from $t < 1/2$.
        The claim follows immediately.
    \end{proof}

    Thus, using the fact that $\xi_i$ is conditionally independent of $\xi_{-i}$ given $Y$,
    a standard Chernoff bound \cite[Proposition 2.9]{wainwright19} implies for each $j \in [n]$ that 
    \begin{align*}
        & \P\left(d(\wt{Y},\xi_{j+1:j+k}') \leq d(\wt{Y},\xi_{\tau+1:\tau+k})\mid \wt{Y},Y\right) \\
        & \leq \P\left(d(\wt{Y},\xi_{1:m}) - \Ep[d(\wt{Y},\xi_{1:m})] \geq k\est{\alpha}/2 \mid \wt{Y},Y\right) \\
        &  \qquad + \P\left(\Ep[d(\wt{Y},\xi_{j+1:j+m}')] - d(\wt{Y},\xi_{j+1:j+m}') \geq k \est{\alpha}/2 \mid \wt{Y},Y\right) \\
        & \leq 2\exp\left(-\min\{k\est{\alpha}^2/8,k\est{\alpha}/4\}\right).
    \end{align*}
    Recalling the definition of the test statistic $\phi$ via Algorithm~\ref{algorithm:test},
    the main claim then follows from taking a union bound over all $j \in [n]$.
\end{proof}
\section{Details of experiments}\label{app:experiments}
\subsection{Experimental protocol}\label{app:experiments-protocol}
In Experiments 1-6, for each watermark we first generate a sequence tokens, decode the tokens 
into text (i.e., a string) using the appropriate tokenizer for the language model, and then 
encode the text back into tokens before running $\detect$.
Each generation is coditioned on a prompt; we obtain the prompts by sampling documents from the news-like 
subset of the C4 dataset and truncating the last $m$ tokens.
We enforce a minimum prompt size of $50$ tokens in all experiments; we skip over any document that is not long enough.
The retokenization is not always equal to the original tokens generated by the model;\footnote{
  Byte-pair tokenizations of text (used by both the OPT and LLaMA) are not unique, due to the fact that they augment a base vocabulary (e.g., characters) 
  with extra tokens to represent common substrings.
}
in order to ensure $\detect$ always receives 
at least $m$ tokens, we pad its input with special pad tokens (specific to each model's tokenizer).
We also initially generate a number of buffer tokens beyond $m$, so in most cases the padding is unnecessary.
We set the number of buffer tokens to be $20$ in every experiment except for Experiment 5, where we set it to be 
$100$ in order to ensure that even after deleting tokens there are typically still at least $m$ tokens remaining.
We always truncate the number of tokens given to $\detect$ to be at most $m$, irrespective of the number of buffer tokens.

\subsection{Roundtrip translation}\label{app:experiments-roundtrip}
In Experiment 6, we perform round-trip translations from English to French and from English to Russian using the OPUS-MT collection of translation models \citep{tiedemann2020opus,tiedemann2022democratizing}. 
Specifically, we use the versions of these models hosted on the HuggingfaceHub,\footnote{https://huggingface.co/} associated with the identifiers:
\begin{itemize}
\item \texttt{Helsinki-NLP/opus-mt-tc-big-en-fr} - English to French,
\item \texttt{Helsinki-NLP/opus-mt-tc-big-fr-en} - French to English,
\item \texttt{Helsinki-NLP/opus-mt-en-ru} - English to Russian,
\item \texttt{Helsinki-NLP/opus-mt-ru-en} - Russian to English.
\end{itemize}

\subsection{Computing p-values}\label{app:experiments-pvals}
As we mention previously, to save computation we modify $\detect$ to use a fixed reference distribution to compute 
$p$-values.
For the sake of concreteness, we give the full pseudocode for the modified version of $\detect$ in Algorithm~\ref{algorithm:modified-detect};
in Experiments 1-6, we compute $p$-values using Algorithm~\ref{algorithm:reference} to construct the reference distribution using the news-like subset of
the C4 dataset as the text distribution.

\begin{algorithm}[h]
    \DontPrintSemicolon
    \SetKwInOut{Input}{Input}
    \SetKwInOut{Output}{Output}
    \SetKwInOut{Params}{Params}
    \SetKwInOut{Noise}{Noise}
    \caption{\label{algorithm:modified-detect}
        Watermarked text detection with fixed reference distribution}
    \Input{string $y \in \mc{V}^*$, seed sequence $\xi \in \Xi^n$}
    \Params{test statistic $\phi$; reference distribution $\{\phi_t\}_{t=1}^T$}
    \Output{p-value $\est{p} \in [0,1]$}
    $\est{p} \leftarrow \frac{1}{T} \sum_{t=1}^T \indic{\phi(y,\xi) < \phi_t}$\;
    \Return $\est{p}$\;
\end{algorithm}

\begin{algorithm}[h]
    \DontPrintSemicolon
    \SetKwInOut{Input}{Input}
    \SetKwInOut{Output}{Output}
    \SetKwInOut{Params}{Params}
    \SetKwInOut{Noise}{Noise}
    \caption{\label{algorithm:reference}
        Reference distribution construction}
    \Input{resample size $T \in \N$, text length $m \in \N$, watermark key sequence distribution $\nu \in \Delta(\Xi^n)$}
    \Params{test statistic $\phi$; text distribution $P$; minimum prompt length $m_0$}
    \Output{reference distribution $\{\phi_t\}_{t=1}^T \in \R^T$}
    $t \leftarrow 1$\;
    \While{$t \leq T$}{
        $Y \sim P$\;
        \If{$\len{Y} \leq m_0 + m$}{
            continue\;
        }
        $\xi^{(t)} \sim \nu$\;
        $\phi_t \leftarrow \phi(Y_{-m:},\xi^{(t)})$\;
        $t \leftarrow t + 1$\;
    }
    \Return $\{\phi_t\}_{t=1}^T$\;
\end{algorithm}

As a sanity check, we include histograms of the $p$-values we compute for nonwatermarked text for each method to verify that 
they are roughly uniformly distributed on the interval $[0,1]$ (setting $m = 50$ and sampling prompts 
from the news-like subset of the C4 dataset, as in Experiment 1).
In the cases of {\kgwone} and {\kgwtwo}, the distribution is not quite uniform 
due to the discrete nature of their test statistics.

\begin{figure}[H]
  \centering
  \begin{subfigure}[b]{0.49\textwidth}
    \includegraphics[width=\textwidth]{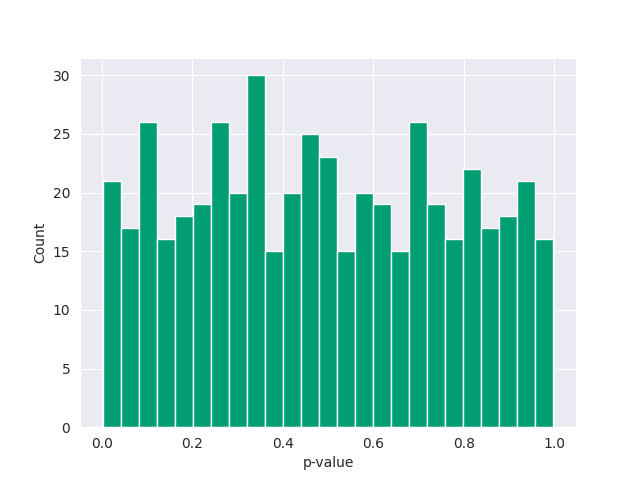}
    \caption{{OPT-1.3B}}
  \end{subfigure}
  \hfill
  \begin{subfigure}[b]{0.49\textwidth}
    \includegraphics[width=\textwidth]{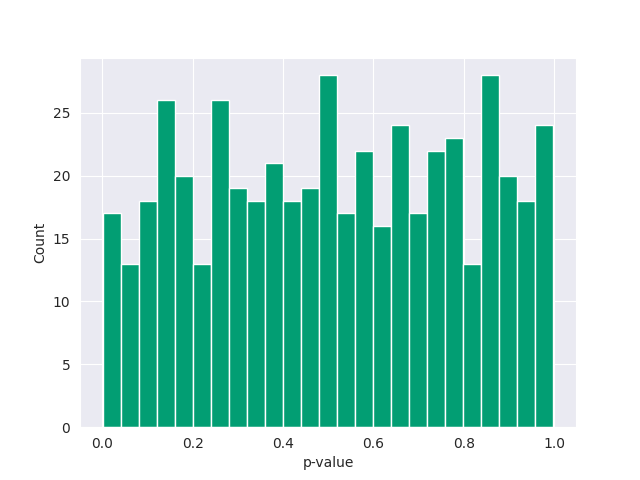}
    \caption{{LLaMA-7B}}
  \end{subfigure}
  \caption{Distribution of $p$-values for nonwatermarked text using {\its} detector.}
\end{figure}

\begin{figure}[H]
  \centering
  \begin{subfigure}[b]{0.49\textwidth}
    \includegraphics[width=\textwidth]{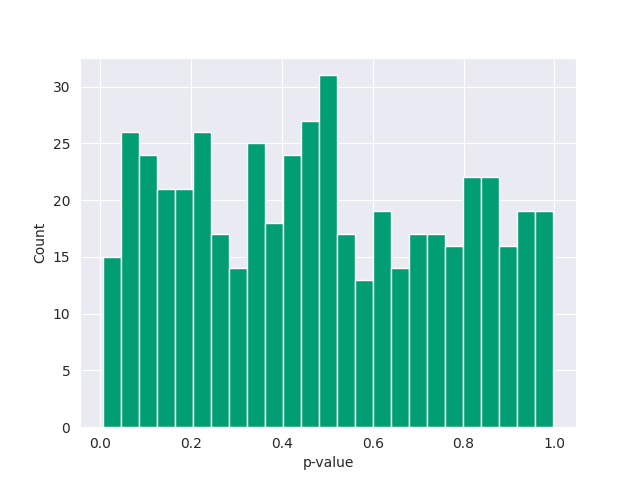}
    \caption{{OPT-1.3B}}
  \end{subfigure}
  \hfill
  \begin{subfigure}[b]{0.49\textwidth}
    \includegraphics[width=\textwidth]{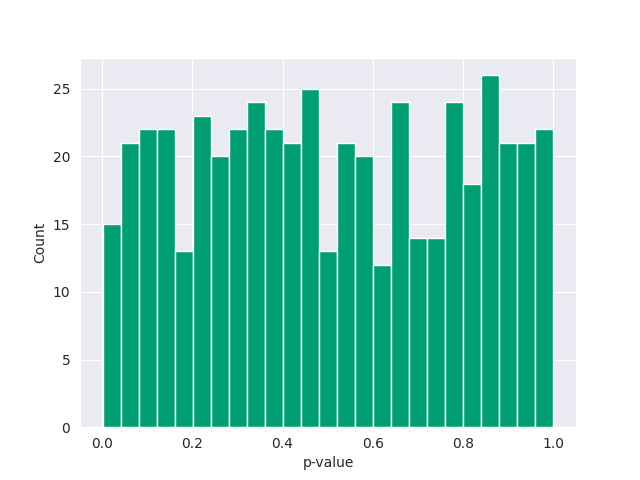}
    \caption{{LLaMA-7B}}
  \end{subfigure}
  \caption{Distribution of $p$-values for nonwatermarked text using {\itsedit} detector.}
\end{figure}

\begin{figure}[H]
  \centering
  \begin{subfigure}[b]{0.49\textwidth}
    \includegraphics[width=\textwidth]{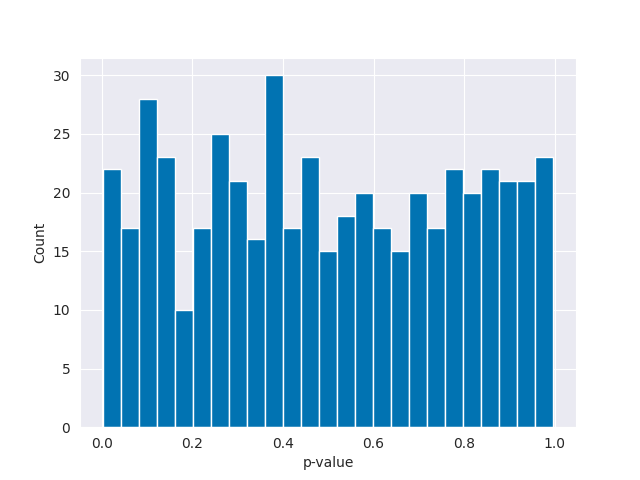}
    \caption{{OPT-1.3B}}
  \end{subfigure}
  \hfill
  \begin{subfigure}[b]{0.49\textwidth}
    \includegraphics[width=\textwidth]{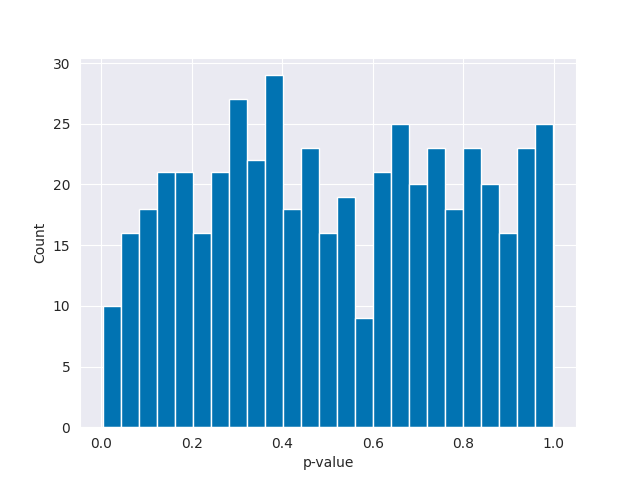}
    \caption{{LLaMA-7B}}
  \end{subfigure}
  \caption{Distribution of $p$-values for nonwatermarked text using {\gumb} detector.}
\end{figure}

\begin{figure}[H]
  \centering
  \begin{subfigure}[b]{0.49\textwidth}
    \includegraphics[width=\textwidth]{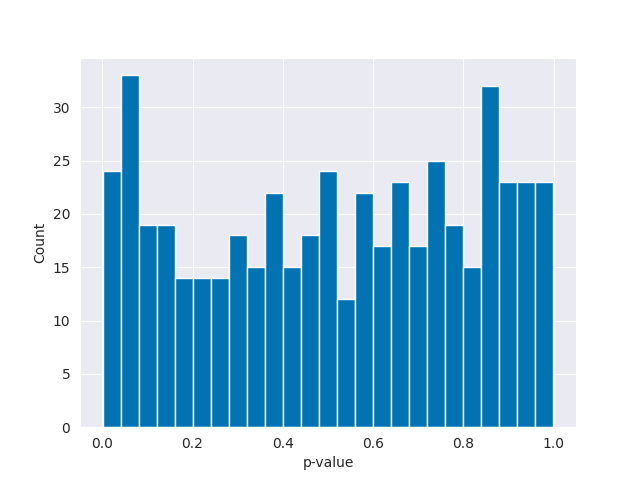}
    \caption{{OPT-1.3B}}
  \end{subfigure}
  \hfill
  \begin{subfigure}[b]{0.49\textwidth}
    \includegraphics[width=\textwidth]{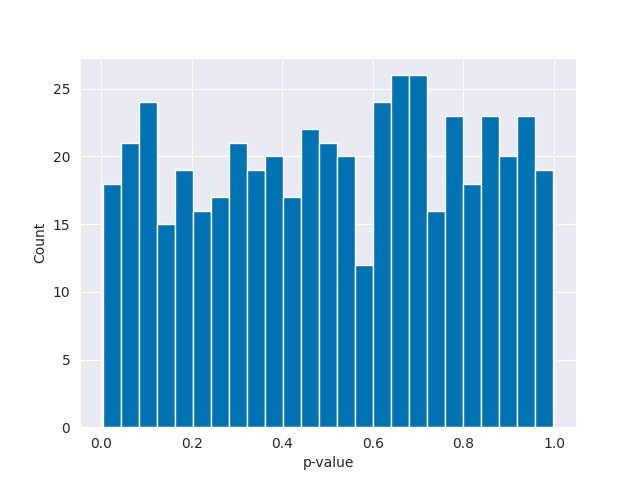}
    \caption{{LLaMA-7B}}
  \end{subfigure}
  \caption{Distribution of $p$-values for nonwatermarked text using {\gumbedit} detector.}
\end{figure}

\begin{figure}[H]
  \centering
  \begin{subfigure}[b]{0.49\textwidth}
    \includegraphics[width=\textwidth]{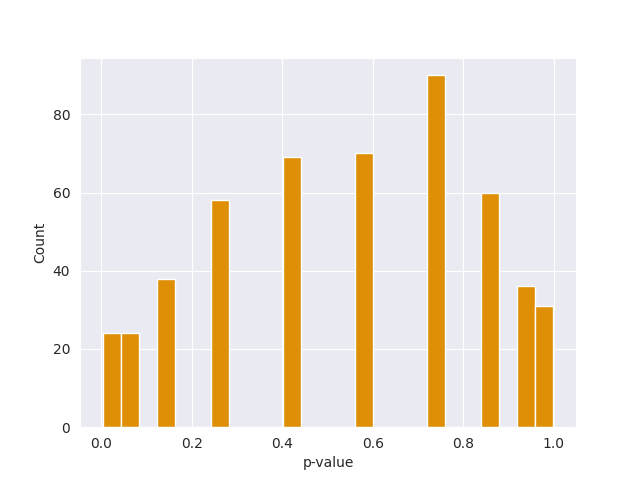}
    \caption{{OPT-1.3B}}
  \end{subfigure}
  \hfill
  \begin{subfigure}[b]{0.49\textwidth}
    \includegraphics[width=\textwidth]{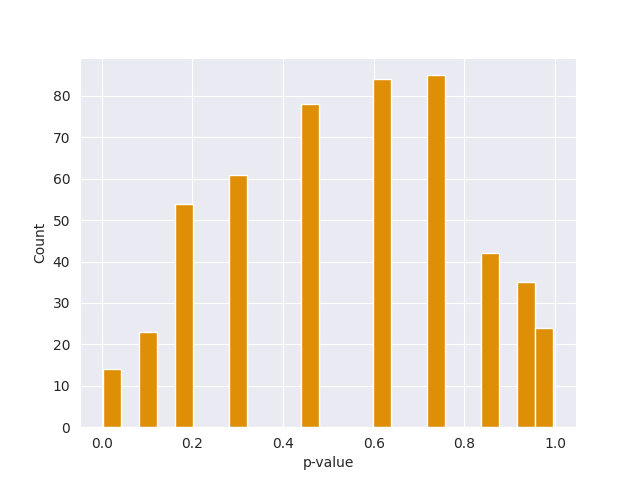}
    \caption{{LLaMA-7B}}
  \end{subfigure}
  \caption{Distribution of $p$-values for nonwatermarked text using {\kgwone} detector.}
\end{figure}

\begin{figure}[H]
  \centering
  \begin{subfigure}[b]{0.49\textwidth}
    \includegraphics[width=\textwidth]{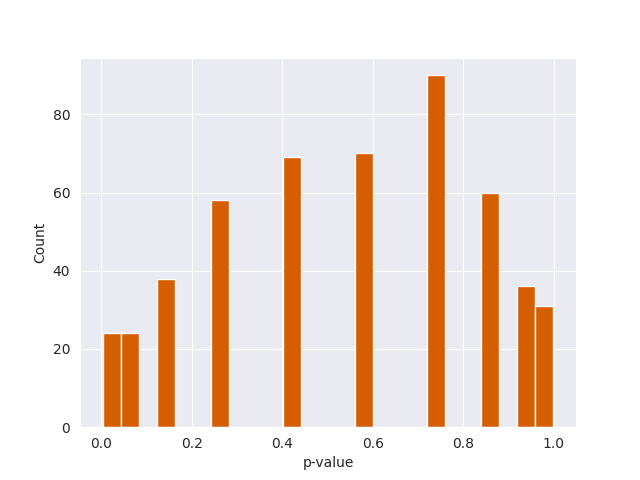}
    \caption{{OPT-1.3B}}
  \end{subfigure}
  \hfill
  \begin{subfigure}[b]{0.49\textwidth}
    \includegraphics[width=\textwidth]{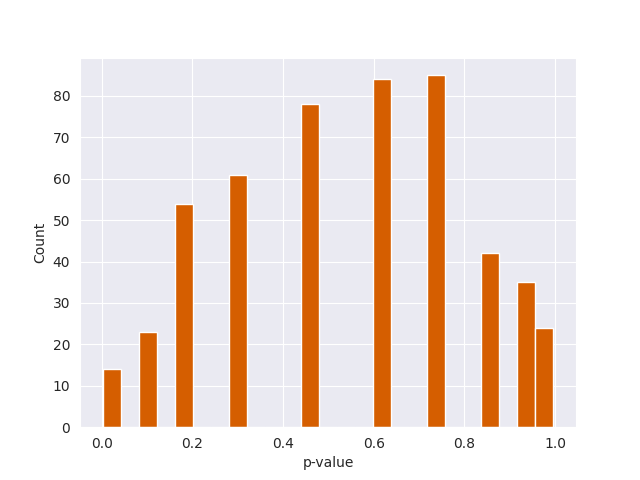}
    \caption{{LLaMA-7B}}
  \end{subfigure}
  \caption{Distribution of $p$-values for nonwatermarked text using {\kgwtwo} detector.}
\end{figure}

\subsection{Hyperparameter tuning}\label{app:experiments-hyperparam}
There are two hyperparameters involved in computing each of our watermark test statistics (i.e., Algorithm~\ref{algorithm:test}),
the block size $k$ and the alignment score $d$.
We do not tune the block size $k$ for our experiments, instead simply letting $k = m$, i.e., the text length, and the alignment score 
is also fixed for each of our watermarks, except for the hyperparameter $\gamma$ in both {\itsedit} and {\gumbedit}.
Smaller values of $\gamma$ (at least to a certain point) tend to make these watermarks more robust to insertion and deletion errors,
as Figure~\ref{fig:tuning-eps} illustrates,
but also hurts their statistical power for large values of $n$, i.e., the watermark key length, as Figure~\ref{fig:tuning-rng} illustrates.
We set $\gamma = 0.4$ for {\itsedit} and $\gamma = 0.0$ for {\gumbedit} to balance these two competing desiderata.

\begin{figure}[H]
  \centering
  \begin{subfigure}[b]{0.49\textwidth}
    \includegraphics[width=\textwidth]{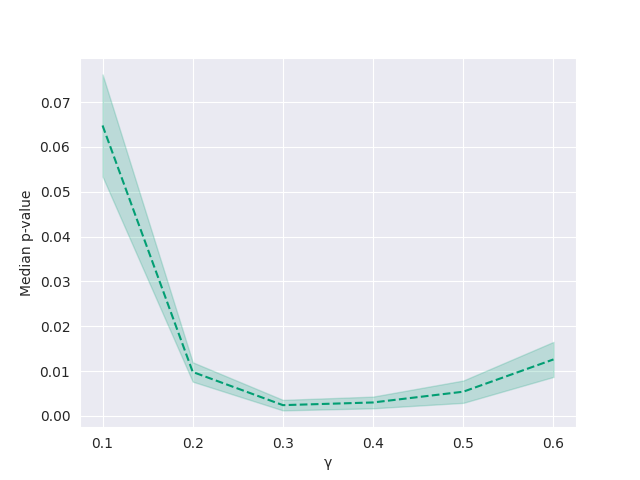}
    \caption{{\itsedit}}
  \end{subfigure}
  \hfill
  \begin{subfigure}[b]{0.49\textwidth}
    \includegraphics[width=\textwidth]{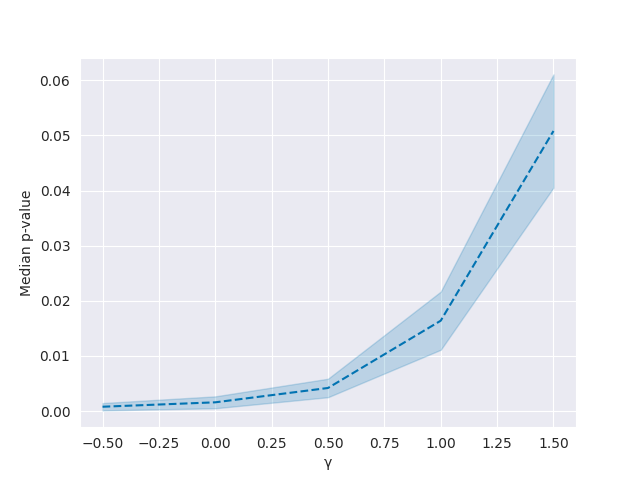}
    \caption{{\gumbedit}}
  \end{subfigure}

  \caption{Median $p$-value of watermarked text for varying $\gamma$, 
  with OPT-1.3B models and $m = 70$ for {\itsedit} and $m = 35$ for {\gumbedit},
  after corrupting the text with random insertions (fraction of inserted tokens is $0.1$
  for {\itsedit} and $0.6$ for {\gumbedit}).
  }
  \label{fig:tuning-eps}
\end{figure}

\begin{figure}[H]
  \centering
  \begin{subfigure}[b]{0.49\textwidth}
    \includegraphics[width=\textwidth]{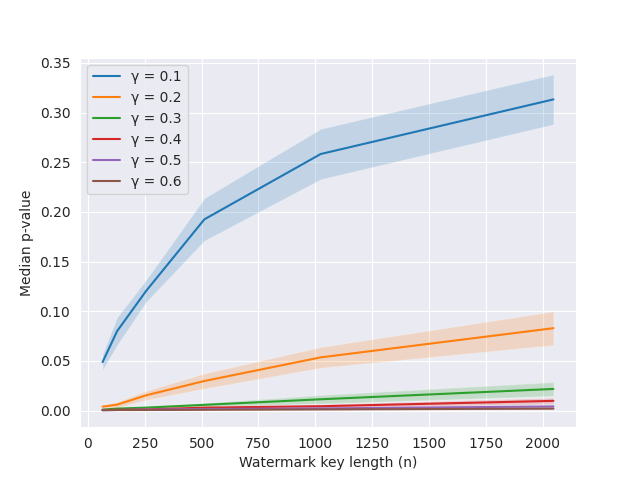}
    \caption{{\itsedit}}
  \end{subfigure}
  \hfill
  \begin{subfigure}[b]{0.49\textwidth}
    \includegraphics[width=\textwidth]{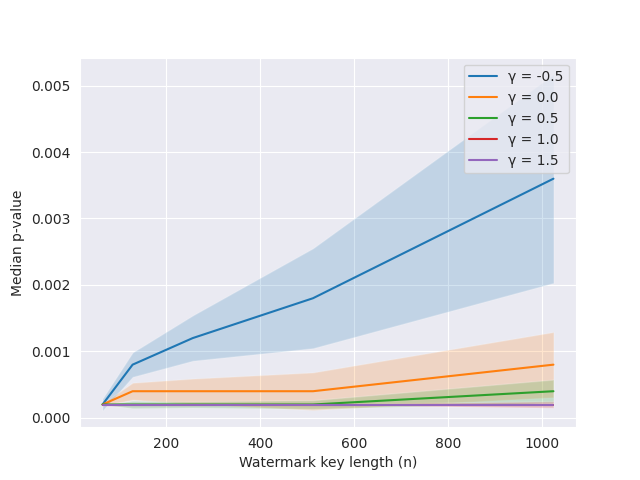}
    \caption{{\gumbedit}}
  \end{subfigure}
  \caption{Median $p$-value of watermarked text, varying $\gamma$ and $n$, 
  with OPT-1.3B model and $m = 40$ for {\itsedit} and $m = 10$ for {\gumbedit}.
  }
  \label{fig:tuning-rng}
\end{figure}

\subsection{Deferred results}\label{app:experiments-defer}
\subsubsection{Experiment 3}
\begin{figure}[H]
    \centering
    \begin{subfigure}[b]{0.49\textwidth}
      \includegraphics[width=\textwidth]{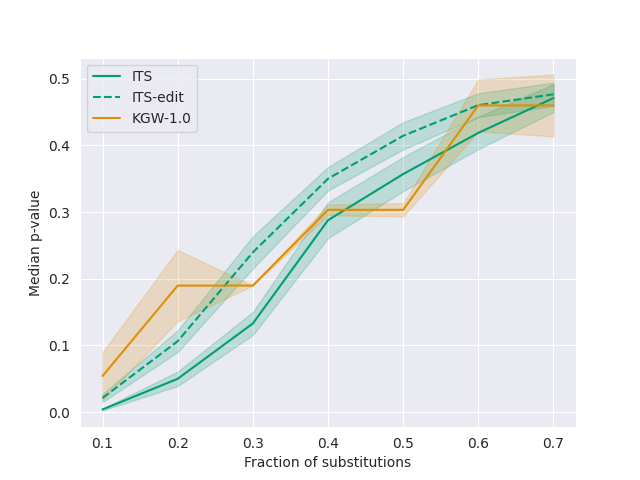}
      \caption{OPT-1.3B}
    \end{subfigure}
    \hfill
    \begin{subfigure}[b]{0.49\textwidth}
      \includegraphics[width=\textwidth]{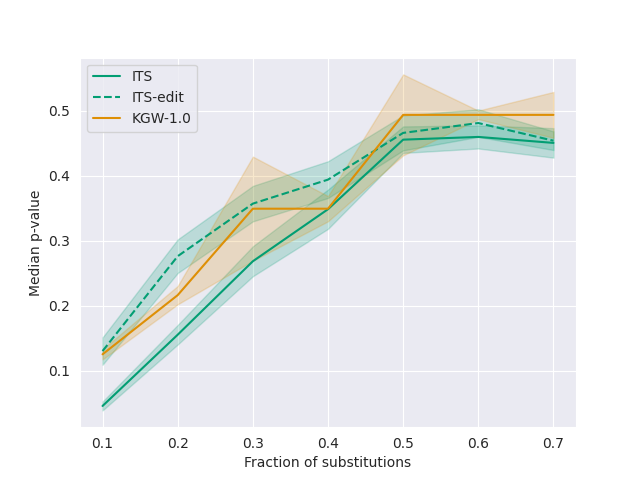}
      \caption{LLaMA 7B}
    \end{subfigure}
  
    \caption{Median $p$-value of watermarked text relative to the fraction of substitution errors, 
    for OPT-1.3B and LLaMA 7B models with $m = 35$.
    }
\end{figure}

\begin{figure}[H]
    \centering
    \begin{subfigure}[b]{0.49\textwidth}
      \includegraphics[width=\textwidth]{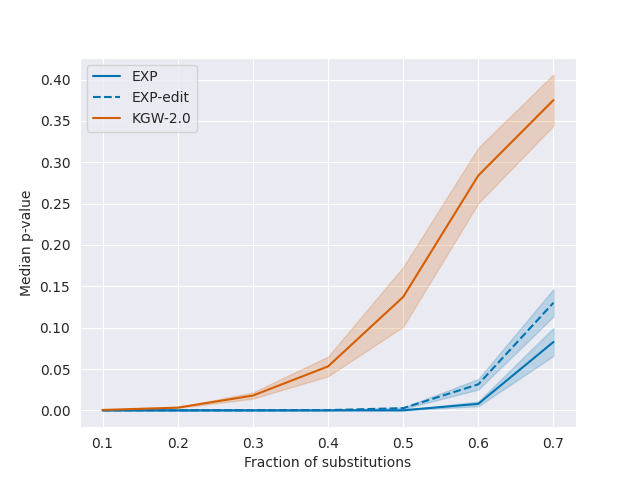}
      \caption{OPT-1.3B}
    \end{subfigure}
    \hfill
    \begin{subfigure}[b]{0.49\textwidth}
      \includegraphics[width=\textwidth]{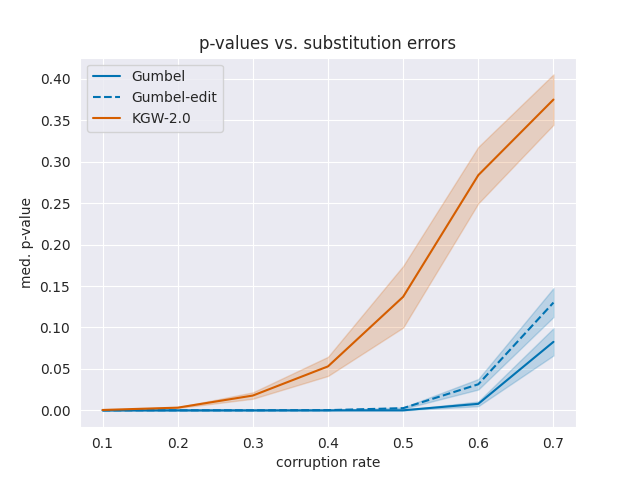}
      \caption{LLaMA 7B}
    \end{subfigure}
  
    \caption{Median $p$-value of watermarked text relative to the fraction of substitution errors, 
    for OPT-1.3B and LLaMA 7B models with $m = 70$.
    }
\end{figure}

\begin{figure}[H]
    \centering
    \begin{subfigure}[b]{0.49\textwidth}
      \includegraphics[width=\textwidth]{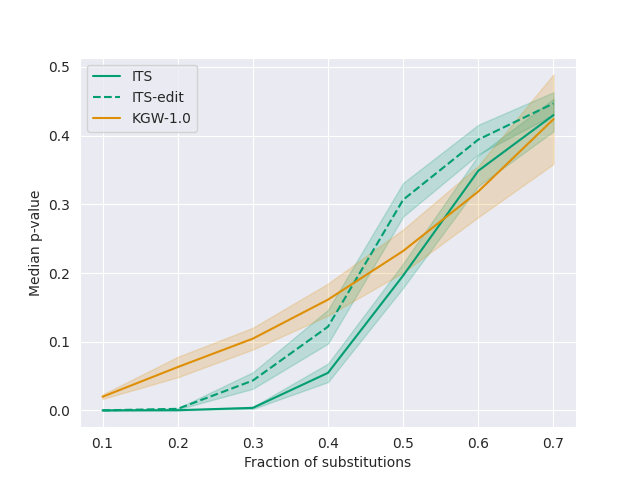}
      \label{fig:experiment-3-opt-app2}
      \caption{OPT-1.3B}
    \end{subfigure}
    \hfill
    \begin{subfigure}[b]{0.49\textwidth}
      \includegraphics[width=\textwidth]{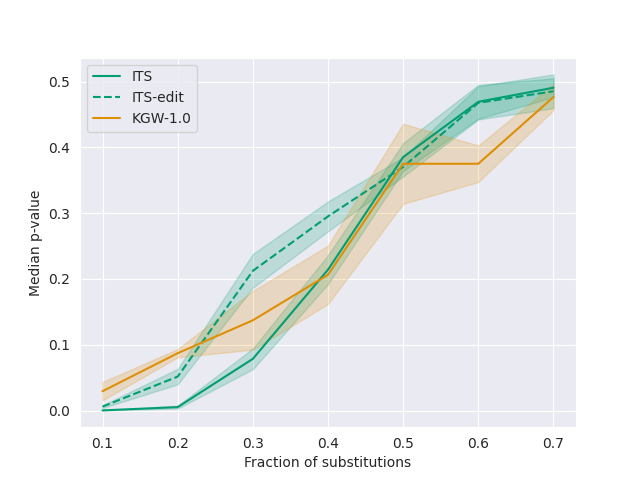}
      \label{fig:experiment-3-llama-app2}
      \caption{LLaMA 7B}
    \end{subfigure}
  
    \caption{Median $p$-value of watermarked text relative to the fraction of substitution errors, 
    for OPT-1.3B and LLaMA 7B models with $m = 70$.
    }
    \label{fig:experiment-3-app2}
\end{figure}

\subsubsection{Experiment 4}
\begin{figure}[H]
    \centering
    \begin{subfigure}[b]{0.49\textwidth}
      \includegraphics[width=\textwidth]{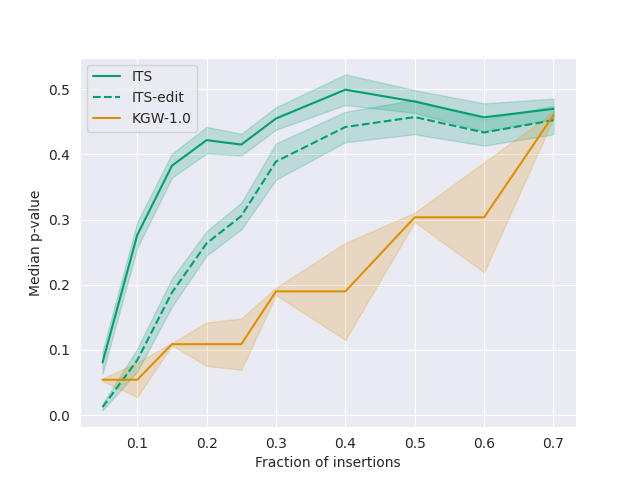}
      \caption{OPT-1.3B}
    \end{subfigure}
    \hfill
    \begin{subfigure}[b]{0.49\textwidth}
      \includegraphics[width=\textwidth]{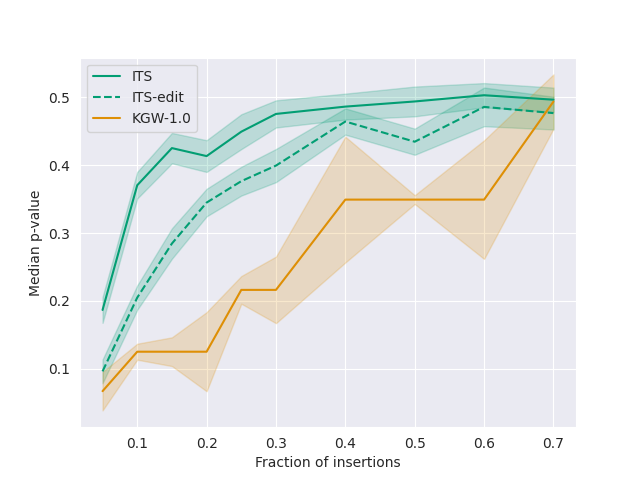}
      \caption{LLaMA 7B}
    \end{subfigure}
  
    \caption{Median $p$-value of watermarked text relative to the fraction of insertion errors, 
    for OPT-1.3B and LLaMA 7B models with $m = 35$.
    }
\end{figure}

\begin{figure}[H]
    \centering
    \begin{subfigure}[b]{0.49\textwidth}
      \includegraphics[width=\textwidth]{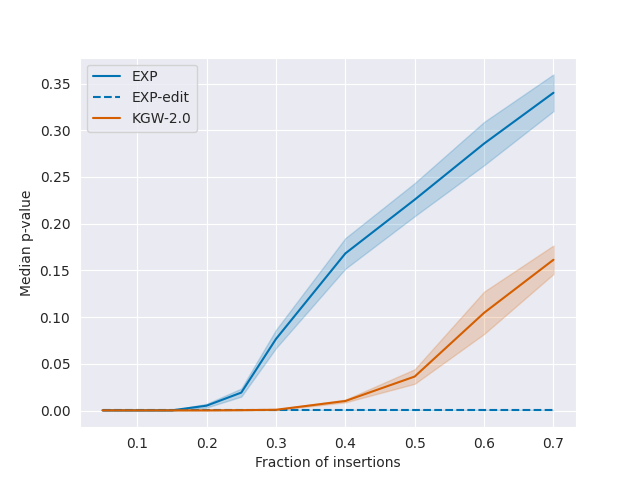}
      \caption{OPT-1.3B}
    \end{subfigure}
    \hfill
    \begin{subfigure}[b]{0.49\textwidth}
      \includegraphics[width=\textwidth]{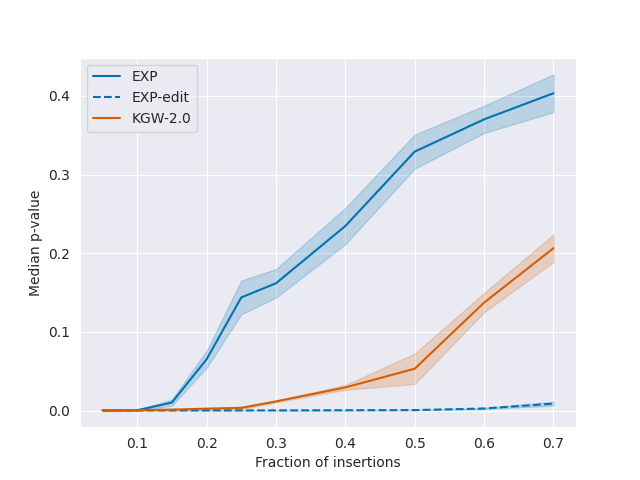}
      \caption{LLaMA 7B}
    \end{subfigure}
  
    \caption{Median $p$-value of watermarked text relative to the fraction of insertion errors, 
    for OPT-1.3B and LLaMA 7B models with $m = 70$.
    }
\end{figure}

\begin{figure}[H]
    \centering
    \begin{subfigure}[b]{0.49\textwidth}
      \includegraphics[width=\textwidth]{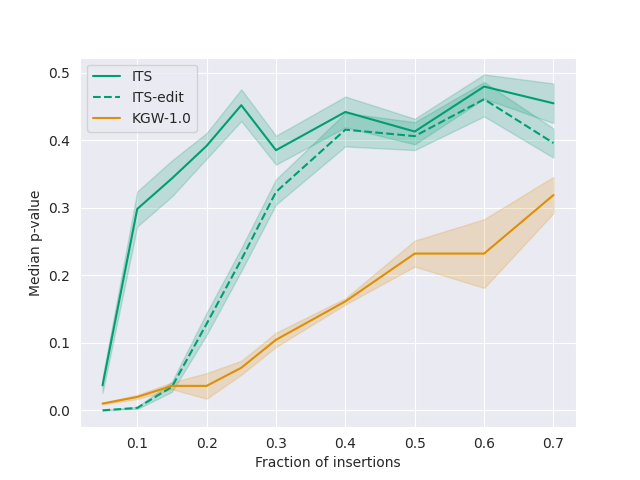}
      \caption{OPT-1.3B}
    \end{subfigure}
    \hfill
    \begin{subfigure}[b]{0.49\textwidth}
      \includegraphics[width=\textwidth]{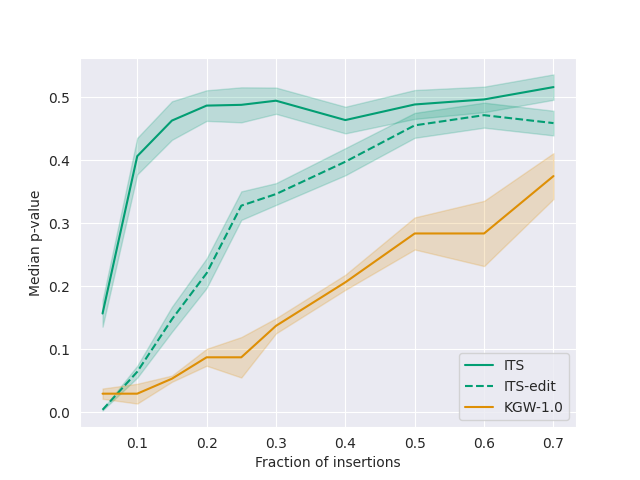}
      \caption{LLaMA 7B}
    \end{subfigure}
  
    \caption{Median $p$-value of watermarked text relative to the fraction of insertion errors, 
    for OPT-1.3B and LLaMA 7B models with $m = 70$.
    }
\end{figure}

\subsubsection{Experiment 5}
\begin{figure}[H]
    \centering
    \begin{subfigure}[b]{0.49\textwidth}
      \includegraphics[width=\textwidth]{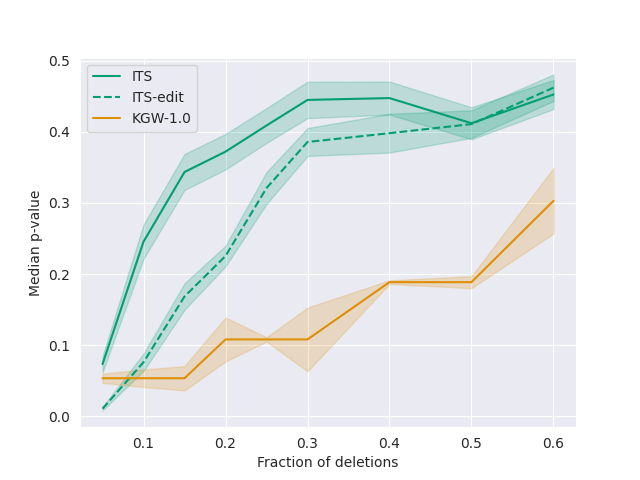}
      \caption{OPT-1.3B}
    \end{subfigure}
    \hfill
    \begin{subfigure}[b]{0.49\textwidth}
      \includegraphics[width=\textwidth]{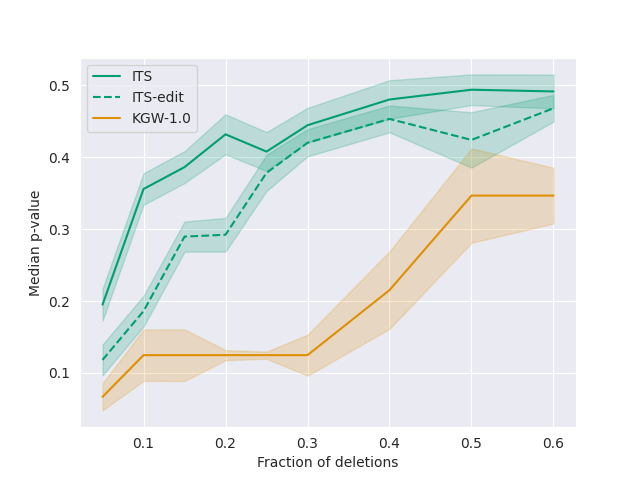}
      \caption{LLaMA 7B}
    \end{subfigure}
  
    \caption{Median $p$-value of watermarked text relative to the fraction of deletion errors, 
    for OPT-1.3B and LLaMA 7B models with $m = 35$.
    }
\end{figure}

\begin{figure}[H]
    \centering
    \begin{subfigure}[b]{0.49\textwidth}
      \includegraphics[width=\textwidth]{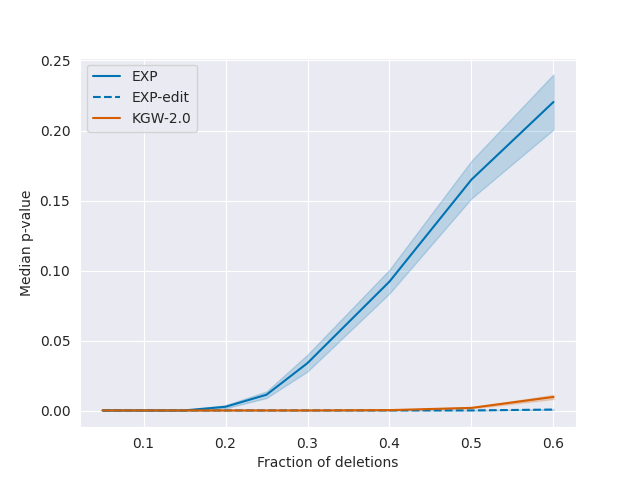}
      \caption{OPT-1.3B}
    \end{subfigure}
    \hfill
    \begin{subfigure}[b]{0.49\textwidth}
      \includegraphics[width=\textwidth]{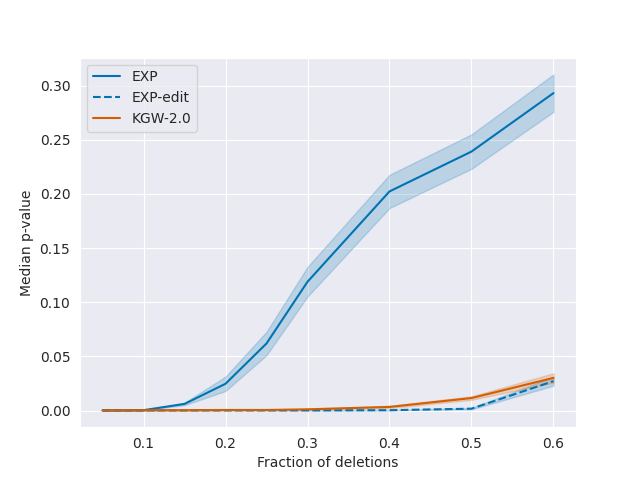}
      \caption{LLaMA 7B}
    \end{subfigure}
  
    \caption{Median $p$-value of watermarked text relative to the fraction of deletion errors, 
    for OPT-1.3B and LLaMA 7B models with $m = 70$.
    }
\end{figure}

\begin{figure}[H]
    \centering
    \begin{subfigure}[b]{0.49\textwidth}
      \includegraphics[width=\textwidth]{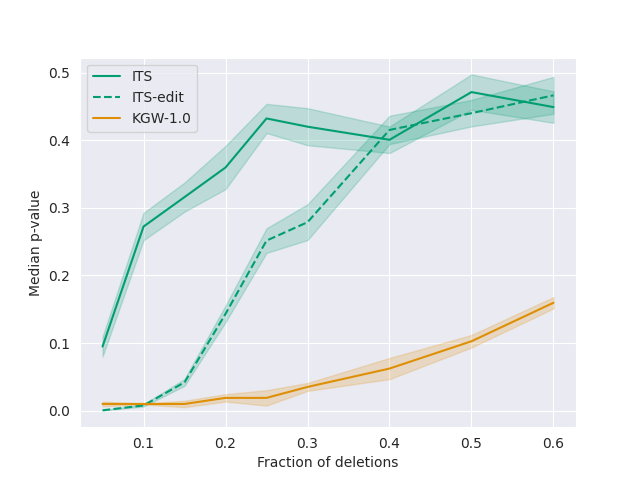}
      \caption{OPT-1.3B}
    \end{subfigure}
    \hfill
    \begin{subfigure}[b]{0.49\textwidth}
      \includegraphics[width=\textwidth]{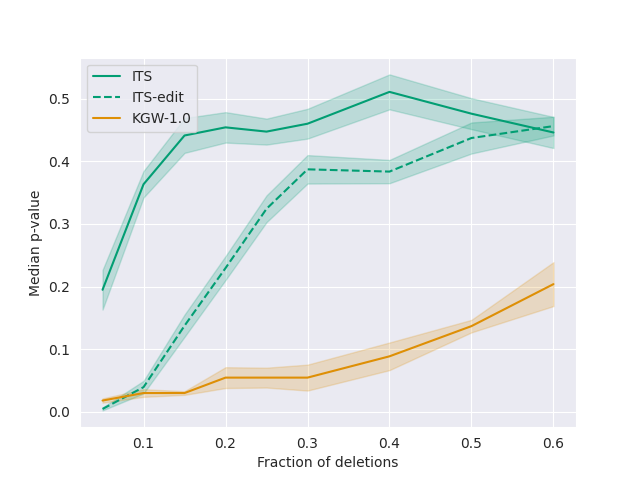}
      \caption{LLaMA 7B}
    \end{subfigure}
  
    \caption{Median $p$-value of watermarked text relative to the fraction of deletion errors, 
    for OPT-1.3B and LLaMA 7B models with $m = 70$.
    }
\end{figure}

\subsubsection{Experiment 6}

\begin{figure}[H]
    \centering
    \begin{subfigure}[b]{0.49\textwidth}
      \includegraphics[width=\textwidth]{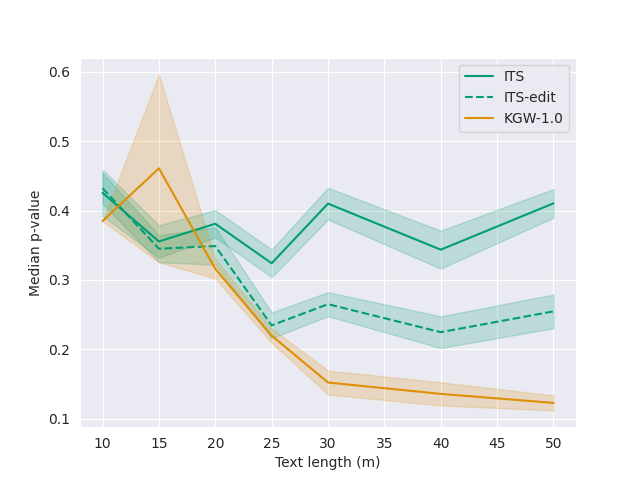}
      \caption{OPT-1.3B}
    \end{subfigure}
    \hfill
    \begin{subfigure}[b]{0.49\textwidth}
      \includegraphics[width=\textwidth]{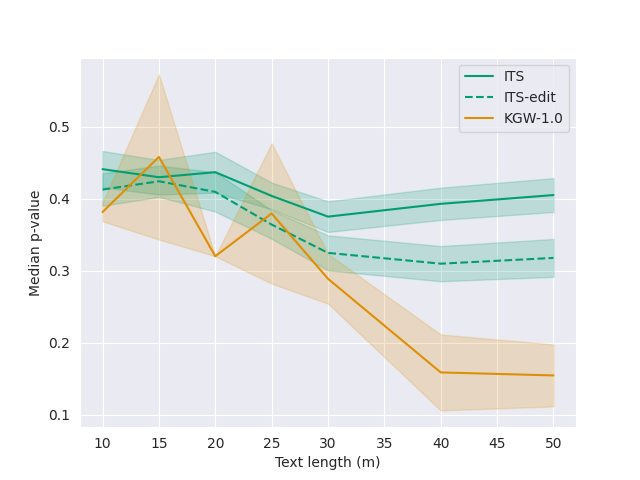}
      \caption{LLaMA 7B}
    \end{subfigure}
  
    \caption{Median $p$-value of watermarked text relative to the fraction of insertion errors, after roundtrip translation via French,
    for OPT-1.3B and LLaMA 7B models with $m = 35$.
    }
\end{figure}

\begin{figure}[H]
    \centering
    \begin{subfigure}[b]{0.49\textwidth}
      \includegraphics[width=\textwidth]{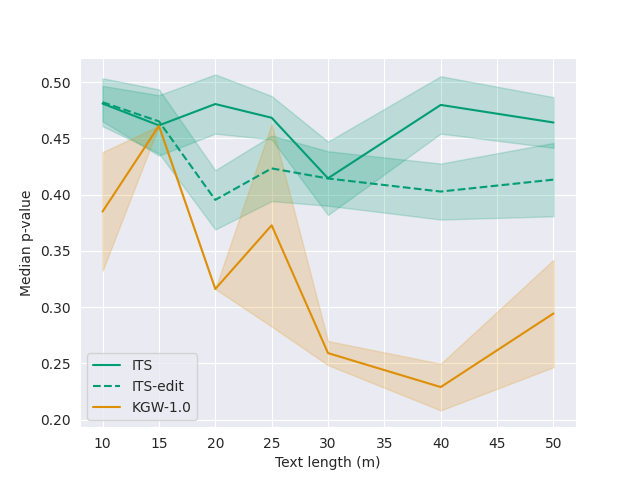}
      \caption{OPT-1.3B}
    \end{subfigure}
    \hfill
    \begin{subfigure}[b]{0.49\textwidth}
      \includegraphics[width=\textwidth]{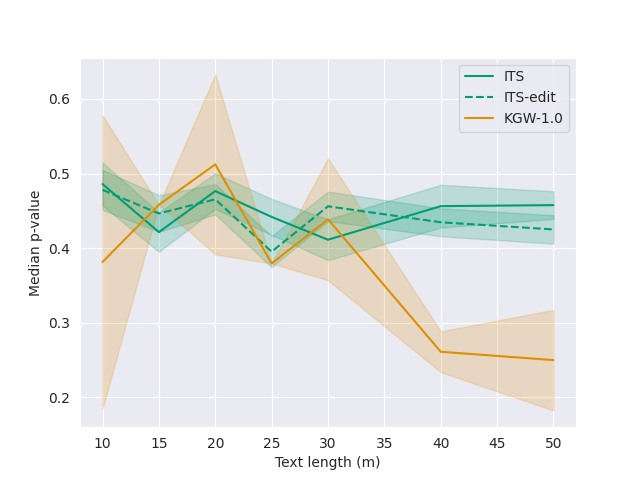}
      \caption{LLaMA 7B}
    \end{subfigure}
  
    \caption{Median $p$-value of watermarked text relative to the text length, after roundtrip translation via Russian, 
    for OPT-1.3B and LLaMA 7B models with $m = 35$.
    }
\end{figure}

\subsubsection{Instruction following case study}\label{app:experiments-instruct}
We give three examples of instructions for which hashing produces qualitatively worse responses 
than regular samples from the language model:
\begin{itemize}
  \item[1.] ``Give me 20 ideas for the title of a paper on watermarking language models.''
  \item[2.] ``Give me 20 ideas for startup names.''
  \item[3.] ``Give me a list of 20 movies.'' 
\end{itemize}
We format each of the instructions as described by \citet{taori23alpaca} before calling the model.

We compare samples from our {\gumb} watermark strategy,\footnote{
  Recall both {\gumb} and {\gumbedit} use the same $\generate$ method.
}
which are equivalent to regular samples from the language model,
to samples from {\kgwtwo} and the hashing-based version of {\gumb} we describe in the main text (i.e.,
the watermark of \citet{aaronson2023openai}), i.e., {\gumbhash}.
For both {\gumb} and {\kgwtwo}, we generate the samples using five different random seeds (the hash 
function in {\kgwtwo} is fixed in the implementation of \citet{kirchenbauer2023watermark}),
whereas in the case of {\gumbhash} we use five different hash functions (namely,
we let the previous $k$ tokens $\{y_i\}_{i=1}^k$ hash to $j + \sum_{i=1}^k y_i$ for $j \in \{0,\dots,4\}$).
We label each sample using the seed/hash we used to generate it.
We include samples from two versions of {\gumbhash}: one where we hash the previous tokens ($k = 1$) and another 
where we hash the previous four tokens ($k = 4$).
For {\kgwtwo}, we only hash the previous token since the public implementation of \citet{kirchenbauer2023watermark}
does not include the option to hash more tokens.

We find that {\gumbhash} with $k = 1$ often produces qualitatively worse responses that degenerate into 
repetition. With $k = 4$, the repetition is substantially less noticeable, though occasionally it still manifests.
In contrast, even when we only hash the previous token, the repetition of {\kgwtwo} is not nearly as noticeable as in {\gumbhash}.
We speculate this is due to stochasticity of {\kgwtwo} (i.e., {\kgwtwo} biases the distribution over the next token 
to a subset of tokens but still ultimately samples from this distribution randomly).
Of course, this stochasticity comes at a price: {\kgwtwo} was generally less powerful compared to the {\gumb} and {\gumbedit} strategies
in our other experiments.

We include sample sheets for all methods for the first instruction below. To avoid excessive clutter,
we defer the sample sheets for the remaining two instructions to our code release.

\begin{longtable}{|>{\centering\arraybackslash}m{0.1\textwidth}|>{\centering\arraybackslash}m{0.05\textwidth}|m{0.85\textwidth}|}
  \hline
  \textbf{Method} & \textbf{Seed} & \textbf{Response} \\
  \hline
  {\gumbhash} ($k = 1$) & 0 & 1. Watermarking Deep Learning Models: A Review
  2. Understanding and Evaluating Watermarking Techniques for Neural Networks
  3. Towards Unobtrusive and Invisible Watermarking for Neural Networks
  4. Customizable Watermarking for Neural Networks
  5. Reverse Engineering Watermarked Neural Networks
  6. Understanding and Exploiting Watermarking Attacks on Neural Networks
  7. Combatting Watermarking Attacks on Neural Networks
  8. Customizable Watermarking for Neural Networks
  9. Customizable Watermarking for Neural Networks
  10. Customizable Watermarking for Neural Networks
  11. Customizable Watermarking for Neural Networks
  12. Customizable Watermarking for Neural Networks
  13. Customizable Watermarking for Neural Networks
  14. Customizable Watermarking for Neural Networks
  15. Customizable Watermarking for Neural Networks
  16. Customizable Watermarking for Neural Networks
  17. Customizable Watermarking for Neural Networks
  18. Customizable Watermarking for Neural Networks
  19. Customizable Watermarking for Neural Networks \\
  \hline
  {\gumbhash} ($k = 1$) & 1 & 1. Exploiting Watermarks in Language Modeling 
  2. Watermarking Language Models: State of the Art 
  3. Keys to the Kingdom: Refining Watermarking Language Models 
  4. Exploiting Watermarks in Natural Language Processing 
  5. Keys to the Kingdom: Refining Watermarking Language Models 
  6. Keys to the Kingdom: Refining Watermarking Language Models 
  7. Keys to the Kingdom: Refining Watermarking Language Models 
  8. Keys to the Kingdom: Refining Watermarking Language Models 
  9. Keys to the Kingdom: Refining Watermarking Language Models 
  10. Keys to the Kingdom: Refining Watermarking Language Models 
  11. Keys to the Kingdom: Refining Watermarking Language Models 
  12. Keys to the Kingdom: Refining Watermarking Language Models 
  13. Keys to the Kingdom: Refining Watermarking Language Models 
  14. Keys to the Kingdom: Refining Watermarking Language Models 
  15. Keys to the Kingdom: Refining Watermarking Language Models 
  16. Keys to the Kingdom: Refining Watermarking Language Models 
  17. Keys to the Kingdom: Refining Watermarking Language Models 
  18. Keys to the Kingdom: Refining Watermarking Language Models 
  19. Keys to the Kingdom: Refining Watermarking Language Models 
  20. Keys to the Kingdom: Refining Watermarking Language Models \\
  \hline
  {\gumbhash} ($k = 1$) & 2 & 1. "Making Watermarks: Creative Approaches to Watermarking Language Models"
  2. "Making Watermarks: Creative Approaches to Watermarking Language Models"
  3. "Making Watermarks: Creative Approaches to Watermarking Language Models"
  4. "Making Watermarks: Creative Approaches to Watermarking Language Models"
  5. "Making Watermarks: Creative Approaches to Watermarking Language Models"
  6. "Making Watermarks: Creative Approaches to Watermarking Language Models"
  7. "Making Watermarks: Creative Approaches to Watermarking Language Models"
  8. "Making Watermarks: Creative Approaches to Watermarking Language Models"
  9. "Making Watermarks: Creative Approaches to Watermarking Language Models"
  10. "Making Watermarks: Creative Approaches to Watermarking Language Models"
  11. "Making Watermarks: Creative Approaches to Watermarking Language Models"
  12. "Making Watermarks: Creative Approaches to Watermarking Language Models"
  13. "Making Watermarks: Creative Approaches to Watermarking Language Models"
  14. "Making Watermarks: Creative Approaches to Watermarking Language Models"
  15. "Making Watermarks: Creative Approaches to Watermarking Language Models"
  16. "Making Watermarks: Creative Approaches to Watermarking Language Models"
  17. "Making Watermarks: Creative Approaches to Watermarking Language Models"
  18. "Making Watermarks: Creative Approaches to Watermarking Language Models"
  19. "Making Watermarks: Creative Approaches to Watermarking Language Models"
  20. "Making Watermarks: Creative Approaches to Watermarking Language Models" \\
  \hline
  {\gumbhash} ($k = 1$) & 3 & 1. Watermarking Language Models: A Review
  2. Watermarking Language Models for Copyright Protection
  3. Watermarking Language Models for Forensic Analysis
  4. Watermarking Language Models for Audio Signals
  5. Watermarking Language Models for Video Signals
  6. Watermarking Language Models for Documents
  7. Watermarking Language Models for Software
  8. Watermarking Language Models for Big Data
  9. Watermarking Language Models for Natural Language Processing
  10. Watermarking Language Models for Machine Learning
  11. Watermarking Language Models for Deep Learning
  12. Watermarking Language Models for Speech Recognition
  13. Watermarking Language Models for Audio Synthesis
  14. Watermarking Language Models for Computer Vision
  15. Watermarking Language Models for Robotics
  16. Watermarking Language Models for Autonomous Driving
  17. Watermarking Language Models for Medical Image Processing
  18. Watermarking Language Models for Medical Diagnostics
  19. Watermarking Language Models for Text Translation \\
  \hline
  {\gumbhash} ($k = 1$) & 4 & 1. Image based watermarking for language models
  2. Steganographic watermarking for language models
  3. Timestamps for language models
  4. Watermarking for language model improvisation
  5. Secure watermarking for language model sharing
  6. Image based watermarking for language model improvisation
  7. Timestamps for language model improvisation
  8. Watermarking for language model evolution
  9. Secure watermarking for language model evolution
  10. Image based watermarking for language model evolution
  11. Timestamps for language model evolution
  12. Watermarking for language model generation
  13. Image based watermarking for language model generation
  14. Timestamps for language model generation
  15. Watermarking for language model understanding
  16. Image based watermarking for language model understanding
  17. Timestamps for language model understanding
  18. Watermarking for language model improvisation and understanding
  19. Image based watermarking for language model improvisation and understanding \\
  \hline
  \caption{Responses from {\gumbhash} with $k = 1$ to the prompt: ``Give me 20 ideas for the title of a paper on watermarking language models.''}
\end{longtable}

\begin{longtable}{|>{\centering\arraybackslash}m{0.1\textwidth}|>{\centering\arraybackslash}m{0.05\textwidth}|m{0.85\textwidth}|}
  \hline
  \textbf{Method} & \textbf{Seed} & \textbf{Response} \\
  \hline
  {\gumbhash} ($k = 4$) & 0 & 1. "Safeguarding Language Models: Exploring Watermarking for Vertext and Image Models"
  2. "Watermarking Language Models: Protecting Your Creations"
  3. "The Future of Watermarking Language Models: Ethical and Practical Considerations"
  4. "Securing Language Models: Exploring Emerging Technologies"
  5. "Watermarking Language Models: Balancing Creativity and Security"
  6. "Watermarking Language Models: Balancing Security and Accessibility"
  7. "Watermarking Language Models: Defending your Creativity"
  8. "Watermarking Language Models: Defending your Creativity with Emerging Technologies"
  9. "Protecting Language Models: Exploring Emerging Technologies"
  10. "Securing Language Models: Balancing Security and Accessibility"
  11. "The Future of Watermarking Language Models: Ethical and Practical Considerations"
  12. "Watermarking Language Models: Balancing Security and Accessibility"
  13. "Watermarking Language Models: Balancing Security and Accessibility with Emerging Technologies"
  14. "Watermarking Language Models: Defending your Creativity with Emerging Technologies"
  15. "Watermarking Language Models: Defending your Creativity with Emerging Technologies"
  16. "Securing Language Models: Balancing Security and Accessibility with Emerging Technologies"
  17. "Exploring Watermarking for Vertext and Image Models"
  18. "Watermarking Language Models: Balancing Security and Accessibility with Emerging Technologies"
  19. "Defending your Creativity with Emerging Technologies" \\
  \hline
  {\gumbhash} ($k = 4$) & 1 & 1. Towards a New Era of Transparent Language Models 
  2. A Review of the State of Watermarking Language Models 
  3. The Benefits of Embedding Watermarks in Language Models 
  4. Protecting Language Models with Multiscale Watermarks 
  5. Impact of Watermarking on the Performance of Language Models
  6. A Survey on Watermarking for Language Models 
  7. Practical Perspectives on Watermarking for Language Models 
  8. A Comprehensive Study on Designing Watermarks for Language Models 
  9. Overview of Techniques for Adding Watermarks to Language Models 
  10. Exploring the Possibilities of Watermarking for Language Models 
  11. How to Incorporate Watermarks in Your Language Model 
  12. The Science behind Watermarking for Language Models 
  13. AI for Insertion of Watermarks in Language Models 
  14. The Role of Machine Learning in Watermarking for Language Models 
  15. Future Trends in Watermarking for Language Models 
  16. A Review on Watermarking for Language Models 
  17. Applications of Watermarking in Language Modeling 
  18. A Comprehensive Study on Designing Robust Watermarks for Language Models 
  19. A Novel Approach to Incorporate Watermarks in Your Language Model. \\
  \hline
  {\gumbhash} ($k = 4$) & 2 & 1. Securing Your Language Model
  2. Stamping Out Unauthorized Use
  3. Coloring Outside the Lines: Creative Watermarks
  4. Avoiding Watermarks: Best Practices
  5. Authentication Made Easy with Watermarks
  6. Defending Your Language Model
  7. Unique Identifiers: Adding Value to your Model
  8. Connected Learning: Leveraging Watermarks
  9. The Problem with Open Access
  10. How to Effectively Mark a Language Model
  11. Making a Splash with Creative Watermarks
  12. Understanding the Benefits of Watermarking
  13. Utilizing Watermarks for Better Attribution
  14. Stewarding Your Language Model
  15. The Role of Technology in Watermarking
  16. Beyond the Horizon: Adaptive Watermarking
  17. The Art of Discretion in Watermarking
  18. Harnessing the Power of Invisible Watermarks
  19. Practical Considerations for Watermarking \\
  \hline
  {\gumbhash} ($k = 4$) & 3 & 1. Elucidating Watermarking Strategies for Language Models 
  2. Innovative Techniques for Watermarking Language Models 
  3. Intelligent and Innovative Watermarking Strategies for Language Models 
  4. Overview of Watermarking Strategies for Language Models 
  5. Emerging Technologies in Watermarking for Language Models 
  6. Practical Perspectives on Watermarking Language Models 
  7. Reducing Watermarking Strategies for Language Models 
  8. Robust Techniques for Watermarking Language Models 
  9. Security Factors in Watermarking Language Models 
  10. Understanding Watermarking Strategies for Language Models 
  11. Evaluation of Watermarking Strategies for Language Models 
  12. Exploring Watermarking Strategies for Language Models 
  13. Innovations in Watermarking Language Models 
  14. Intelligent Techniques for Watermarking Language Models 
  15. Methods for Watermarking Language Models 
  16. Overlapping Security Factors in Watermarking Language Models 
  17. Understanding Watermarking Strategies for Language Models 
  18. Emerging Technologies in Watermarking Language Models 
  19. Exploring Watermarking Strategies for Language Models 
  20. Intelligent and Innovative Watermarking Strategies for Language Models \\
  \hline
  {\gumbhash} ($k = 4$) & 4 & 1. Securing Language Models 
  2. Extra Secure Language Models 
  3. Defending Language Models 
  4. Protecting Language Models 
  5. Cryptographic Language Models 
  6. Unique Language Models
  7. Robust Language Models
  8. Marking Language Models
  9. Preventing Language Models
  10. Hiding Language Models
  11. Invisible Language Models
  12. Steganographic Language Models
  13. Watermarking Language Models
  14. Preserving Language Models
  15. Uniquely Identifiable Language Models
  16. Characterizing Language Models
  17. Distinguishable Language Models
  18. Language Model Defence
  19. Ultra Defence Language Models \\
  \hline
  \caption{Responses from {\gumbhash} with $k = 4$ to the prompt: ``Give me 20 ideas for the title of a paper on watermarking language models.''}
\end{longtable}

\begin{longtable}{|>{\centering\arraybackslash}m{0.1\textwidth}|>{\centering\arraybackslash}m{0.05\textwidth}|m{0.85\textwidth}|}
    \hline
    \textbf{Method} & \textbf{Seed} & \textbf{Response} \\
    \hline
    {\kgwtwo} & 0 & 1. Steganography in Language Models: A Review
    2. Combating Model Piracy with Watermarking
    3. Techniques for Traceability and Watermarking
    4. Salient Features for Traceability and Watermarking
    5. Reflections on Model Security and Watermarking
    6. Insights on Model Security and Watermarking
    7. Protecting Models with Watermarking
    8. Techniques for Esoteric Traceability and Watermarking
    9. Ins and Outs of Steganography in Language Models
    10. Technologies for Efficient Watermarking
    11. Vision Beyond Model Piracy: Watermarking Perspectives
    12. Impact of Model Security on Watermarking
    13. Emerging Trends in Watermarking of Language Models
    14. Future of Watermarking Techniques for Language Models
    15. Drivers for Success in Watermarking Language Models
    16. Robustness of Models against Watermarking
    17. Taking Security beyond Watermarking
    18. Leveraging Model History and Watermarking
    19. Techniques for Secure Watermarking
    20. Comprehensive Overview of Model Security and Watermarking. \\
    \hline
    {\kgwtwo} & 1 & 1. Embedding Digital Signatures: Towards Traceable and Transparent Language Models
    2. Stamping Out Deception: Recording Attribution in Language Models 
    3. Defining Your Digital Dashboard: Watermarking for Traceability
    4. Scaling Trust in Language Models: Watermarking for Traceability and Transparency 
    5. Taking an East Asian Approach to Transparency in Language Models 
    6. Riding the Big Five: Towards Transparency in Language Models 
    7. Measuring the "Ps" of Language Models: Perceptual and Practical Transparency 
    8. Marking a Mark: Comprehensive Examination of Attribution in Language Models 
    9. Visualizing Deception: Comprehensive Examination of Deceptive Language Models 
    10. Peak Performance: Evaluating Language Models for Transparency 
    11. Closing the Loop: Examining the Long-term Impacts of Language Models 
    12. Diversifying Perspective: Impacts of Modelling Language 
    13. Disclosing the Hidden Layers: Understanding Transparency in Language Models 
    14. Auditable: Making Language Models Accountable 
    15. Labelling Human Language: Exploring the Capabilities of Language Models 
    16. Setting the Record: Estimating the Real-world Impacts of Language Models 
    17. Language Models: Towards Traceable, Transparent and Accountable Systems 
    18. Think Outside the Box: Exploring Future Trends in Language Models 
    19. Embedding Trust: Towards Traceable and Transparent Language Models.  \\
    \hline
    {\kgwtwo} & 2 & 1. Defending Latent Space with Secure Watermarks
    2. Emergent Waveforms: A Watermarking Approach
    3. Embedding Digital Identities in Neural Networks
    4. Marking up Models: A Survey on Watermarking
    5. Leveraging Linked Data for Watermarking
    6. Stamping Out Superfluous Models
    7. Inspecting and Marking Models 
    8. Rethinking Digital Identity with Watermarks
    9. Marking Up Models: A Usage-Driven Approach
    10. Bridge the Distance with Digital Watermarks
    11. Introducing Transparency in Model Formation 
    12. Unlocking Models with Secure Watermarks
    13. Giving Visible Identity to Models 
    14. Unveiling Dark Knowledge with Watermarks
    15. Linking Models: A Visual Approach 
    16. Visualizing Dark Knowledge through Watermarks
    17. Leveraging Physical World for Model Tracking 
    18. Augmented Reality for Model Tracking
    19. Securing Deep Models with Watermarks  \\
    \hline
    {\kgwtwo} & 3 & 1. Providing Type-Safe Watermarks
    2. Constructing Robust Watermarks for Multimedia Applications
    3. Adding Security Features to Text Vectors
    4. Stretching Visible Watermarks
    5. Embedding Watermarks in Natural Language Models
    6. Training Recurrent Neural Networks with Watermarks
    7. Detection of Unsupervised Watermarks in Text Sequences
    8. Interactive Visualization of Watermarks
    9. Observe Hidden Watermarks with Machine Learning
    10. Extracting Information from Unseen Watermarks
    11. Building Robust Watermarks for the Blockchain
    12. Adding Dynamic Watermarks to Video Streams
    13. Storing Watermarks in Individual Dataset Elements
    14. Extracting Pertinent Information from Unwanted Watermarks
    15. Decoding Hidden Digital Watermarks
    16. Introducing Watermarks to Secure Semi-honest Adversaries
    17. Genuine Components with Integrated Watermarks
    18. Intelligent Techniques for Steganography 
    19. Protecting Quality with Data-Driven Watermarks  \\
    \hline
    {\kgwtwo} & 4 & 1. Embedding Watermarks in Language Models 
    2. Secure Language Models using Watermarks
    3. Steganographic watermarking of natural language models 
    4. Exploring Watermarking Techniques for Natural Language Models 
    5. Analyzing Natural Language Models using Watermarking Techniques
    6. Detecting Watermarks in Natural Language Models
    7. Supervised Learning of Watermarks for Natural Language Models 
    8. Unsupervised Learning of Watermarks for Natural Language Models 
    9. Multiple watermarks for secure language models
    10. Generative watermarking techniques for natural language models
    11. Evaluation of watermarks for secure natural language models
    12. Efficacy of watermarks in preventing model poisoning
    13. Robustness of watermarks for natural language models 
    14. Comparison of watermarking techniques for natural language models 
    15. Linking of natural language models with watermarks
    16. Extraction of hidden information from natural language models 
    17. Universality of watermarks for different languages 
    18. Learning watermarked natural language models 
    19. Reduced Dimensionality Modeling using Watermarks  \\
    \hline
    \caption{Responses from {\kgwtwo} to the prompt: ``Give me 20 ideas for the title of a paper on watermarking language models.''}
\end{longtable}

\begin{longtable}{|>{\centering\arraybackslash}m{0.1\textwidth}|>{\centering\arraybackslash}m{0.05\textwidth}|m{0.85\textwidth}|}
  \hline
  \textbf{Method} & \textbf{Seed} & \textbf{Response} \\
  \hline
  {\gumb} & 0 & Here is the list of 20 ideas for the title of a paper on watermarking language models: 

  1. "Understanding Watermarking Techniques for Language Models" 
  2. "A Comprehensive Study on Watermarking for Language Models"
  3. "Watermarking Techniques for Optimizing Language Models"
  4. "A Survey of Watermarking Techniques for Language Models"
  5. "Exploring the Potential of Watermarking for Language Models"
  6. "Implementing Watermarking for Language Model Optimization"
  7. "Watermarking Strategies for Enhancing Language Models"
  8. "Investigating the Possibilities of Watermarking for Language Models"
  9. "Advanced Watermarking Approaches for Language Models"
  10. "Exploring the Use of Watermarking for Language Model Validation"
  11. "Understanding the Benefits of Watermarking for Language Model Ensemble"
  12. "Leveraging Watermarking Techniques for Enhanced Language Model Performance"
  13. "Enhancing Language Models with Watermarking"
  14. "Evaluating the Impact of Watermarking Techniques on Language Models"
  15. "Analyzing the Feasibility of Watermarking for Language Model Comparison"
  16. "Exploring the Possibilities of Watermarking for Optimizing Language Models"
  17. "Exploiting Watermarking to Enhance Language Model Accuracy"
  18. "Advantages of Using Watermarking for Validating Language Models"
  19. "The Promise of Watermarking for Evaluating Language Model Performance" \\
  \hline
  {\gumb} & 1 & Title ideas for a paper on watermarking language models:
  1. Hidden Messages: Protecting Language Models with Watermarking
  2. Securing Language Models with Watermarking
  3. Defending Language Models against Tampering
  4. Watermarking Language Models for Better Protection
  5. Verifying the Integrity of Language Models
  6. Utilizing Watermarking to Boost Language Model Security
  7. Ensuring the Authenticity of Language Models
  8. Safekeeping Language Models with Watermarking
  9. Defending Against Model Tampering with Watermarking
  10. A Review of Model Security Techniques: Watermarking
  11. Utilizing Watermarks to Protect Language Models
  12. Protective Techniques for Language Models: Watermarking
  13. Implementing Watermarks to Boost Language Model Security
  14. Preventing Model Tampering with Watermarking
  15. Methods for Verifying Language Model Integrity
  16. Analyzing the Security of Language Models
  17. Techniques for Securing Language Models
  18. Investigating Model Security with Watermarking
  19. Extending Language Model Security with Watermarking \\
  \hline
  {\gumb} & 2 & 1. "Deep Watermarks: Towards Pervasive Protection for Language Models" 
  2. "Adding Transparency to Language Models: A Watermarking Approach"
  3. "Watermarking Language Models for Traceability and Verification"
  4. "A Comprehensive Review on Watermarking Methods for Language Models"
  5. "The Power of Watermarking for Language Model Protection"
  6. "Attributing Authorship in Language Models with Watermarking"
  7. "Understanding the Promise of Watermarking for Language Model Security"
  8. "Watermarking Language Models: A Survey"
  9. "Watermarking Language Models for Better Security"
  10. "Mitigating Plagiarism in Language Models with Watermarking"
  11. "Watermarking Language Models: Exploring the Possibilities"
  12. "A Generative Approach to Watermarking Language Models"
  13. "Watermarking Strategies for Protecting Language Models"
  14. "Watermark Detection for Language Model Security"
  15. "Making Language Model Security Watertight with Watermarking"
  16. "Leveraging Watermarking for Enhancing Language Model Security"
  17. "Understanding the Role of Watermarking in Language Model Security"
  18. "A Novel Approach to Watermarking Language Models"
  19. "Exploring the Possibilities of Watermarking for Language Model Security" \\
  \hline
  {\gumb} & 3 & 1. A Comprehensive Review of Watermarking Techniques for Language Models
  2. A Survey of Watermarking Approaches for Language Modeling
  3. A Normative Analysis of Watermarking for Language Modeling
  4. Investigating Watermarking Techniques for Language Modeling
  5. An Overview of Watermarking Methods for Language Modeling
  6. Exploring Watermarking Solutions for Language Modeling
  7. A Taxonomy of Watermarking Methods for Language Modeling
  8. A Comparative Study of Watermarking Approaches for Language Modeling
  9. Evaluating Watermarking Methods for Language Modeling
  10. A Theoretical Analysis of Watermarking for Language Modeling
  11. Investigating Watermarking Techniques for Natural Language Modeling
  12. An Analytical Study of Watermarking for Language Modeling
  13. Exploring Watermarking Methods for Natural Language Modeling
  14. A Review of Watermarking Techniques for Natural Language Modeling
  15. A Comparative Study of Watermarking Approaches for Natural Language Modeling
  16. Investigating Watermarking Solutions for Natural Language Modeling
  17. A Survey of Watermarking Techniques for Natural Language Modeling
  18. Evaluating Watermarking Methods for Natural Language Modeling
  19. A Review of Watermarking for Natural Language Modeling
  20. Exploring Watermarking Solutions for Natural Language Modeling \\
  \hline
  {\gumb} & 4 & 1. Towards a Unified Watermarking Mechanism for Natural Language Processing Models
  2. A Review of Methods for Watermarking Natural Language Models
  3. Extracting Invariant Features for Watermarking Language Models
  4. The Use of Steganography for Watermarking Natural Language Models 
  5. Introducting Secure Watermarking Techniques for Natural Language Models
  6. A Comprehensive Study on Watermarking Techniques for Natural Language Models
  7. Toward Remarkably Visible Watermarks for Natural Language Models
  8. Analyzing the Impact of Watermarking on Natural Language Models
  9. A Practical Guide to Marking Language Models
  10. Enhancing the Accuracy of Watermarking Natural Language Models
  11. Evaluating Strategies for Watermarking Natural Language Models
  12. A Comparison of Watermarking Approaches for Natural Language Models
  13. Promising Solutions for Securely Watermarking Natural Language Models
  14. Generative and Discriminative Approaches for Watermarking Natural Language Models
  15. Exploring the Possibilities of Steganography for Natural Language Models
  16. Understanding the Challenges of Watermarking Natural Language Models
  17. Evaluating the Effectiveness of Watermarking Techniques for Natural Language Models
  18. Enhancing the Transparency of Watermarking Techniques for Natural Language Models
  19. Extending the Capabilities of Watermarking Techniques for Natural Language Models
  20. Assessing the Sophistication of Watermarking Techniques for Natural Language Models \\
  \hline
  \caption{Responses from {\gumb} to the prompt: ``Give me 20 ideas for the title of a paper on watermarking language models.''}
\end{longtable}

\end{document}